\documentclass{article} 
\usepackage{iclr2024_conference,times}

\usepackage{amsmath,amsfonts,bm}

\def\eqref#1{equation~\ref{#1}}

\def\1{\bm{1}}

\def\vzero{{\bm{0}}}

\def\vb{{\bm{b}}}
\def\vc{{\bm{c}}}

\def\vh{{\bm{h}}}

\def\vr{{\bm{r}}}
\def\vs{{\bm{s}}}

\def\vx{{\bm{x}}}
\def\vy{{\bm{y}}}
\def\vz{{\bm{z}}}

\def\mA{{\bm{A}}}
\def\mB{{\bm{B}}}

\def\mI{{\bm{I}}}

\def\mN{{\bm{N}}}

\DeclareMathAlphabet{\mathsfit}{\encodingdefault}{\sfdefault}{m}{sl}
\SetMathAlphabet{\mathsfit}{bold}{\encodingdefault}{\sfdefault}{bx}{n}

\def\gD{{\mathcal{D}}}

\def\gF{{\mathcal{F}}}
\def\gG{{\mathcal{G}}}
\def\gH{{\mathcal{H}}}

\def\gL{{\mathcal{L}}}

\def\sQ{{\mathbb{Q}}}

\DeclareMathOperator*{\argmax}{arg\,max}

\usepackage{booktabs}
\usepackage{url}
\usepackage{amsfonts, amssymb}
\usepackage{mathrsfs}

\usepackage{color}
\usepackage{graphicx}
\usepackage{epstopdf}
\usepackage{dsfont}
\usepackage{makecell}
\usepackage{caption}
\usepackage{algorithm}
\usepackage{algorithmicx}
\usepackage{algpseudocode}
\usepackage{multirow}
\usepackage{amsthm}
\usepackage{amsmath}

\usepackage{subcaption}
\usepackage{wrapfig}
\usepackage{lipsum}
\usepackage{hyperref}

\newtheorem{definition}{Definition}
\newtheorem{lemma}{Lemma}
\newtheorem{proposition}{Proposition}

\iclrfinalcopy

\definecolor{pandacolor}{RGB}{0,128,0}
\newcommand{\hy}[1]{\textcolor{pandacolor}{#1}}

\title{DIG-MILP: a \underline{D}eep \underline{I}nstance \underline{G}enerator \\for Mixed-Integer Linear Programming\\ with Feasibility Guarantee}

\author{Haoyu Wang \\
Georgia Tech\\
\scriptsize{\texttt{haoyu.wang@gatech.edu}} \\
\And
Jialin Liu \\
Damo Academy, Alibaba US \\
\scriptsize{\texttt{jialin.liu@alibaba-inc.com}} \\
\And
Xiaohan Chen \\
Damo Academy, Alibaba US \\
\scriptsize{\texttt{xiaohan.chen@alibaba-inc.com}} \\
\And
Xinshang Wang\\
Damo Academy, Alibaba US \\
\scriptsize{\texttt{xinshang.w@alibaba-inc.com}}
\And
Pan Li\\
Georgia Tech\\
\scriptsize{\texttt{panli@gatech.edu}}
\And 
Wotao Yin\\
Damo Academy, Alibaba US \\
\scriptsize{\texttt{wotao.yin@alibaba-inc.com}}
}

\begin{document}

\newcommand{\dgm}{DIG-MILP}  
\newcommand{\jl}[1]{\textcolor{cyan}{[JL: #1]}}
\newcommand{\pan}[1]{\textcolor{blue}{[P: #1]}}

\maketitle

\begin{abstract}
Mixed-integer linear programming (MILP) stands as a notable NP-hard problem pivotal to numerous crucial industrial applications. The development of effective algorithms, the tuning of solvers, and the training of machine learning models for MILP resolution all hinge on access to extensive, diverse, and representative data. Yet compared to the abundant naturally occurring data in image and text realms, MILP is markedly data deficient, underscoring the vital role of synthetic MILP generation. We present \dgm, a deep generative framework based on variational auto-encoder (VAE), adept at extracting deep-level structural features from highly limited MILP data and producing instances that closely mirror the target data. Notably, by leveraging the MILP duality, \dgm{} guarantees a correct and complete generation space as well as ensures the boundedness and feasibility of the generated instances. Our empirical study highlights the novelty and quality of the instances generated by \dgm{} through two distinct downstream tasks: (S1) Data sharing, where solver solution times correlate highly positive between original and \dgm-generated instances, allowing data sharing for solver tuning without publishing the original data; (S2) Data Augmentation, wherein the \dgm-generated instances bolster the generalization performance of machine learning models tasked with resolving MILP problems\footnote{code is available at \url{https://github.com/Graph-COM/DIG_MILP.git}}.
\end{abstract}

\section{Introduction}
\label{sec:intro}
Mixed integer linear programming (MILP) is a prominent problem central to operations research (OR)~\citep{achterberg2013mixed,wolsey2020integer}. It forms the basis for modeling numerous crucial industrial applications, including but not limited to supply chain management~\citep{hugos2018essentials}, production scheduling~\citep{branke2015automated}, financial portfolio optimization~\citep{mansini2015linear}, and network design~\citep{al2017technologies,radosavovic2020designing}. 
This article aims to answer the question: \textit{How can one produce a series of high-quality MILP instances?} The motivation behind this inquiry is illustrated through the subsequent scenarios:

\textbf{(Scenario I).} In industry, clients from real-world business seek specialized companies to develop or fine-tune intricate solver systems~\citep{cplex2009v12,bestuzheva2021scip,gurobi} for solving MILP problems. The empirical success of the systems heavily depends on well-tuned hyper-parameters for the solvers, which demands ample and representative testing cases that accurately reflect the actual cases. However, real data is often scarce during the early stages of a business. In addition, clients are typically reluctant to publish data that might encompass some specific information (e.g., schedules or contract stipulations for flight arrangement~\citep{richards2002aircraft, roling2008optimal}, platform costs or audience data for ad placements~\citep{rodriguez2016automatic}). These scenarios intensify the emergent need for generating instances that closely mirror the target data.

\textbf{(Scenario II).} In academia, beyond the improvement of algorithms~\citep{lawler1966branch, gamrath2015progress} for solving MILP, recent efforts have explored the use of machine learning (ML), which bypasses the need for expert knowledge and instead leverages historical data to foster accelerated resolutions~\citep{khalil2016learning,khalil2017learning,nair2020solving}. Notably, the efficacy of ML-driven approaches relies on high-quality, large-capacity, and representative training data~\citep{lu2022roco}.

\begin{wrapfigure}{r}{0.5\textwidth}
    \centering
    \includegraphics[width = 0.5\textwidth]{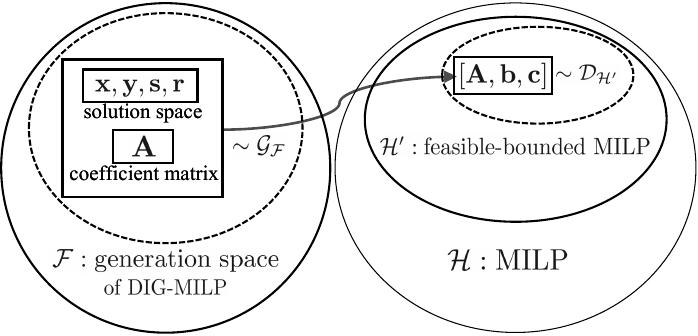}
    \vspace{-0.6cm}
    \caption{\dgm{} generates feasible-bounded instances that resemble the target MILP data from distribution $\mathcal{D}_{\mathcal{H'}}$ by learning to sample the coefficient matrix along with a set of feasible solutions for both the primal format and dual format of the linear relaxation from the corresponding distribution $\mathcal{G}_{\mathcal{F}}$. See detailed explanations in Section.~\ref{sec:methodology}.}
    \label{fig:dgm_overview}
    \vspace{-0.1cm}
\end{wrapfigure}

Given the scarce availability of real-world datasets~\citep{gleixner2021miplib}, the scenarios mentioned above underscore the motivation to synthetically generate novel instances that resemble the limited existing MILP data. To meet the requirements of both the industrial and academic sectors, the challenge in synthetic MILP generation lies in ensuring feasibility-boundedness, representativeness, and diversity. 
``Feasibility-boundedness'' refers to the general expectation in business scenarios that MILP problems should be bounded and feasible, where, otherwise, the applicability of the modeling and the corresponding real-world problem would diminish significantly. ``Representativeness'' means that the generated data should closely mirror the original data in terms of the problem scale and modeling logic (the structure of objective and constraints). ``Diversity'' implies that the generation method should be capable of catering to different problem formulations and encompassing extreme cases such as large dynamic ranges or degeneracy~\citep{gamrath2020exploratory}. Existing methods for MILP generation fall short of fulfilling the criteria above: Some are tailored to specific problems (e.g., knapsack~\citep{hill2011test} and quadratic assignment~\citep{drugan2013instance}), requiring substantial expert effort for domain knowledge, hence struggling to generalize across different problems and failing in diversity; The others sample new instances in an embedding space by manipulating certain statistics~\citep{smith2015generating,bowly2020generation,bowly2019stress}. The latter methods, which model MILPs' coefficients with simple distributions such as Gaussian distributions, generate instances with very limited structural characters, leading to not being representative enough.

With this in mind, we introduce \dgm{}, a deep generative framework for MILP based on variational auto-encoder (VAE)~\citep{kingma2013auto,kipf2016variational}. By employing deep neural networks (NNs) to extract the structural information, \dgm{} enables the generation of ``representative'' data that resembles the original samples without expert knowledge. \dgm{} leverages the MILP duality theories to ensure the feasibility and boundedness of each generated instance by controlling its primal format and the dual format of its linear relaxation having at least a feasible solution, which achieves the ``feasibility-boundedness''  of the generated data.
Moreover, any feasible-bounded MILP is inside the generation space of \dgm, meeting the demand for ``diversity''. An illustration of \dgm's generation strategy is shown in Figure.~\ref{fig:dgm_overview}. Recognizing the limited original data along with the requirements on scalability and numerical precision in MILP generation, instead of generating from scratch, \dgm{} iteratively modifies parts of existing MILPs, allowing control on the degree of structural similarity towards the original data. 

We conduct two downstream tasks to validate the quality and novelty of \dgm-generated instances, corresponding to the motivation of data generation in industry and in academia respectively. Specifically, the first task involves MILP problem sharing for solver hyper-parameter tuning without publishing original data. Across four distinct problems, the solution time of solver SCIP~\citep{bestuzheva2021scip} exhibits a highly positive correlation between the \dgm-generated instances and the original data w.r.t. different hyper-parameter sets.
The other task is envisioned as data augmentation, where the generated instances assist in training NNs to predict the optimal objective values for MILP problems~\citep{chen2022representing}. Models trained on datasets augmented with \dgm-generated instances demonstrate enhanced generalization capabilities.

\section{Related Work}
In the following, we discuss works on MILP generation.
In light of Hooker's proposals~\citep{hooker1994needed,hooker1995testing}, research on MILP generation diverges into two paths. The first focuses on leveraging expert domain knowledge to create generators for specific problems such as set covering~\citep{balas1980set}, traveling sales person~\citep{pilcher1992partial,vander1995heuristic}, graph colouring~\citep{culberson2002graph}, knapsack~\citep{hill2011test}, and quadratic assignment~\citep{drugan2013instance}. This specificity causes poor generalization across different problems and thus fails diversity.
In contrast, the second path aims at generating general MILPs. \cite{asahiro1996random} propose to generate completely random instances, which is inadequate for producing instances with specific distributional features~\citep{hill2000effects}. \cite{bowly2019stress,bowly2020generation} attempt to sample feasible instances similar to target data by manually controlling distributions in an embedding space. The formulation used in~\citep{bowly2019stress} to guarantee feasibility is similar to our method, however, its manual feature extraction and statistic control by simple distributions leads to instances with too limited structural characteristics to be representative enough.
Inspired by~\cite{bowly2019stress}, \dgm{} generates instances from the solution space and uses DNNs to dig out more details, aiming to delineate the structural attributes more precisely.

\section{Methodology}
\label{sec:methodology}
We start by providing a preliminary background on MILP generation. Subsequently, we discuss the theoretical foundation based on which \dgm's generation strategy ensures the feasibility and boundedness of its generated instances. Finally, we delve into the training and inference process of \dgm{} along with its neural network architecture.

\subsection{Preliminaries}
Given a triplet of coefficient matrix $\mA \in \mathbb{R}^{m\times n}$, right-hand side constant $\vb \in \mathbb{R}^m$, and objective coefficient $\vc \in \mathbb{R}^n$, an MILP is defined as:
\begin{equation}
\label{equ:primal_format}
\textbf{MILP}(\mA,\vb,\vc): \quad \max_{\vx} \vc^\top \vx, \quad  \text{s.t. } \mA \ \vx \leq \vb, \ \vx \in \mathbb{Z}^{n}_{\geq 0}.
\end{equation}
To solve MILP is to identify a set of non-negative integer variables that maximize the objective function while satisfying a series of linear constraints. Merely finding a set of feasible solutions to such a problem could be NP-hard. Within the entire MILP space $\mathcal{H} =\{[\mA, \vb, \vc]: \mA \in \mathbb{R}^{m \times n}, \vb \in \mathbb{R}^{m}, \vc \in \mathbb{R}^{n}\}$, the majority of MILP problems are infeasible or unbounded. However, \textit{In real-world business scenarios, MILPs derived from practical issues are often expected to be feasible, bounded, and yield an optimal solution\footnote{Definitions of boundedness, feasibility, and optimal solution of MILP in Definition.~\ref{def:feasibility_milp}~\ref{def:boundedness_milp}~\ref{def:optimal_solution_milp} in the appendix.}}, otherwise the modeling for the practical problem would be meaningless. Therefore, we are particularly interested in MILPs from the following space that corresponds to feasible-bounded instances only: \footnote{Narrowing from the real domain to the rational domain is common in MILP studies to avoid cases where an MILP is feasible and bounded but lacks an optimal solution~\cite{schrijver1998theory}. For example, $\min \sqrt{3} x_1 - x_2, \ \text{s.t.} \ \sqrt{3} x_1 - x_2 \geq 0, x_1 \geq 1, \vx \in \mathbb{Z}^2_{\geq 0}$. No feasible solution has objective equal to zero, but there are feasible solutions with objective arbitrarily close to zero.}
\[\mathcal{H'}:=\{ [\mA, \vb, \vc]: \mA \in \sQ^{m \times n}, \vb \in \sQ^{m}, \vc \in \sQ^{n} \text{ and MILP}(\mA, \vb, \vc) \text{ is feasible and bounded.} \}.\]

Suppose a target MILP dataset $D$ that models a particular business scenario is sampled from a distribution $\mathcal{D}_{\mathcal{H'}}(\mA, \vb, \vc)$ defined on $\gH'$, the task of MILP instance generation is to approximate the distribution $\mathcal{D}_{\mathcal{H'}}$ and sample novel MILP instances from it.

\subsection{\dgm{} with Feasibility Guarantee}

\label{sec:feasibility_guarantee}

 An intuitive idea for MILP generation is to directly sample $[\mA, \vb, \vc]$ from $\gD_{\gH'}$, which is practically hard to implement as it's hard to guarantee the generated instance to be feasible-bounded. 

According to MILP duality theories, we observe that as long as \dgm{} could ensure that a generated instance's primal format $\text{MILP}(\mA, \vb, \vc)$ and the dual format of its linear relaxation $\text{DualLP}(\mA, \vb, \vc)$ (as defined in Equation.~\ref{equ:dual_format}) both have at least one set of feasible solutions, then the newly generated instance will be guaranteed to be feasible-bounded (as proved in Proposition.~\ref{thm:main}). 
\begin{equation}
\label{equ:dual_format}
\begin{aligned}
\textbf{DualLP}(\mA,\vb,\vc): \quad \min_{\vy} \vb^\top \vy, \quad \text{s.t. }  \ \mA^\top \vy \geq \vc, \  \vy \geq 0,
\end{aligned}
\end{equation}
To guarantee the existence of feasible solutions to both problems, inspired by \citep{bowly2019stress}, 
we propose to sample the instances from another space $\mathcal{F}$, where
\begin{equation}
\label{equ:generation_space}
\begin{aligned}
\mathcal{F} := \{[\mA, \vx, \vy, \vs, \vr]: \mA \in \mathbb{Q}^{m \times n}, \vx \in \mathbb{Z}^{n}_{\geq 0}, \vy \in \mathbb{Q}^{m}_{\geq 0}, \vs \in \mathbb{Q}^n_{\geq 0}, \vr \in \mathbb{Q}^{m}_{\geq 0}\}.
\end{aligned}
\end{equation}
$\mathcal{F}$ defines an alternative space to represent feasible-bounded MILPs, with each element $[\mA, \vx, \vy, \vs, \vr]$ consisting of the coefficient matrix $\mA$ along with a set of feasible solutions $\vx, \vy$ to $\text{MILP}(\mA, \vb, \vc)$ and $\text{DualLP}(\mA, \vb, \vc)$, respectively, where $\vb, \vc$ are determined by the corresponding slacks $\vs, \vr$ via the equalities defined in Equation.~\ref{equ:slack_variables}. By leveraging this idea, \dgm{} aims to learn a distribution $\mathcal{G}_{\mathcal{F}}$ over the space of $\mathcal{F}$ to sample $[\mA, \vx, \vy, \vs, \vr]$, which can be further transformed into  $[\mA, \vb, \vc]$ that defines an MILP problem based on Equation.~\ref{equ:slack_variables}.
\begin{equation}
\label{equ:slack_variables}
\begin{aligned}
    \textbf{Slack Variables:} \quad \mA \vx + \vr = \vb, \mA^\top \vy - \vs = \vc, \quad \text{where} \ \vr \in \mathbb{Q}^{m}_{\geq 0}, \vs \in \mathbb{Q}^{n}_{\geq 0}
\end{aligned}
\end{equation}
Such a generation strategy offers theoretical guarantees on the boundedness and feasibility of the generated instances, ensuring the ``feasibility-boundedness'' of the produced data. Moreover, all the feasible and bounded MILPs in $\gH'$ correspond to at least a tuple $[\mA, \vx, \vy, \vs, \vr]$. Therefore, this procedure also offers theoretical assurances for the capability to produce ``diverse'' instances. These points are formally stated in Proposition.~\ref{thm:main}. See detailed proof in~\ref{app:thm_main_proof} in the appendix.

\begin{proposition}[\textbf{Boundedness and Feasibility Guarantee of \dgm{}}]
\label{thm:main}
    \dgm{} guarantees to produce feasible-bounded MILP instances only, and any feasible-bounded MILP could be generated by \dgm{}. In other words, it holds that 
        $\gH' = \Big\{ [\mA,\vb,\vc]: \vb = \mA \vx + \vr, \vc = \mA^\top \vy - \vs, [\mA, \vx, \vy, \vs, \vr] \in \gF \Big\}.$
\end{proposition}

\subsection{Generation Process and Architecture}
Having shown the equivalence between sampling from space $\gF$ and $\gH'$, we then present how \dgm{} learns a distribution $\gG_{\gF}$ to sample $[\mA, \vx, \vy, \vx, \vr]$ from. 
We encode each $[\mA, \vx, \vy, \vs, \vr]$ as a variable-constraint (VC) bipartite graph $G(\mathcal{V}, \mathcal{C}, \mathcal{E})$: On side $\mathcal{V}$, each node in $\{v_1,...,v_m\}$ corresponds to a variable, while on $\mathcal{C}$ side, each node in $\{c_1,...,c_m\}$ represents a constraint. Edges in $\mathcal{E}$ connect constraints to variables according to the non-zero entries in the coefficient matrix $\mA$, implying that $\mA$ serves as the adjacency matrix of graph $G$. The input features of nodes and edges are detailed in Table.~\ref{tab:vc_encoding}. With this graph representation, we transform the MILP generation challenge into a graph generation task. \dgm{} iteratively modifies part of the original graph to produce new graphs. 
\begin{figure}[t]
    \centering
    \includegraphics[width = 0.8 \textwidth]{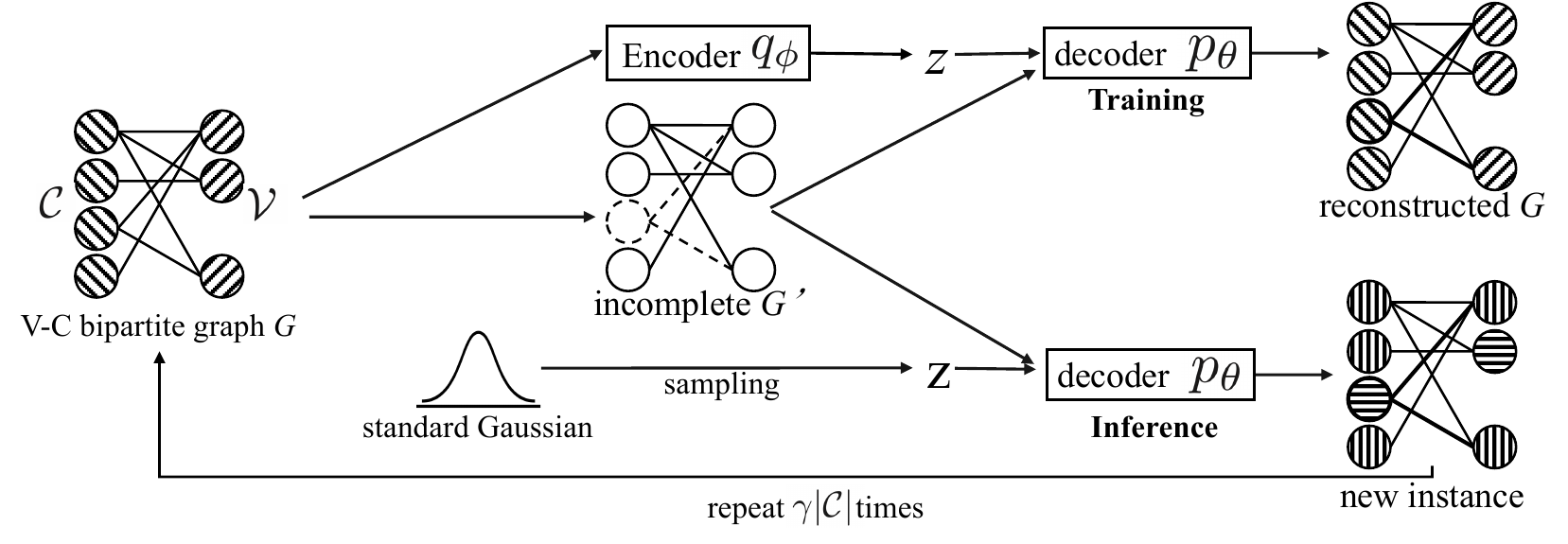}
    \vspace{-0.3cm}
    \caption{The training and inference pipeline of \dgm. In each training step, \dgm{} removes a random constraint node, its connected edges, along with the solution and slack features on all the nodes, resulting in an incomplete graph $G'$. The training objective of \dgm{} is to reconstruct $G$ from $G'$ and $\vz$ sampled by the encoder $q_{\phi}$. As to inference, \dgm{} employs an auto-regressive approach, generating new instances by iteratively modifying the existing MILPs.}
    \label{fig:dgm_pipeline}
    \vspace{-0.3cm}
\end{figure}


\begin{wraptable}{r}{0.43 \textwidth}
\renewcommand{\arraystretch}{1.1} 
\setlength{\tabcolsep}{4pt}
\vspace{-0.1cm}
\centering
\captionsetup{font=small} 
\caption{The input encoding into $G$ from MILP.}
\vspace{-0.3cm}
\label{tab:vc_encoding}
\begin{tabular}{cc}
\hline
object & feature \\ \hline
\multirow{3}{*}{\begin{tabular}[c]{@{}c@{}}constraint-\\ nodes:\\ $\mathcal{C} = \{c_1...c_m\}$\end{tabular}} & all 0's \\
 & $\vy = [y_1,...,y_m]^\top$ \\
 & $\vr = [r_1,...,r_m]^\top$ \\ \hline
\multirow{3}{*}{\begin{tabular}[c]{@{}c@{}}variable-\\ nodes:\\ $\mathcal{V}=\{v_1...v_n\}$\end{tabular}} & all 1's \\
 & $\vx = [x_1,...,x_n]^\top$ \\
 & $\vs = [s_1,...,s_n]^\top$ \\  \hline
edge $\mathcal{E}$ & non-zero weights in $\mA$ \\ \hline
\end{tabular}
\vspace{-0.5cm}
\end{wraptable}

\textbf{Generation pipeline}
We display the training and inference pipeline in Figure.~\ref{fig:dgm_pipeline}. 
As illustrated in Algorithm.~\ref{alg:dgm_train}, on each training step of \dgm, we randomly select and remove a constraint node $c_i$ (corresponding to the $i$-th constraint) from the bipartite graph, along with all its connected edges $\mathcal{E}_G(c_i)$. Concurrently, we erase the features of the solution space $\vx, \vy, \vs, \vr$ on all the nodes, resulting in an incomplete graph $G'({\mathcal{C} \backslash c_i}_{- \vy, \vs}; \mathcal{V}_{-\vx, \vr}; \mathcal{E} \backslash \mathcal{E}_G(c_i))$. The training objective is to learn \dgm{} to reconstruct $G$ from the given $G'$ by maximizing the log likelihood:

\begin{equation}
\centering
\label{equ:max_likelihhod}
\argmax_{\theta,\phi} \mathbb{E}_{G \sim D} \mathbb{E}_{G' \sim p(G'|G)} \log \mathbb{P}(G|G';\theta,\phi),
\end{equation}
where $p(G'|G)$ refers to randomly removing structures along with features to produce the incomplete graph, $\theta$ and $\phi$ denote the NN parameters. To address the dependency issues and foster diversity into generation, we adhere to the standard procedure in VAEs~\citep{kingma2013auto,kipf2016variational} by introducing a latent variable $\vz = [z_1,...,z_{m+n}]$ with the assumption that $\vz$ is independent with $G'$. Utilizing the principles of the variational evidence lower bound (ELBO), we endeavor to maximize the training objective through the optimization of the ensuing loss function:
\begin{small}
\begin{equation}
\label{equ:loss}
    \min_{\theta,\phi} \mathcal{L}_{\theta,\phi} = \mathbb{E}_{G \sim D} \mathbb{E}_{G' \sim p(G'|G)} \left[ \alpha \mathbb{E}_{\vz \sim q_{\phi}(\vz|G)} [- \log p_{\theta}(G| G',\vz)] + \mathcal{D}_{KL}[q_{\phi}(\vz|G) \Vert \mathcal{N}(0,I)] \right], 
\end{equation}
\end{small}where the decoder parameterized by $\theta$ is to adeptly reconstruct graph $G$ based on the latent variables $\vz$ and the incomplete graph $G'$; the encoder parameterized by $\phi$ is to depict the posterior distribution of $\vz$ which is required to align with the prior standard Gaussian. The hyper-parameter $\alpha$ functions as a balancing factor between the two parts of the loss. See detailed derivation of the loss in ~\ref{app:derivation_loss} in the appendix. During training, \dgm{} modifies only one constraint of the data at a time. In the inference phase, the graph rebuilt after removing a constraint can be fed back as an input, allowing iterative modifications to the original data. The number of iterations controls the degree of structural similarity to the original problem. The inference procedure is shown in Algorithm.~\ref{alg:dgm_inference}, where $\gamma |\mathcal{C}$ denotes the number of iterations to remove a constraint.

\vspace{-0.3cm}
\begin{minipage}{0.48\textwidth}
\begin{algorithm}[H]
\caption{\dgm{} Training}
\label{alg:dgm_train}
\begin{algorithmic}[1]
\Require: dataset $D$, epoch $N$, batch size $B$
\State Solve MILPs for $\{[\vx,\vy,\vs,\vr]\}$ over $D$
\State Encode MILPs into graphs $\{G(\mathcal{V}, 
\mathcal{C}, \mathcal{E})\}$
\For{epoch=1,...,N}
\State Allocate empty batch $\mathcal{B} \gets \emptyset$
\For{idx=1,...,$B$}
\State $G \sim D$; \  $G' \sim p(G'|G)$
\State $\mathcal{B} \gets \mathcal{B} \cup \{(G, G')\}$ 
\State Encode $\vz \sim q_{\phi}(\vz |G)$
\State Decode $G \sim  p_{\theta}(G|G',\vz)$
\State Calculate $\mathcal{L}_{\theta, \phi}(G,G')$
\EndFor
\State $\mathcal{L}_{\theta, \phi}$ $\gets$ $\frac{1}{B} \sum_{(G,G')\in \mathcal{B}} \mathcal{L}_{\theta, \phi}(G,G')$
\State Update $\phi, \theta$ by minimizing $\mathcal{L}_{\theta, \phi}$
\EndFor
\State \textbf{return} $\theta, \phi$
\end{algorithmic}
\end{algorithm}
\end{minipage}
\hfill
\begin{minipage}{0.48\textwidth}
\begin{algorithm}[H]
\caption{\dgm{} Inference}
\label{alg:dgm_inference}
\begin{algorithmic}[1]
\Require: dataset $D$, batch size $B$, constraint replace rate $\gamma$
\vspace{0.15cm}
\State Solve MILPs for $\{[\vx,\vy,\vs,\vr]\}$ over $D$
\State Encode MILPs into graphs $\{G(\mathcal{V}, \mathcal{C}, \mathcal{E})\}$
\State Allocate empty batch $\mathcal{B} \gets \emptyset$
\For{id=1,...,$B$}
\State $G \sim D$
\For{t=1,...,$\gamma |\mathcal{C}|$}
\State $G' \sim p(G'|G)$
\State $\vz \sim \mathcal{N}(0,I)$
\State Decode $\tilde{G} \sim  p_{\theta}(\tilde{G}|G', \vz)$
\State $G \gets \tilde{G}$
\EndFor
\State $\mathcal{B} \gets \mathcal{B} \cup G$
\EndFor
\State \textbf{return} new instance batch $\mathcal{B}$
\end{algorithmic}
\end{algorithm}
\end{minipage}

\textbf{Neural Network Architecture} 
For both the encoder and decoder, we employ the same bipartite graph neural network (GNN) as delineated in ~\citep{gasse2019exact} as the backbone. The encoder encodes the graph into the distribution of the latent variable $\vz$, as depicted in the following equation:
\begin{equation}
    q_{\phi}(\vz|G) = \prod_{u \in \mathcal{C} \cup \mathcal{V}} q_{\phi}(\vz_u|G), \quad \quad \quad  q_{\phi}(\vz_u|G) = \mathcal{N} (\mu_{\phi}(\vh_u^G), \Sigma_{\phi}(\vh_u^G)),
\end{equation}
where $\vz_u$ is conditionally independent with each other on $G$, $\vh^G = \text{GNN}_{\phi}(G)$ denotes the node embeddings of $G$ outputted by the encoder backbone, $\mu_\phi$ and $\Sigma_{\phi}$ are two MLP layers that produce the mean and variance for the distribution of $\vz$. 
The decoder connects seven parts conditionally independent on the latent variable and node representations, with detailed structure as follows:
\begin{equation}
\resizebox{0.93\textwidth}{!}{$
\begin{aligned}
    p_{\theta}(G|G',\vz) = & \ p_{\theta}(d_{c_i}|\vh^{G'}_{c_i},\vz_{c_i}) \cdot \prod_{u \in \mathcal{V}} p_{\theta}(e(c_i,u)|\vh^{G'}_{\mathcal{V}}, \vz_{\mathcal{V}}) \cdot \prod_{u \in \mathcal{V}:e(c_i,u)=1}  p_{\theta}(w_{c_i}|\vh^{G'}_{\mathcal{V}},\vz_{\mathcal{V}})\\
    & \cdot \prod_{u \in \mathcal{C}}p_{\theta}(\vy_{u}|\vh^{G'}_{\mathcal{C}},\vz_{\mathcal{C}}) p_{\theta}(\vr_u|\vh^{G'}_{\mathcal{C}},\vz_{\mathcal{C}})  \cdot \prod_{u \in \mathcal{V}}p_{\theta}(\vx_u|\vh^{G'}_{\mathcal{V}},\vz_{\mathcal{V}})p_{\theta}(\vs_u|\vh^{G'}_{\mathcal{V}},\vz_{\mathcal{V}}),
\end{aligned}$}
\end{equation}
where $\vz_{\mathcal{C}}, \vh^{G'}_{\mathcal{C}}$ denotes the latent variable and node representations on side $\mathcal{C}$ outputted by the decoder backbone, while $\vz_{\mathcal{V}}, \vh^{G'}_{\mathcal{V}}$ signifies those on side $\mathcal{V}$; $d_{c_i}$ predicts the degree of the deleted node $c_i$; $e(c_i,\cdot)$ denotes the probability of an edge between $c_i$ and a node on side $\mathcal{V}$; $w_{c_i}$ is the edge weights connected with $c_i$; $\vx, \vy, \vs, \vr$ are value of the solution and slacks. We use separate layers of MLP to model each part's prediction as a regression task. We optimize each part of the decoder with the Huber Loss~\citep{huber1992robust}.
See Section.~\ref{app:implemnentation_dgm} in the appendix for more details.

\vspace{-0.3cm}
\section{Numerical Evaluations}
\vspace{-0.3cm}
In this section, we first delineate the experimental setup. Then we calculate the structural statistical similarity between generated and original instances. Subsequently, we evaluate \dgm{} with two downstream tasks: \emph{(i)} MILP data sharing for solver tuning and \emph{(ii)} MILP data augmentation for ML model training.
\vspace{-0.3cm}
\subsection{Settings}
\vspace{-0.1cm}
\textbf{Datasets:} We perform \dgm{} on four MILP datasets, encompassing scenarios involving simple and complex instances, a mix of small and large problem scale, varying instance quantities, and generation/collection from both synthetic and real-world sources. Specifically, we include two manually generated datasets, namely the set covering (SC) and the combinatorial auctions (CA), following the generation methodologies outlined in ~\citep{gasse2019exact}. 
The remaining two datasets, namely CVS and IIS, are from the MIPLIB2017 benchmark~\citep{gleixner2021miplib}\footnote{\url{https://miplib.zib.de/tag_benchmark.html}}, which comprises challenging instances from a large pool of problem-solving contexts. CVS pertains to the capacitated vertex separator problem on hypergraphs, while IIS mirrors real-world scenarios and resembles the set covering problems. Details are elaborated in Table.~\ref{tab:dataset_metadata}. It's worth emphasizing that for CVS and IIS, we exclusively employ the `training' data during the training of \dgm{} and all downstream models. The `testing' data is used only for downstream task evaluation.

\begin{table}[t]
\caption{Datasets Meta-data . For CVS and IIS, `training' (non-bold) instances are for \dgm{} or downstream model training, `testing' (bold) instances are used in downstream testing only.}
\vspace{-0.35cm}
\label{tab:dataset_metadata}
\renewcommand{\arraystretch}{1.3}
\resizebox{\linewidth}{!}{\begin{tabular}{c|c|c|ccccc|cc}
\hline
 & SC & CA & \multicolumn{5}{c|}{CVS} & \multicolumn{2}{c}{IIS} \\ \hline
\multirow{2}{*}{\# data} & \multirow{2}{*}{1000} & \multirow{2}{*}{1000} & \multicolumn{3}{c|}{training} & \multicolumn{2}{c|}{testing} & \multicolumn{1}{c|}{training} & testing \\ \cline{4-10} 
 &  &  & cvs08r139-94 & cvs16r70-62 & \multicolumn{1}{c|}{cvs16r89-60} & \textbf{cvs16r106-72} & \textbf{cvs16r128-89} & \multicolumn{1}{c|}{iis-glass-cov} & \textbf{iis-hc-cov} \\ \hline
\# variable & 400 & 300 & 1864 & 2112 & \multicolumn{1}{c|}{2384} & 2848 & 3472 & \multicolumn{1}{c|}{214} & 297 \\ \hline
\# constraint & 200 & $\sim$10\textasciicircum{}2 & 2398 & 3278 & \multicolumn{1}{c|}{3068} & 3608 & 4633 & \multicolumn{1}{c|}{5375} & 9727 \\ \hline
difficulty & easy & easy & \multicolumn{5}{c|}{hard} & \multicolumn{2}{c}{hard} \\ \hline
\end{tabular}}
\vspace{-0.4cm}
\end{table}

\textbf{Downstream Tasks:}
We devise two downstream applications, tailored to address distinct motivations. One motivation pertains to generating and sharing data that can substitute target instances. The other motivation involves data augmentation for better training ML models.

\textit{(S1): Data Sharing for Solver Configuration Tuning} We simulate the process where clients utilize \dgm{} to generate new instances and hand over to companies specializing in MILP solver tuning. In particular, we calculate the Pearson positive correlation of the solution times required by the SCIP~\citep{bestuzheva2021scip} solver between the generated examples and the original testing data across various hyper-parameter configurations. Should the solution time consistently demonstrate a positive correlation between the original and generated problems across varied parameter settings, it implies a consistent level of the effectiveness on the original and new instances under the same parameter configuration, which facilitates sharing data for parameter tuning.

\textit{(S2): Optimal Value Prediction via ML} Following the settings presented in ~\citep{chen2022representing}, this supervised regression task employs GNNs to express the optimal value of the objective function in an MILP. We utilize newly generated instances as a means of augmentation to formulate training datasets for ML models. For more detailed implementation, see ~\ref{app:implementation_downstream2} in the appendix.

\textbf{Solvers and Baselines:} We use the open source solver SCIP~\citep{bestuzheva2021scip} with its Python interface, namely PySCIPOpt~\citep{maher2016pyscipopt} for all the experiments. We consider two approaches as our baselines. The first, named `Bowly', aligns with~\cite{bowly2019stress} that generates MILP instances from scratch by sampling in an embedding space based on manually designed distributions. The second baseline `random' employs identical NN architectures to \dgm{} but randomizes the network's outputs, further validating the importance and efficacy of model training. For more implementation details of the baselines, please refer to ~\ref{app:implementation_baseline} in the appendix.

\vspace{-0.2cm}
\subsection{Results and Analysis}
\vspace{-0.2cm}
\subsubsection{Statistical Characteristics of the Generated Instances}
We compare the statistical metrics between the generated instances and the original instances on the SC and CA datasets. We do not calculate the statistics on the CVS and IIS due to their limited size that prevents meaningful statistical comparisons. We count nine statistic metrics in total, see Table.~\ref{app:implementation_statistic_metric} in the appendix for details. The similarity score is derived from the Jensen-Shannon (JS) divergence (the lower the better) between each metric of the generated and original data, as shown in Table.~\ref{tab:js_score}. `Bowly' shows the least similarity. As the the constraint replacement ratio $\gamma$ increases from $0.01$ to $0.50$, the table shows a decreasing similarity between new and original instances for both \dgm{} and `random', aligning with our expectation of controlling structural similarity by adjusting the number of constraint nodes to replace. Instances generated by \dgm{} more closely mirror the target data in structural statistical metrics across all $\gamma$. For detailed calculations of the similarity score and the specific values of each statistic metric, see ~\ref{app:implementation_statistic_metric} and ~\ref{app:results_statistic_metric} in the appendix.

\subsubsection{downstream task \#1: Data Sharing for Solver Configuration Tuning}

\begin{table}[t]
\centering
\caption{The similarity score $\uparrow$ between the original and generated data .}
\vspace{-0.35cm}
\label{tab:js_score}
\renewcommand{\arraystretch}{1.0}
\begin{tabular}{cccccccc}
\hline
\multicolumn{2}{c}{constraint replace rates $\gamma$} & - & 0.01 & 0.05 & 0.10 & 0.20 & 0.50 \\ \hline
\multirow{3}{*}{SC} & Bowly & 0.337 & - & - & - & - & - \\
 & random & - & 0.701 & 0.604 & 0.498 & 0.380 & 0.337 \\
 & ours & - & \textbf{0.856} & \textbf{0.839} & \textbf{0.773} & \textbf{0.652} & \textbf{0.570} \\ \hline
\multirow{3}{*}{CA} & Bowly & 0.386 & - & - & - & - & - \\
 & random & - & 0.630 & 0.566 & 0.508 & 0.432 & 0.306 \\
 & ours & - & \textbf{0.775} & \textbf{0.775} & \textbf{0.768} & \textbf{0.733} & \textbf{0.630} \\ \hline
\end{tabular}
\vspace{-0.4cm}
\end{table}

\begin{figure}[t]
  \centering
  \begin{minipage}{\textwidth}
    \centering
    \begin{subfigure}[b]{0.20\textwidth}
      \centering
      \includegraphics[width=\textwidth]{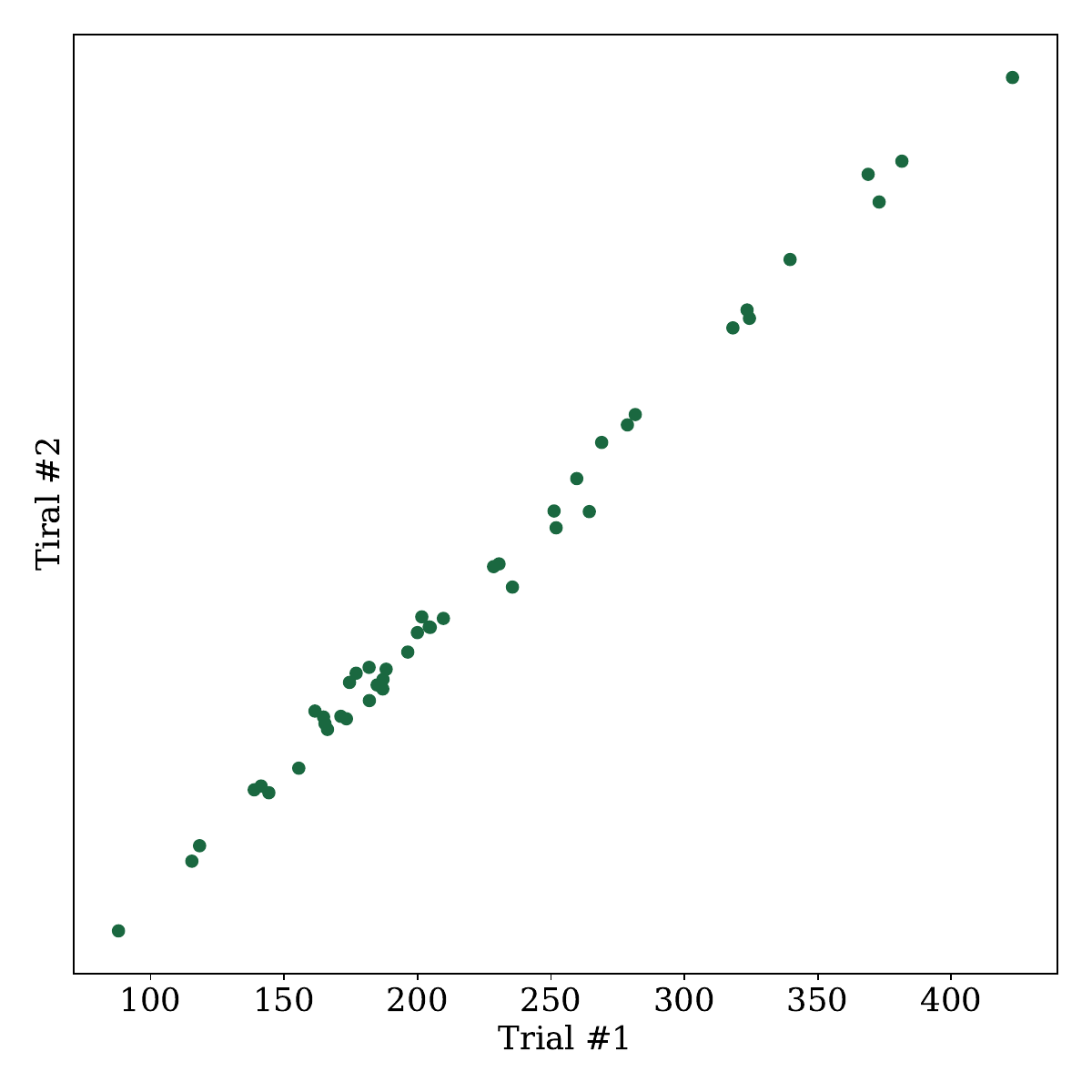}
      \vspace{-0.6cm}
      \caption{two trials}
      \label{fig:downstream1_cvs_cvs}
    \end{subfigure}
    \centering
    \hspace{0.2cm}
    \begin{subfigure}[b]{0.20\textwidth}
      \centering
      \includegraphics[width=\textwidth]{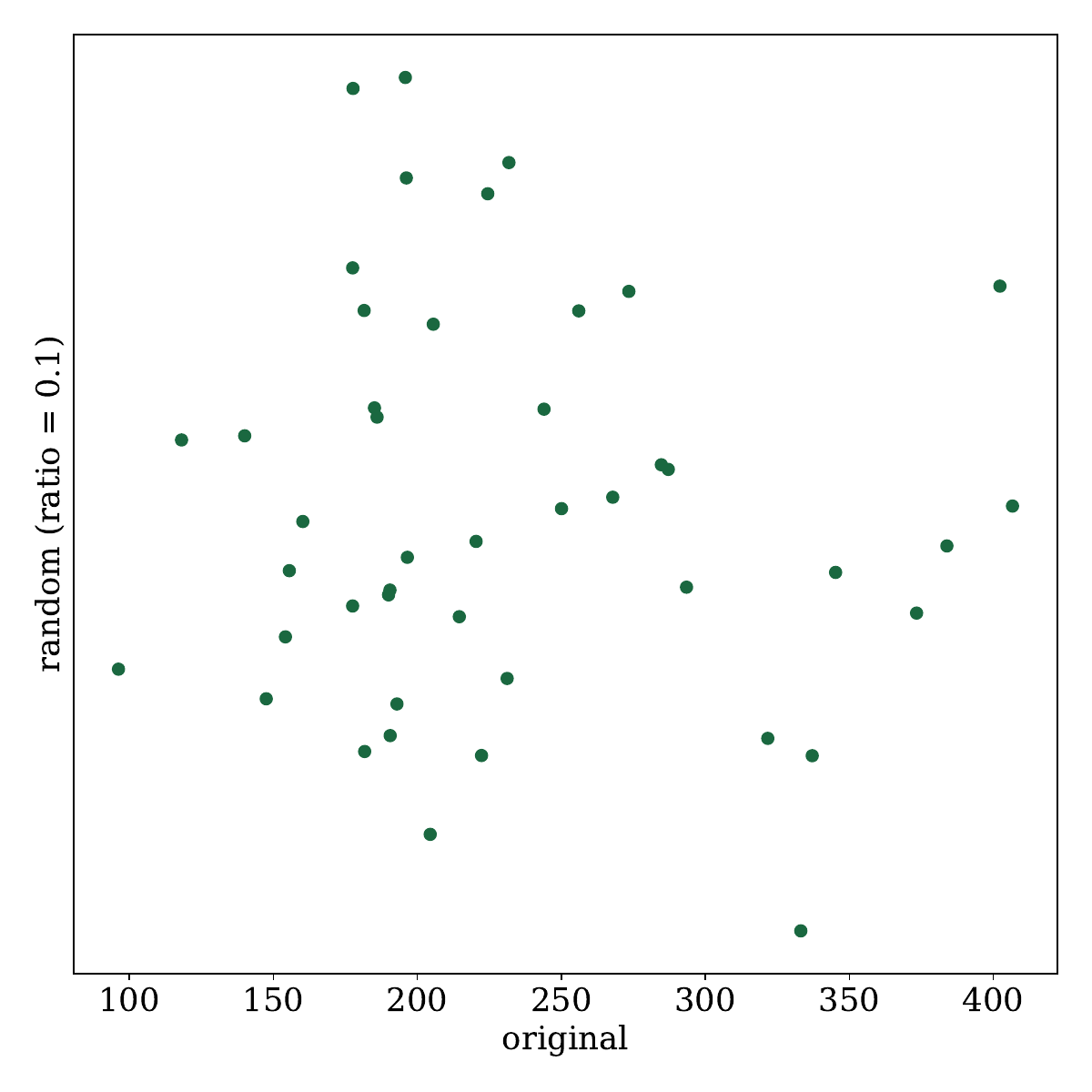}
      \vspace{-0.6cm}
      \caption{random ($\gamma$ = 0.1)}
      \label{fig:downstream1_cvs_random010}
    \end{subfigure}
    \hspace{0.2cm}
    \begin{subfigure}[b]{0.20\textwidth}
      \centering
      \includegraphics[width=\textwidth]{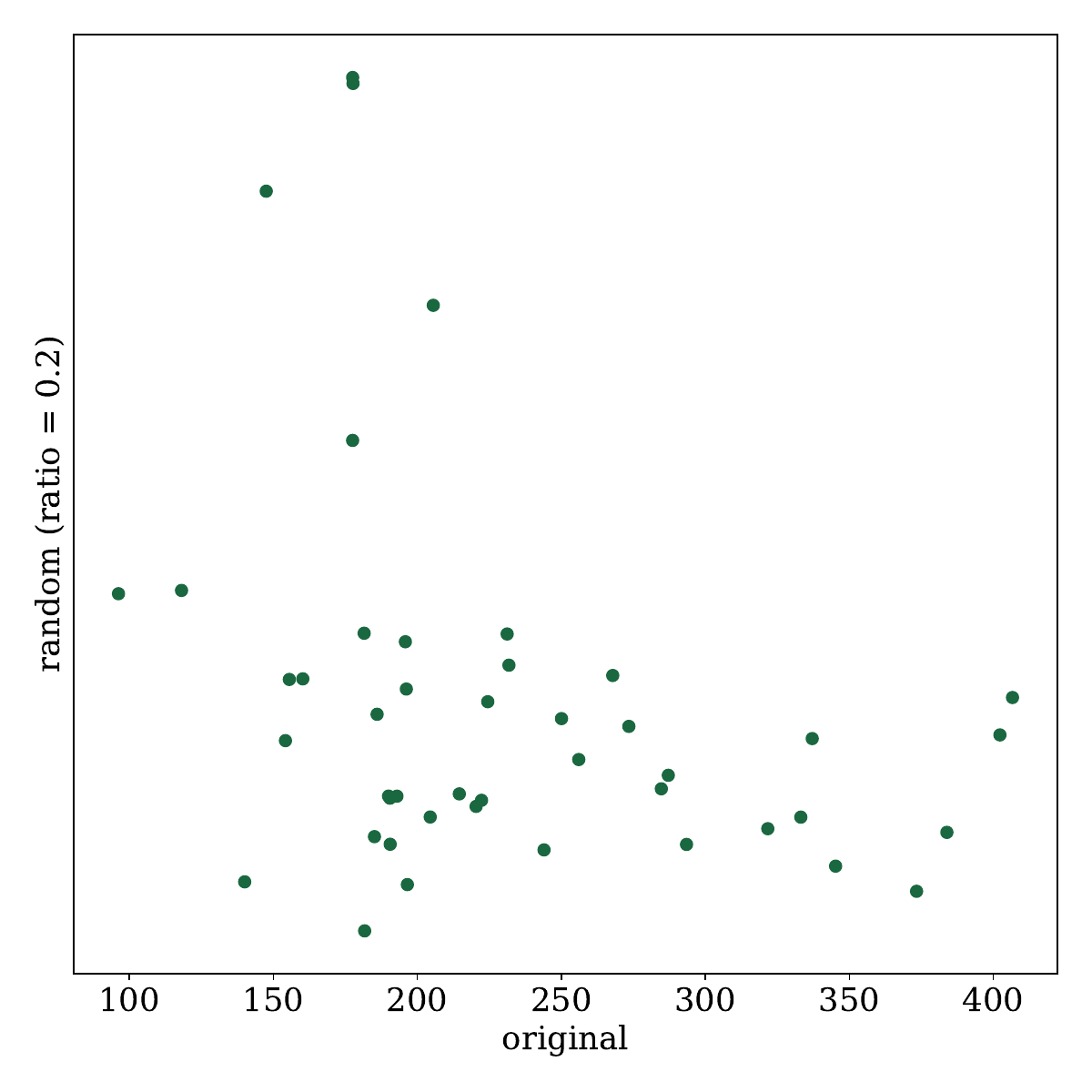}
      \vspace{-0.6cm}
      \caption{random ($\gamma$ = 0.2)}
      \label{fig:downstream1_cvs_random020}
    \end{subfigure}
    \hspace{0.2cm}
    \begin{subfigure}[b]{0.20\textwidth}
      \centering
      \includegraphics[width=\textwidth]{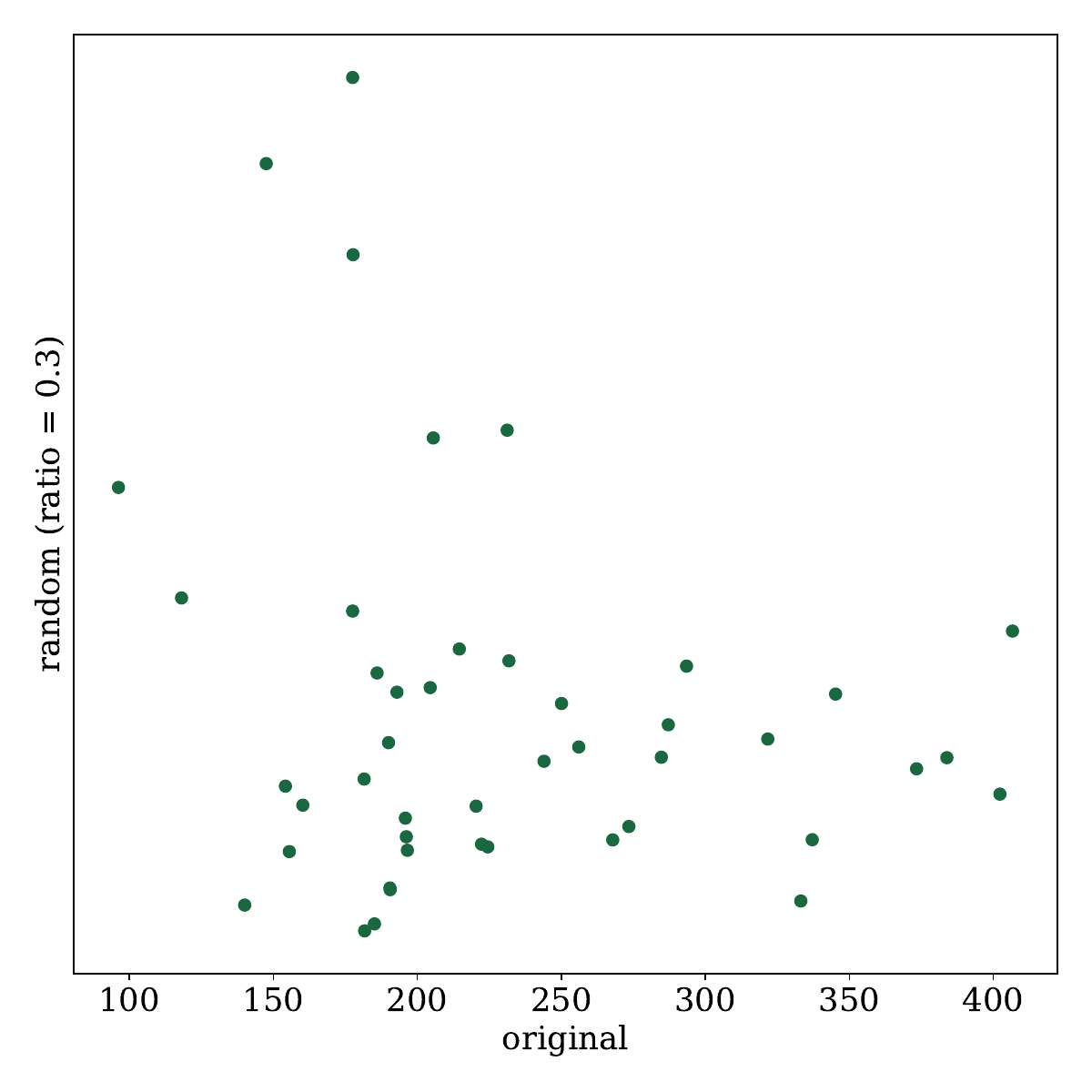}
      \vspace{-0.6cm}
      \caption{random ($\gamma$ = 0.3)}
      \label{fig:downstream1_cvs_random030}
    \end{subfigure}
    \vspace{0.0cm} 
    \vfill
    \begin{subfigure}[b]{0.20\textwidth}
      \centering
      \includegraphics[width=\textwidth]{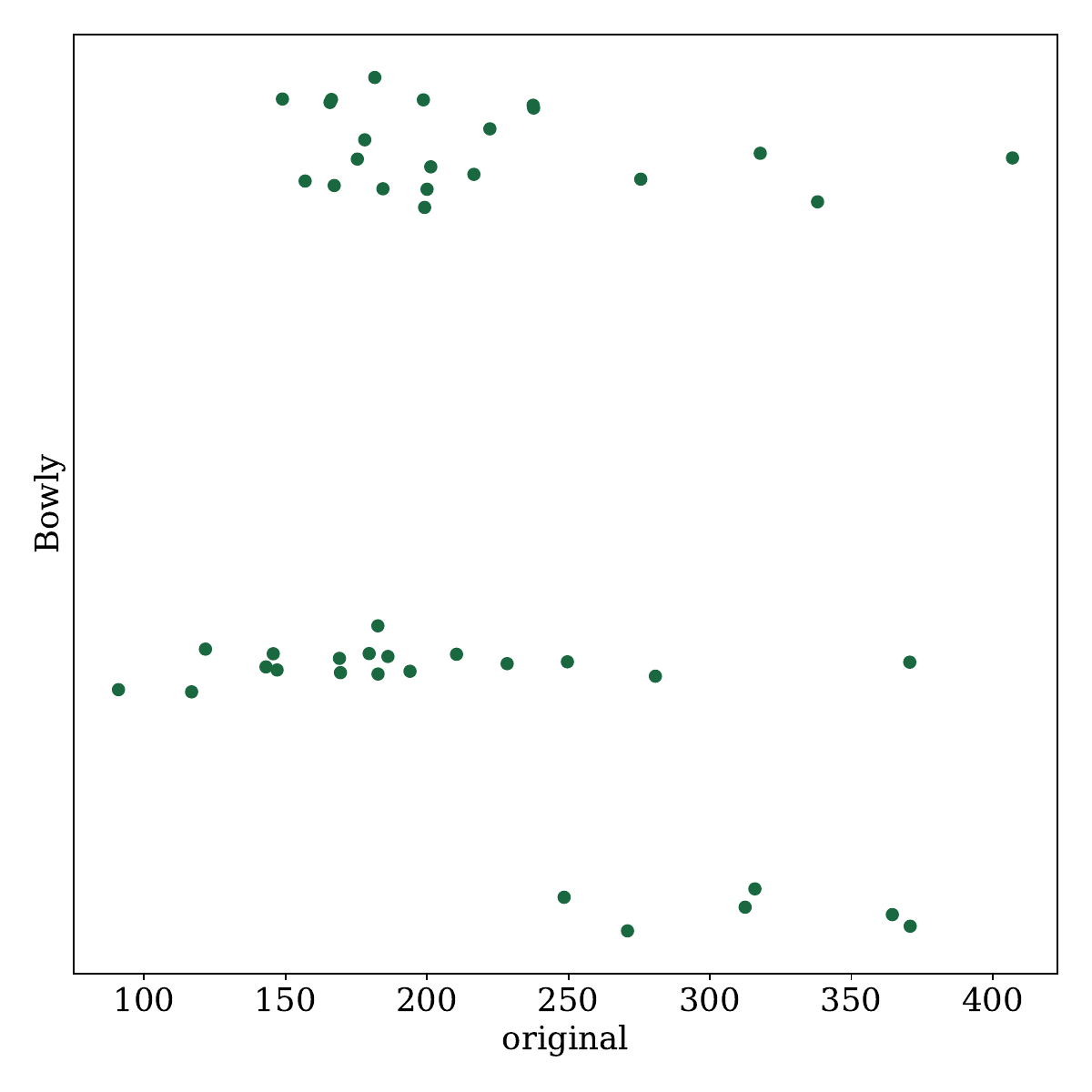}
      \vspace{-0.6cm}
      \caption{Bowly}
      \label{fig:downstream1_cvs_bowly}
    \end{subfigure}
    \hspace{0.2cm}
    \begin{subfigure}[b]{0.20\textwidth}
      \centering
      \includegraphics[width=\textwidth]{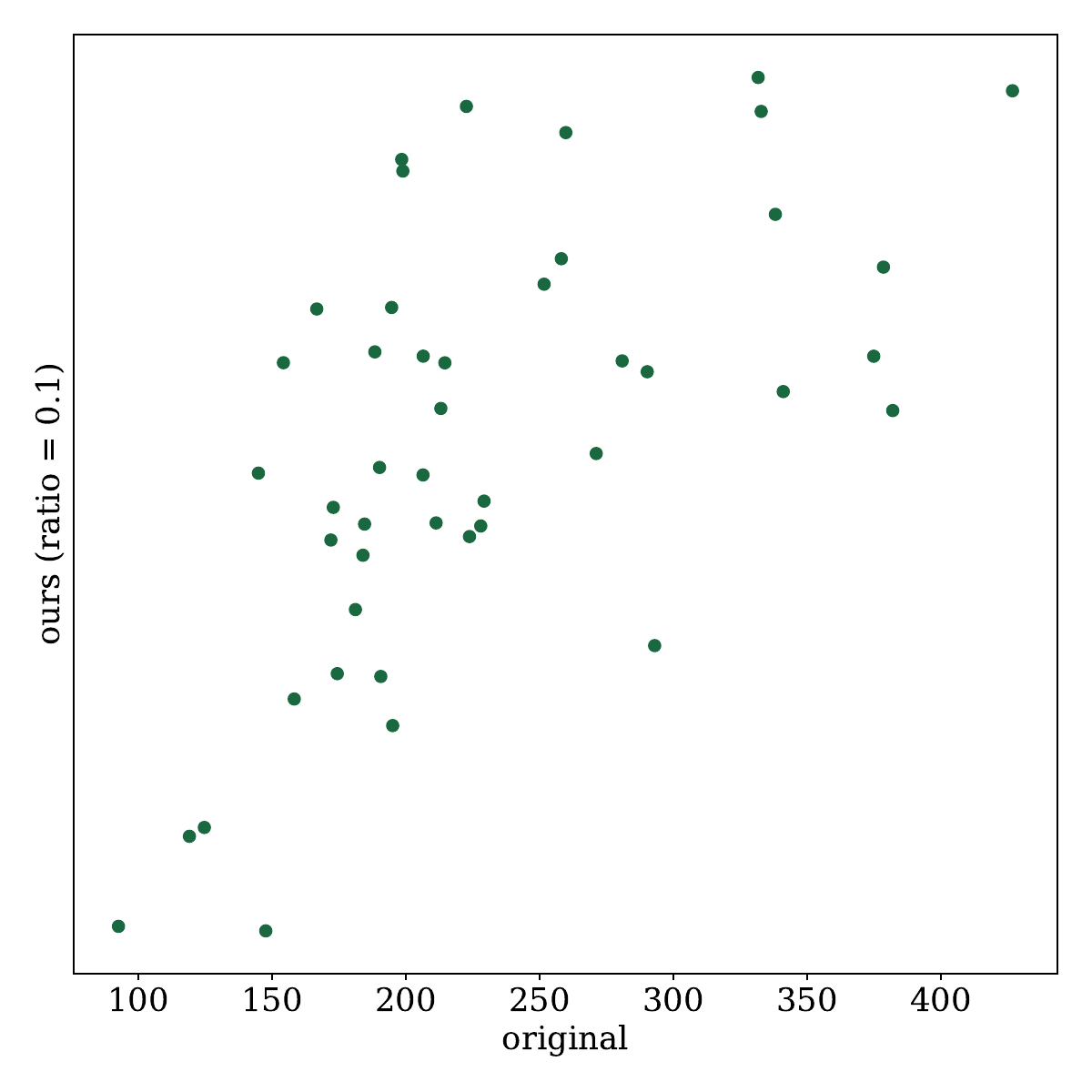}
      \vspace{-0.6cm}
      \caption{ours ($\gamma$ = 0.1)}
      \label{fig:downstream1_cvs_ours010}
    \end{subfigure}
     \hspace{0.2cm}
    \begin{subfigure}[b]{0.20\textwidth}
      \centering
      \includegraphics[width=\textwidth]{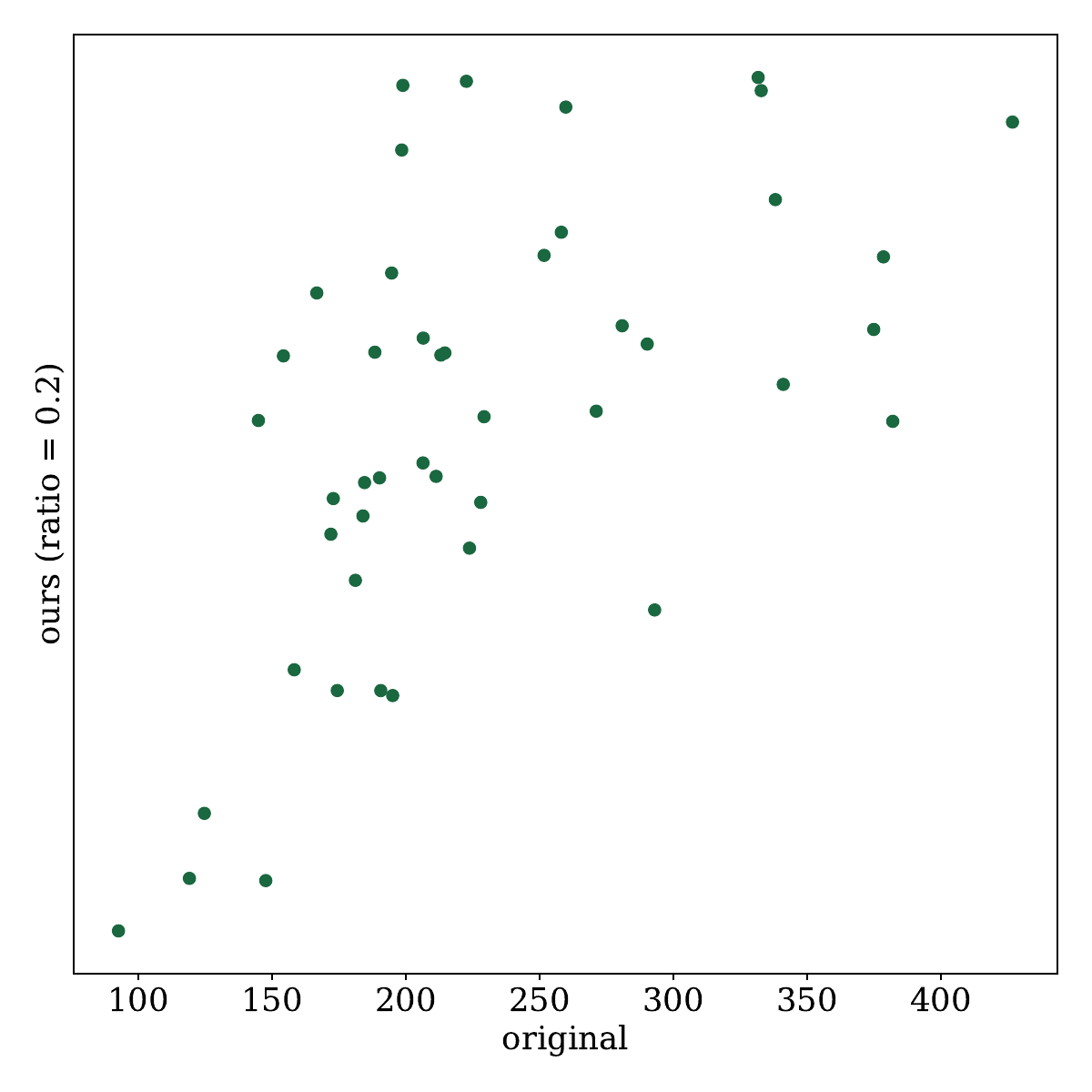}
      \vspace{-0.6cm}
      \caption{ours ($\gamma$ = 0.2)}
      \label{fig:downstream1_cvs_ours020}
    \end{subfigure}
     \hspace{0.2cm}
    \begin{subfigure}[b]{0.20\textwidth}
      \centering
      \includegraphics[width=\textwidth]{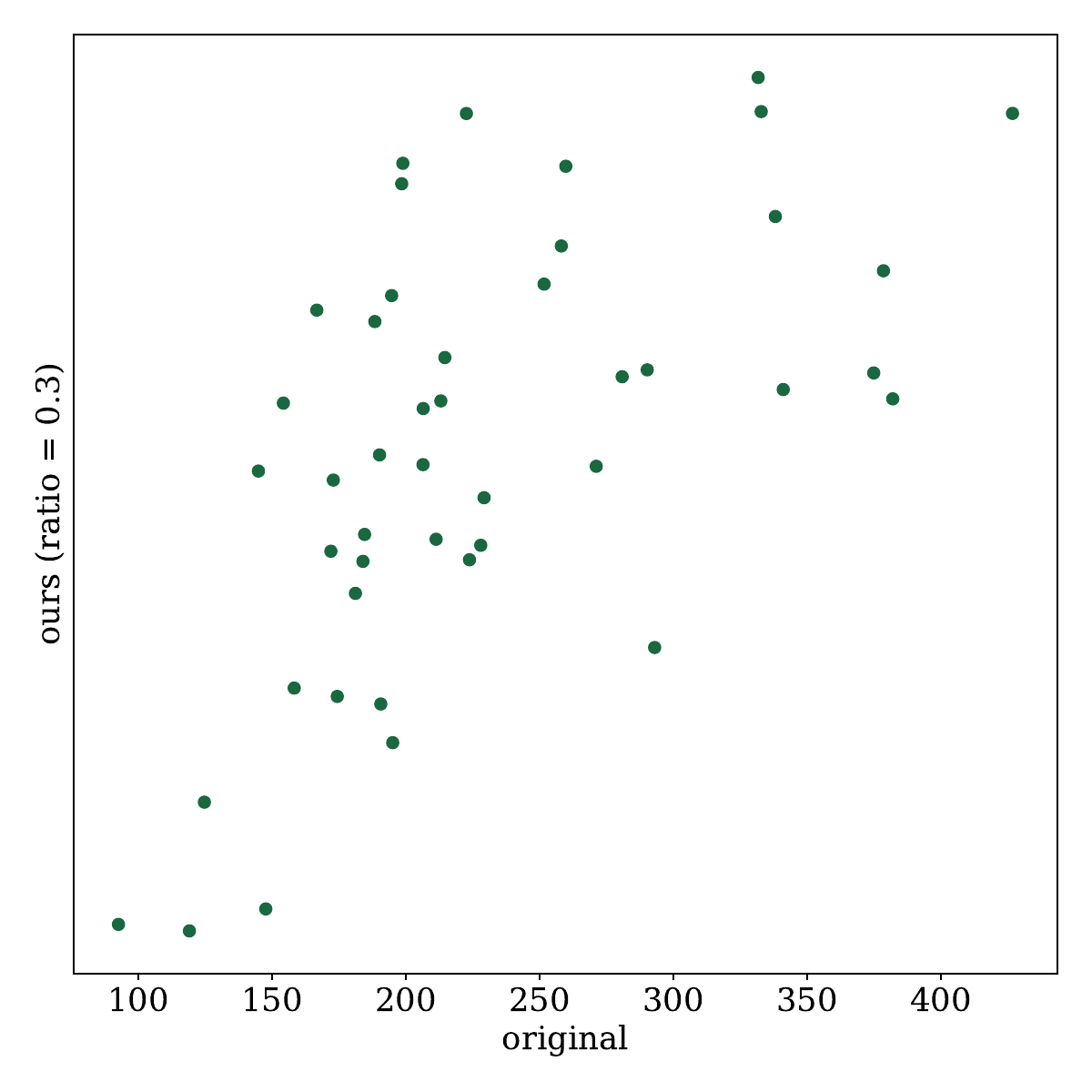}
      \vspace{-0.6cm}
      \caption{ours ($\gamma$ = 0.3)}
      \label{fig:downstream1_cvs_ours030}
    \end{subfigure}
  \end{minipage}
  \vspace{-0.2cm}
  \caption{The solution time (second) of SCIP on CVS with $45$ different hyper-parameter sets.}
    \label{fig:downstream1_visualization_cvs}
    \vspace{-0.4cm}
\end{figure}

We conduct experiments on all the four datasets. SCIP boasts an extensive array of parameters, rendering a tuning across the entire range impractical. Therefore, we adopt the reduced parameter space consistent with mainstream research on SCIP solver tuning~\citep{hutter2011sequential,lindauer2018warmstarting, lindauer2022smac3}. See Table.~\ref{tab:scip_hyper_parameters} in the appendix for detailed parameter space selection.
We employ random seed $0 - 44$ to generate $45$ distinct parameter configurations. To validate the impact of randomness on SCIP, we initiate two independent trials on the same original testing data and compare the Pearson score of solution time. As illustrated in the diagonal of Table.~\ref{tab:downstream1_original}, it clearly demonstrates a very high positive correlation for two independent trials on the same data. For subsequent experiments, each is run three times independently, with results averaged to mitigate randomness effects. We then compare the correlation of solution time on the original data across different datasets, as presented in the upper triangle of Table.~\ref{tab:downstream1_original}. We observe a certain degree of positive correlation between synthetic datasets SC and CA, as well as between MIPLIB datasets CVS and IIS, which reveals that the effectiveness of parameters may naturally exhibit some degree of generalization across similar problems. However, the correlation between synthetic and MIPLIB datasets tends to be much lower, underscoring the necessity of generating new instances for solver tuning on specific problems. Finally, we compare the positive correlation of solution time between the generated instances and the original testing instances of the same datasets, as shown in Table.~\ref{tab:downstream1_performance}. Across all four datasets, the \dgm-generated instances, exhibit the highest correlation with the testing data compared to the baselines, with the lowest p-value of significance. 
On the MIPLIB test set, \dgm-generated instances exhibits a slightly lower correlation, primarily due to the very few samples in these datasets. We visualize the correlation of solution time between the original testing data and the generated data on the CVS in Figure.~\ref{fig:downstream1_visualization_cvs}. More detailed implementation and the visualization of the other datasets can be found in ~\ref{app:implementation_downstream1} and 
Figure.~\ref{fig:downstream1_visualization_cross_datasets}-\ref{fig:downstream1_visualization_iis} in the appendix.

\begin{table}[t]
\centering
\caption{The Pearson correlation coefficient (`r') and the significance value (`p') of the SCIP solution time under $45$ different hyper-parameters on dataset-pairs.}
\vspace{-0.35cm}
\label{tab:downstream1_original}
\renewcommand{\arraystretch}{0.5}
\begin{tabular}{@{}cccccc@{}}
\toprule
 &  & SC & CA & CVS & IIS \\ \midrule
\multirow{2}{*}{SC} & r & 0.732 & 0.599 & 0.115 & 0.088 \\
 & p & 1.058e-8 & 1.351e-5 & 0.449 & 0.561 \\ \midrule
\multirow{2}{*}{CA} & r & - & 0.952 & 0.021 & 0.092 \\
 & p & - & 0.762e-24 & 0.890 & 0.545 \\ \midrule
\multirow{2}{*}{CVS} & r & - & - & 0.997 & 0.550 \\
 & p & - & - & 4.723e-53 & 9.033e-5 \\ \midrule
\multirow{2}{*}{IIS} & r & - & - & - & 0.988 \\
 & p & - & - & - & 1.563e-36 \\ \bottomrule
\end{tabular}
\end{table}

\begin{table}[t]
\vspace{-0.2cm}
\caption{The Pearson correlation coefficient (`r') and the significance value (`p') of the SCIP solution time between generated data and the original testing data under 45 different hyper-parameters on the SC, CA, CVS, and IIS problem.}
\vspace{-0.3cm}
\label{tab:downstream1_performance}
\renewcommand{\arraystretch}{1.2}
\resizebox{\linewidth}{!}{\begin{tabular}{cc|ccc|ccc|ccc|ccc}
\hline
 &  & \multicolumn{3}{c|}{CA} & \multicolumn{3}{c|}{SC} & \multicolumn{3}{c|}{CVS} & \multicolumn{3}{c}{IIS} \\ \hline
\multirow{2}{*}{Bowly} & r & \multicolumn{3}{c|}{-0.048} & \multicolumn{3}{c|}{0.683} & \multicolumn{3}{c|}{-0.158} & \multicolumn{3}{c}{0.292} \\
 & p & \multicolumn{3}{c|}{0.751} & \multicolumn{3}{c|}{2.295e-7} & \multicolumn{3}{c|}{0.298} & \multicolumn{3}{c}{0.051} \\ \hline
\multicolumn{2}{c|}{ratio} & 0.10 & 0.20 & 0.30 & 0.10 & 0.20 & 0.30 & 0.10 & 0.20 & 0.30 & 0.10 & 0.20 & 0.30 \\ \hline
\multirow{2}{*}{random} & r & 0.723 & 0.563 & 0.515 & 0.542 & 0.568 & 0.609 & -0.085 & -0.337 & -0.201 & 0.114 & 0.182 & 0.149 \\
 & p & 1.971e-8 & 5.522e-5 & 2.942e-4 & 1.174e-4 & 4.535e-5 & 9.028e-6 & 0.578 & 0.023 & 0.184 & 0.452 & 0.228 & 0.327 \\ \hline
\multirow{2}{*}{ours} & r & \textbf{0.728} & \textbf{0.771} & \textbf{0.780} & \textbf{0.747} & \textbf{0.717} & \textbf{0.665} & \textbf{0.609} & \textbf{0.590} & \textbf{0.607} & \textbf{0.542} & \textbf{0.300} & \textbf{0.551} \\
 & p & \textbf{1.446e-8} & \textbf{5.371e-10} & \textbf{2.544e-10} & \textbf{3.646e-9} & \textbf{2.908e-8} & \textbf{6.353e-7} & \textbf{8.834e-6} & \textbf{1.986e-5} & \textbf{9.581e-6} & \textbf{1.187e-4} & \textbf{0.044} & \textbf{8.497e-5} \\ \hline
\end{tabular}
}
\vspace{-0.3cm}
\end{table}

\subsubsection{downstream task \#2: Optimal Value Prediction via machine learning}
We conduct experiments for the second downstream task on all four datasets. 

\begin{table}[t]
\caption{The relative mean square error (MSE) of the optimal objective value task on the set covering (SC) problem. The $500$ original instances in training dataset $\#2 - \#14$ are identical.}
\label{tab:downstream2_setcover}
\vspace{-0.3cm}
\resizebox{\linewidth}{!}{\begin{tabular}{cccccccc|ccccc}
\hline
\multirow{2}{*}{dataset} & \multirow{2}{*}{\#original} & \multirow{2}{*}{\#generated} & \multirow{2}{*}{replace ratio} & \multicolumn{4}{c|}{out-of-distribution} & \multicolumn{5}{c}{in-distribution} \\ \cline{5-13} 
 &  &  &  & \textbf{0.03} & \textbf{0.04} & \textbf{0.05} & \textbf{0.10} & 0.15 & 0.20 & 0.25 & 0.30 & 0.35 \\ \hline
1 & 1000 & 0 & - & 0.792 & 0.640 & 0.488 & \textbf{0.022} & \textbf{0.009} & \textbf{0.009} & \textbf{0.010} & \textbf{0.011} & \textbf{0.015} \\
2 & 500 & 500 (Bowly) & - & 3.498 & 17.671 & 43.795 & 81.408 & 0.037 & 0.052 & 0.052 & 0.065 & 0.045 \\
3 & 500 & 500 (random) & 0.10 & 0.449 & 4.176 & 12.624 & 86.592 & 0.048 & 0.064 & 0.053 & 0.069 & 0.045 \\
4 & 500 & 500 (\dgm) & 0.01 & 0.505 & 0.280 & 0.142 & 0.032 & 0.032 & 0.040 & 0.044 & 0.044 & 0.040 \\
5 & 500 & 500 (\dgm) & 0.05 & 0.575 & 0.329 & 0.155 & 0.080 & 0.036 & 0.044 & 0.046 & 0.056 & 0.056 \\
6 & 500 & 500 (\dgm) & 0.10 & \textbf{0.362} & \textbf{0.141} & \textbf{0.045} & 0.065 & 0.017 & 0.012 & 0.012 & 0.010 & 0.015 \\
7 & 500 & 500 (\dgm) & 0.20 & 0.625 & 0.418 & 0.265 & 0.034 & 0.059 & 0.083 & 0.077 & 0.099 & 0.069 \\
8 & 500 & 500 (\dgm) & 0.50 & 0.884 & 0.822 & 0.769 & 0.285 & 0.017 & 0.025 & 0.033 & 0.047 & 0.032 \\ \hline
9 & 500 & 0 & - & 0.868 & 0.758 & 0.637 & 0.072 & \textbf{0.016} & 0.014 & 0.014 & 0.017 & 0.027 \\
10 & 500 & 50 (\dgm) & 0.10 & 0.693 & 0.497 & 0.327 & 0.031 & 0.035 & 0.039 & 0.046 & 0.039 & 0.052 \\
11 & 500 & 100 (\dgm) & 0.10 & 0.603 & 0.361 & 0.179 & 0.096 & 0.031 & 0.033 & 0.038 & 0.042 & 0.038 \\
12 & 500 & 200 (\dgm)& 0.10 & 0.628 & 0.396 & 0.215 & 0.086 & 0.038 & 0.035 & 0.039 & 0.043 & 0.039 \\
13 & 500 & 500 (\dgm)& 0.10 & \textbf{0.362} & \textbf{0.141} & \textbf{0.045} & \textbf{0.065} & 0.017 & \textbf{0.012} & \textbf{0.012} & \textbf{0.010} & \textbf{0.015} \\
14 & 500 & 1000 (\dgm)& 0.10 & 0.473 & 0.211 & 0.063 & 0.339 & 0.013 & 0.014 & 0.014 & 0.014 & 0.024 \\ \hline
\end{tabular}}
\end{table}

\textbf{Set Covering (SC)} One of the hyper-parameters of the SC instances is `density', representing the number of sets to be covered within a constraint. The training set (for both \dgm{} and the downstream predictor) comprises data with densities ranging from $0.15$ to $0.35$ only. We not only present test sets for each in-distribution density ($0.15$ to $0.35$) but also design the test sets with densities falling within the unexplored range of $0.03$ to $0.10$, to reflect the predictor's ability to generalize across distribution shift. The relative mean squared error (MSE) values of the models' predictions are presented in Table.~\ref{tab:downstream2_setcover}. 
In the first part (Datasets \#1-\#8), we curate datasets with a fixed training set size of $1000$. Dataset \#1 consist of $1000$ original data, dataset \#2 generates instance via the `Bowly', \#3 uses the `random' baseline. Datasets \#4-\#8 comprise a combination of $500$ original instances and $500$ \dgm-generated instances, with varying constraint node replacement ratios $\gamma$ ranging from $0.01$ to $0.50$. Models trained exclusively on in-distribution data exhibit superior fitting and predictive accuracy within the in-distribution test sets. However, models trained on a combination of original and \dgm-generated instances display significantly enhanced prediction accuracy on out-of-distribution testing data. We attribute this phenomenon to the increased structural and label diversity in the newly generated instances, mitigating over-fitting on in-distribution data and consequently bolstering the model's cross-distribution capabilities. It's worth noting that `Bowly' or `random' neither enhances the model's in-distribution nor out-of-distribution performance. We believe this is due to the less precise representation of the target distribution by the manually-designed `Bowly' baseline and the excessively high randomness in `random', causing the generated instances to deviate substantially from the original problems in both solution space and structure. 
In the second part (Datasets \#9-\#14), we investigate the impact of progressively incorporating \dgm-generated instances into the dataset, initially starting with $500$ original instances. We observe a consistent improvement in model performance with the gradual inclusion of additional newly generated instances, with peak performance achieved when augmenting the dataset with $500$ newly generated instances.

\textbf{Combinatorial Auctions (CA)} One of the hyper-parameters for the CA is the number of bid/item pair, which determines the quantity of variables and constraints. Our training set exclusively comprises examples with bid/item values ranging from $40/200$ to $80/400$. With the setting similar to the SC, our testing set not only has in-distribution bid/item value pairs, but also introduces instances with bid/item values ranging from $40/200$ to $160/800$, allowing us to assess the model's ability of cross-scale generalization. The relative mean squared error (MSE) of the model's predictions is provided in Table.~\ref{tab:downstream2_cauctions}. The experiments are also divided into two parts. The first part (Datasets \#1-\#8) yields similar conclusions, where models trained solely on original data excel in fitting within in-distribution test sets, models trained on a mixture of half original and half \dgm-generated instances perform better on test sets at scales never encountered during training (bid/item ranging from $100/500$ to $160/800$). This observation is attributed to the diversity introduced by the generated instances, in terms of both the problem structure and optimal objective labels, that prevents the models from over-fitting and thereby enhance their generalization across scales. Consistent with the SC, the second part demonstrates the impact of gradually increasing the new instances as training data and also achieves the peak performance with $500$ newly generated instances.

\textbf{CVS and IIS} Experiments on CVS and IIS show similar insights, see Appendix.~\ref{app:results_downstream2_cvs_iis} for details.

\begin{table}[t]
\vspace{-0.2cm}
\caption{The relative mean square error (MSE) of the optimal objective value task on the combinatorial auction (CA) problem. The $500$ original instances in training dataset $\#2 - \#14$ are identical.}
\label{tab:downstream2_cauctions}
\vspace{-0.3cm}
\resizebox{\linewidth}{!}{\begin{tabular}{ccccccc|cccc}
\hline
\multirow{2}{*}{dataset} & \multirow{2}{*}{\#original} & \multirow{2}{*}{\#generated} & \multirow{2}{*}{replace ratio} & \multicolumn{3}{c|}{in-distribution} & \multicolumn{4}{c}{out-of-distribution} \\ \cline{5-11} 
 &  &  &  & 40/200 & 60/300 & 80/400 & \textbf{100/500} & \textbf{120/600} & \textbf{140/700} & \textbf{160/800} \\ \hline
1 & 1000 & 0 & - & 0.246 & \textbf{0.003} & 0.060 & 0.155 & 0.239 & 0.312 & 0.379 \\
2 & 500 & 500 (Bowly) & - & \textbf{0.202} & 0.004 & 0.080 & 0.183 & 0.272 & 0.346 & 0.410 \\
3 & 500 & 500 (random) & 0.10 & 0.242 & 0.006 & 0.077 & 0.179 & 0.269 & 0.347 & 0.409 \\
4 & 500 & 500 (\dgm) & 0.01 & 0.346 & 0.008 & 0.043 & 0.131 & 0.219 & 0.292 & 0.359 \\
5 & 500 & 500 (\dgm) & 0.05 & 0.345 & 0.009 & 0.041 & 0.125 & 0.211 & 0.284 & 0.352 \\
6 & 500 & 500 (\dgm) & 0.10 & 0.385 & 0.015 & 0.036 & \textbf{0.118} & \textbf{0.201} & \textbf{0.276} & \textbf{0.340} \\
7 & 500 & 500 (\dgm) & 0.20 & 0.428 & 0.019 & \textbf{0.035} & 0.116 & 0.203 & 0.275 & 0.344 \\
8 & 500 & 500 (\dgm) & 0.30 & 0.381 & 0.012 & 0.040 & 0.126 & 0.215 & 0.289 & 0.356 \\
9 & 500 & 500 (\dgm) & 0.50 & 0.398 & 0.014 & 0.035 & 0.117 & 0.203 & 0.276 & 0.344 \\ \hline
10 & 500 & 0 & - & \textbf{0.216} & \textbf{0.004} & 0.068 & 0.165 & 0.249 & 0.324 & 0.388 \\
11 & 500 & 50 (\dgm) & 0.10 & 0.382 & 0.006 & 0.040 & 0.130 & 0.218 & 0.293 & 0.361 \\
12 & 500 & 100 (\dgm) & 0.10 & 0.446 & 0.014 & \textbf{0.031} & 0.116 & 0.201 & 0.275 & 0.344 \\
13 & 500 & 500 (\dgm) & 0.10 & 0.385 & 0.015 & 0.036 & \textbf{0.118} & \textbf{0.201} & \textbf{0.276} & \textbf{0.340} \\
14 & 500 & 1000 (\dgm) & 0.10 & 0.359 & 0.009 & 0.039 & 0.126 & 0.212 & 0.285 & 0.351 \\ \hline
\end{tabular}}
\vspace{-0.3cm}
\end{table}

\vspace{-0.1cm}
\section{Conclusion}
\vspace{-0.2cm}
This paper introduces \dgm, a deep generative framework for MILP. Contrasting with conventional MILP generation techniques, \dgm{} does not rely on domain-specific expertise. Instead, it employs DNNs to extract profound structural information from limited MILP data, generating ``representative'' instances. Notably, \dgm{} guarantees the feasibility and boundedness of generated data, ensuring the data's ``autheticity''. The generation space of \dgm{} encompasses any feasible-bounded MILP, providing it with the capability of generating ``diverse'' instances. Experiment evaluations highlights \dgm's potential in (S1) MILP data sharing for solver hyper-parameter tuning without publishing the original data and (S2) data augmentation to enhance the generalization capacity of ML models tasked with solving MILPs.

\bibliographystyle{iclr2024_conference}
\bibliography{iclr2024_conference}

\begin{thebibliography}{55}
\providecommand{\natexlab}[1]{#1}
\providecommand{\url}[1]{\texttt{#1}}
\expandafter\ifx\csname urlstyle\endcsname\relax
  \providecommand{\doi}[1]{doi: #1}\else
  \providecommand{\doi}{doi: \begingroup \urlstyle{rm}\Url}\fi

\bibitem[Achterberg \& Wunderling(2013)Achterberg and
  Wunderling]{achterberg2013mixed}
Tobias Achterberg and Roland Wunderling.
\newblock Mixed integer programming: Analyzing 12 years of progress.
\newblock In \emph{Facets of combinatorial optimization: Festschrift for martin
  gr{\"o}tschel}, pp.\  449--481. Springer, 2013.

\bibitem[Al-Falahy \& Alani(2017)Al-Falahy and Alani]{al2017technologies}
Naser Al-Falahy and Omar~Y Alani.
\newblock Technologies for 5g networks: Challenges and opportunities.
\newblock \emph{It Professional}, 19\penalty0 (1):\penalty0 12--20, 2017.

\bibitem[Asahiro et~al.(1996)Asahiro, Iwama, and Miyano]{asahiro1996random}
Yuihci Asahiro, Kazuo Iwama, and Eiji Miyano.
\newblock Random generation of test instances with controlled attributes.
\newblock \emph{Cliques, Coloring, and Satisfiability}, pp.\  377--393, 1996.

\bibitem[Balas \& Ho(1980)Balas and Ho]{balas1980set}
Egon Balas and Andrew Ho.
\newblock \emph{Set covering algorithms using cutting planes, heuristics, and
  subgradient optimization: a computational study}.
\newblock Springer, 1980.

\bibitem[Bengio et~al.(2013)Bengio, L{\'e}onard, and
  Courville]{bengio2013estimating}
Yoshua Bengio, Nicholas L{\'e}onard, and Aaron Courville.
\newblock Estimating or propagating gradients through stochastic neurons for
  conditional computation.
\newblock \emph{arXiv preprint arXiv:1308.3432}, 2013.

\bibitem[Bertsimas \& Tsitsiklis(1997)Bertsimas and
  Tsitsiklis]{bertsimas1997introduction}
Dimitris Bertsimas and John~N Tsitsiklis.
\newblock \emph{Introduction to linear optimization}, volume~6.
\newblock Athena scientific Belmont, MA, 1997.

\bibitem[Bestuzheva et~al.(2021)Bestuzheva, Besan{\c{c}}on, Chen, Chmiela,
  Donkiewicz, van Doornmalen, Eifler, Gaul, Gamrath, Gleixner,
  et~al.]{bestuzheva2021scip}
Ksenia Bestuzheva, Mathieu Besan{\c{c}}on, Wei-Kun Chen, Antonia Chmiela, Tim
  Donkiewicz, Jasper van Doornmalen, Leon Eifler, Oliver Gaul, Gerald Gamrath,
  Ambros Gleixner, et~al.
\newblock The scip optimization suite 8.0.
\newblock \emph{arXiv preprint arXiv:2112.08872}, 2021.

\bibitem[Bowly et~al.(2020)Bowly, Smith-Miles, Baatar, and
  Mittelmann]{bowly2020generation}
Simon Bowly, Kate Smith-Miles, Davaatseren Baatar, and Hans Mittelmann.
\newblock Generation techniques for linear programming instances with
  controllable properties.
\newblock \emph{Mathematical Programming Computation}, 12\penalty0
  (3):\penalty0 389--415, 2020.

\bibitem[Bowly(2019)]{bowly2019stress}
Simon~Andrew Bowly.
\newblock \emph{Stress testing mixed integer programming solvers through new
  test instance generation methods}.
\newblock PhD thesis, School of Mathematical Sciences, Monash University, 2019.

\bibitem[Branke et~al.(2015)Branke, Nguyen, Pickardt, and
  Zhang]{branke2015automated}
Juergen Branke, Su~Nguyen, Christoph~W Pickardt, and Mengjie Zhang.
\newblock Automated design of production scheduling heuristics: A review.
\newblock \emph{IEEE Transactions on Evolutionary Computation}, 20\penalty0
  (1):\penalty0 110--124, 2015.

\bibitem[Byrd et~al.(1987)Byrd, Goldman, and Heller]{byrd1987recognizing}
Richard~H Byrd, Alan~J Goldman, and Miriam Heller.
\newblock Recognizing unbounded integer programs.
\newblock \emph{Operations Research}, 35\penalty0 (1):\penalty0 140--142, 1987.

\bibitem[Chen et~al.(2023)Chen, Liu, Wang, Lu, and Yin]{chen2022representing}
Ziang Chen, Jialin Liu, Xinshang Wang, Jianfeng Lu, and Wotao Yin.
\newblock On representing linear programs by graph neural networks.
\newblock \emph{International Conference on Learning Representations}, 2023.

\bibitem[Cplex(2009)]{cplex2009v12}
IBM~ILOG Cplex.
\newblock V12. 1: User’s manual for cplex.
\newblock \emph{International Business Machines Corporation}, 46\penalty0
  (53):\penalty0 157, 2009.

\bibitem[Culberson(2002)]{culberson2002graph}
J~Culberson.
\newblock A graph generator for various classes of k-colorable graphs.
\newblock \emph{URL http://webdocs. cs. ualberta. ca/\~{}
  joe/Coloring/Generators/generate. html}, 2002.

\bibitem[Drugan(2013)]{drugan2013instance}
M{\u{a}}d{\u{a}}lina~M Drugan.
\newblock Instance generator for the quadratic assignment problem with
  additively decomposable cost function.
\newblock In \emph{2013 IEEE Congress on Evolutionary Computation}, pp.\
  2086--2093. IEEE, 2013.

\bibitem[Fey \& Lenssen(2019)Fey and Lenssen]{fey2019fast}
Matthias Fey and Jan~Eric Lenssen.
\newblock Fast graph representation learning with pytorch geometric.
\newblock \emph{arXiv preprint arXiv:1903.02428}, 2019.

\bibitem[Gamrath et~al.(2015)Gamrath, Koch, Martin, Miltenberger, and
  Weninger]{gamrath2015progress}
Gerald Gamrath, Thorsten Koch, Alexander Martin, Matthias Miltenberger, and
  Dieter Weninger.
\newblock Progress in presolving for mixed integer programming.
\newblock \emph{Mathematical Programming Computation}, 7:\penalty0 367--398,
  2015.

\bibitem[Gamrath et~al.(2020)Gamrath, Berthold, and
  Salvagnin]{gamrath2020exploratory}
Gerald Gamrath, Timo Berthold, and Domenico Salvagnin.
\newblock An exploratory computational analysis of dual degeneracy in
  mixed-integer programming.
\newblock \emph{EURO Journal on Computational Optimization}, 8\penalty0
  (3-4):\penalty0 241--261, 2020.

\bibitem[Gasse et~al.(2019)Gasse, Ch{\'e}telat, Ferroni, Charlin, and
  Lodi]{gasse2019exact}
Maxime Gasse, Didier Ch{\'e}telat, Nicola Ferroni, Laurent Charlin, and Andrea
  Lodi.
\newblock Exact combinatorial optimization with graph convolutional neural
  networks.
\newblock \emph{Advances in neural information processing systems}, 32, 2019.

\bibitem[Gleixner et~al.(2021)Gleixner, Hendel, Gamrath, Achterberg, Bastubbe,
  Berthold, Christophel, Jarck, Koch, Linderoth, et~al.]{gleixner2021miplib}
Ambros Gleixner, Gregor Hendel, Gerald Gamrath, Tobias Achterberg, Michael
  Bastubbe, Timo Berthold, Philipp Christophel, Kati Jarck, Thorsten Koch, Jeff
  Linderoth, et~al.
\newblock Miplib 2017: data-driven compilation of the 6th mixed-integer
  programming library.
\newblock \emph{Mathematical Programming Computation}, 13\penalty0
  (3):\penalty0 443--490, 2021.

\bibitem[{Gurobi}(2023)]{gurobi}
{Gurobi}.
\newblock {Gurobi Optimizer Reference Manual}, 2023.
\newblock URL \url{https://www.gurobi.com}.

\bibitem[Hill et~al.(2011)Hill, Moore, Hiremath, and Cho]{hill2011test}
R~Hill, JT~Moore, C~Hiremath, and YK~Cho.
\newblock Test problem generation of binary knapsack problem variants and the
  implications of their use.
\newblock \emph{Int. J. Oper. Quant. Manag}, 18\penalty0 (2):\penalty0
  105--128, 2011.

\bibitem[Hill \& Reilly(2000)Hill and Reilly]{hill2000effects}
Raymond~R Hill and Charles~H Reilly.
\newblock The effects of coefficient correlation structure in two-dimensional
  knapsack problems on solution procedure performance.
\newblock \emph{Management Science}, 46\penalty0 (2):\penalty0 302--317, 2000.

\bibitem[Hooker(1994)]{hooker1994needed}
John~N Hooker.
\newblock Needed: An empirical science of algorithms.
\newblock \emph{Operations research}, 42\penalty0 (2):\penalty0 201--212, 1994.

\bibitem[Hooker(1995)]{hooker1995testing}
John~N Hooker.
\newblock Testing heuristics: We have it all wrong.
\newblock \emph{Journal of heuristics}, 1:\penalty0 33--42, 1995.

\bibitem[Huber(1992)]{huber1992robust}
Peter~J Huber.
\newblock Robust estimation of a location parameter.
\newblock In \emph{Breakthroughs in statistics: Methodology and distribution},
  pp.\  492--518. Springer, 1992.

\bibitem[Hugos(2018)]{hugos2018essentials}
Michael~H Hugos.
\newblock \emph{Essentials of supply chain management}.
\newblock John Wiley \& Sons, 2018.

\bibitem[Hutter et~al.(2011)Hutter, Hoos, and
  Leyton-Brown]{hutter2011sequential}
Frank Hutter, Holger~H Hoos, and Kevin Leyton-Brown.
\newblock Sequential model-based optimization for general algorithm
  configuration.
\newblock In \emph{Learning and Intelligent Optimization: 5th International
  Conference, LION 5, Rome, Italy, January 17-21, 2011. Selected Papers 5},
  pp.\  507--523. Springer, 2011.

\bibitem[Jang et~al.(2016)Jang, Gu, and Poole]{jang2016categorical}
Eric Jang, Shixiang Gu, and Ben Poole.
\newblock Categorical reparameterization with gumbel-softmax.
\newblock \emph{arXiv preprint arXiv:1611.01144}, 2016.

\bibitem[Khalil et~al.(2016)Khalil, Le~Bodic, Song, Nemhauser, and
  Dilkina]{khalil2016learning}
Elias Khalil, Pierre Le~Bodic, Le~Song, George Nemhauser, and Bistra Dilkina.
\newblock Learning to branch in mixed integer programming.
\newblock In \emph{Proceedings of the AAAI Conference on Artificial
  Intelligence}, volume~30, 2016.

\bibitem[Khalil et~al.(2017)Khalil, Dai, Zhang, Dilkina, and
  Song]{khalil2017learning}
Elias Khalil, Hanjun Dai, Yuyu Zhang, Bistra Dilkina, and Le~Song.
\newblock Learning combinatorial optimization algorithms over graphs.
\newblock \emph{Advances in neural information processing systems}, 30, 2017.

\bibitem[Kingma \& Ba(2014)Kingma and Ba]{kingma2014adam}
Diederik~P Kingma and Jimmy Ba.
\newblock Adam: A method for stochastic optimization.
\newblock \emph{arXiv preprint arXiv:1412.6980}, 2014.

\bibitem[Kingma \& Welling(2013)Kingma and Welling]{kingma2013auto}
Diederik~P Kingma and Max Welling.
\newblock Auto-encoding variational bayes.
\newblock \emph{arXiv preprint arXiv:1312.6114}, 2013.

\bibitem[Kipf \& Welling(2016)Kipf and Welling]{kipf2016variational}
Thomas~N Kipf and Max Welling.
\newblock Variational graph auto-encoders.
\newblock \emph{arXiv preprint arXiv:1611.07308}, 2016.

\bibitem[Lawler \& Wood(1966)Lawler and Wood]{lawler1966branch}
Eugene~L Lawler and David~E Wood.
\newblock Branch-and-bound methods: A survey.
\newblock \emph{Operations research}, 14\penalty0 (4):\penalty0 699--719, 1966.

\bibitem[Lindauer \& Hutter(2018)Lindauer and Hutter]{lindauer2018warmstarting}
Marius Lindauer and Frank Hutter.
\newblock Warmstarting of model-based algorithm configuration.
\newblock In \emph{Proceedings of the AAAI Conference on Artificial
  Intelligence}, volume~32, 2018.

\bibitem[Lindauer et~al.(2022)Lindauer, Eggensperger, Feurer, Biedenkapp, Deng,
  Benjamins, Ruhkopf, Sass, and Hutter]{lindauer2022smac3}
Marius Lindauer, Katharina Eggensperger, Matthias Feurer, Andr{\'e} Biedenkapp,
  Difan Deng, Carolin Benjamins, Tim Ruhkopf, Ren{\'e} Sass, and Frank Hutter.
\newblock Smac3: A versatile bayesian optimization package for hyperparameter
  optimization.
\newblock \emph{The Journal of Machine Learning Research}, 23\penalty0
  (1):\penalty0 2475--2483, 2022.

\bibitem[Lu et~al.(2022)Lu, Li, Wang, Ren, Li, Yuan, Zeng, Yang, and
  Yan]{lu2022roco}
Han Lu, Zenan Li, Runzhong Wang, Qibing Ren, Xijun Li, Mingxuan Yuan, Jia Zeng,
  Xiaokang Yang, and Junchi Yan.
\newblock Roco: A general framework for evaluating robustness of combinatorial
  optimization solvers on graphs.
\newblock In \emph{The Eleventh International Conference on Learning
  Representations}, 2022.

\bibitem[Maddison et~al.(2016)Maddison, Mnih, and Teh]{maddison2016concrete}
Chris~J Maddison, Andriy Mnih, and Yee~Whye Teh.
\newblock The concrete distribution: A continuous relaxation of discrete random
  variables.
\newblock \emph{arXiv preprint arXiv:1611.00712}, 2016.

\bibitem[Maher et~al.(2016{\natexlab{a}})Maher, Miltenberger, Pedroso,
  Rehfeldt, Schwarz, and
  Serrano]{MaherMiltenbergerPedrosoRehfeldtSchwarzSerrano2016}
Stephen Maher, Matthias Miltenberger, Jo{\~{a}}o~Pedro Pedroso, Daniel
  Rehfeldt, Robert Schwarz, and Felipe Serrano.
\newblock {PySCIPOpt}: Mathematical programming in python with the {SCIP}
  optimization suite.
\newblock In \emph{Mathematical Software {\textendash} {ICMS} 2016}, pp.\
  301--307. Springer International Publishing, 2016{\natexlab{a}}.
\newblock \doi{10.1007/978-3-319-42432-3_37}.

\bibitem[Maher et~al.(2016{\natexlab{b}})Maher, Miltenberger, Pedroso,
  Rehfeldt, Schwarz, and Serrano]{maher2016pyscipopt}
Stephen Maher, Matthias Miltenberger, Jo{\~a}o~Pedro Pedroso, Daniel Rehfeldt,
  Robert Schwarz, and Felipe Serrano.
\newblock Pyscipopt: Mathematical programming in python with the scip
  optimization suite.
\newblock In \emph{Mathematical Software--ICMS 2016: 5th International
  Conference, Berlin, Germany, July 11-14, 2016, Proceedings 5}, pp.\
  301--307. Springer, 2016{\natexlab{b}}.

\bibitem[Mansini et~al.(2015)Mansini, W{\'L}, Speranza, and of~European
  Operational Research~Societies]{mansini2015linear}
Renata Mansini, odzimierz~Ogryczak W{\'L}, M~Grazia Speranza, and EURO:
  The~Association of~European Operational Research~Societies.
\newblock \emph{Linear and mixed integer programming for portfolio
  optimization}, volume~21.
\newblock Springer, 2015.

\bibitem[Meyer(1974)]{meyer1974existence}
Robert~R Meyer.
\newblock On the existence of optimal solutions to integer and mixed-integer
  programming problems.
\newblock \emph{Mathematical Programming}, 7:\penalty0 223--235, 1974.

\bibitem[Nair et~al.(2020)Nair, Bartunov, Gimeno, Von~Glehn, Lichocki, Lobov,
  O'Donoghue, Sonnerat, Tjandraatmadja, Wang, et~al.]{nair2020solving}
Vinod Nair, Sergey Bartunov, Felix Gimeno, Ingrid Von~Glehn, Pawel Lichocki,
  Ivan Lobov, Brendan O'Donoghue, Nicolas Sonnerat, Christian Tjandraatmadja,
  Pengming Wang, et~al.
\newblock Solving mixed integer programs using neural networks.
\newblock \emph{arXiv preprint arXiv:2012.13349}, 2020.

\bibitem[Paszke et~al.(2019)Paszke, Gross, Massa, Lerer, Bradbury, Chanan,
  Killeen, Lin, Gimelshein, Antiga, et~al.]{paszke2019pytorch}
Adam Paszke, Sam Gross, Francisco Massa, Adam Lerer, James Bradbury, Gregory
  Chanan, Trevor Killeen, Zeming Lin, Natalia Gimelshein, Luca Antiga, et~al.
\newblock Pytorch: An imperative style, high-performance deep learning library.
\newblock \emph{Advances in neural information processing systems}, 32, 2019.

\bibitem[Pilcher \& Rardin(1992)Pilcher and Rardin]{pilcher1992partial}
Martha~G Pilcher and Ronald~L Rardin.
\newblock Partial polyhedral description and generation of discrete
  optimization problems with known optima.
\newblock \emph{Naval Research Logistics (NRL)}, 39\penalty0 (6):\penalty0
  839--858, 1992.

\bibitem[Radosavovic et~al.(2020)Radosavovic, Kosaraju, Girshick, He, and
  Doll{\'a}r]{radosavovic2020designing}
Ilija Radosavovic, Raj~Prateek Kosaraju, Ross Girshick, Kaiming He, and Piotr
  Doll{\'a}r.
\newblock Designing network design spaces.
\newblock In \emph{Proceedings of the IEEE/CVF conference on computer vision
  and pattern recognition}, pp.\  10428--10436, 2020.

\bibitem[Richards \& How(2002)Richards and How]{richards2002aircraft}
Arthur Richards and Jonathan~P How.
\newblock Aircraft trajectory planning with collision avoidance using mixed
  integer linear programming.
\newblock In \emph{Proceedings of the 2002 American Control Conference (IEEE
  Cat. No. CH37301)}, volume~3, pp.\  1936--1941. IEEE, 2002.

\bibitem[Rodr{\'\i}guez et~al.(2016)Rodr{\'\i}guez, Rubio, and
  Rabanal]{rodriguez2016automatic}
Ismael Rodr{\'\i}guez, Fernando Rubio, and Pablo Rabanal.
\newblock Automatic media planning: optimal advertisement placement problems.
\newblock In \emph{2016 IEEE Congress on Evolutionary Computation (CEC)}, pp.\
  5170--5177. IEEE, 2016.

\bibitem[Roling et~al.(2008)Roling, Visser, et~al.]{roling2008optimal}
Paul~C Roling, Hendrikus~G Visser, et~al.
\newblock Optimal airport surface traffic planning using mixed-integer linear
  programming.
\newblock \emph{International Journal of Aerospace Engineering}, 2008, 2008.

\bibitem[Schrijver(1998)]{schrijver1998theory}
Alexander Schrijver.
\newblock \emph{Theory of linear and integer programming}.
\newblock John Wiley \& Sons, 1998.

\bibitem[Smith-Miles \& Bowly(2015)Smith-Miles and Bowly]{smith2015generating}
Kate Smith-Miles and Simon Bowly.
\newblock Generating new test instances by evolving in instance space.
\newblock \emph{Computers \& Operations Research}, 63:\penalty0 102--113, 2015.

\bibitem[Vander~Wiel \& Sahinidis(1995)Vander~Wiel and
  Sahinidis]{vander1995heuristic}
Russ~J Vander~Wiel and Nikolaos~V Sahinidis.
\newblock Heuristic bounds and test problem generation for the time-dependent
  traveling salesman problem.
\newblock \emph{Transportation Science}, 29\penalty0 (2):\penalty0 167--183,
  1995.

\bibitem[Virtanen et~al.(2020)Virtanen, Gommers, Oliphant, Haberland, Reddy,
  Cournapeau, Burovski, Peterson, Weckesser, Bright, {van der Walt}, Brett,
  Wilson, Millman, Mayorov, Nelson, Jones, Kern, Larson, Carey, Polat, Feng,
  Moore, {VanderPlas}, Laxalde, Perktold, Cimrman, Henriksen, Quintero, Harris,
  Archibald, Ribeiro, Pedregosa, {van Mulbregt}, and {SciPy 1.0
  Contributors}]{2020SciPy-NMeth}
Pauli Virtanen, Ralf Gommers, Travis~E. Oliphant, Matt Haberland, Tyler Reddy,
  David Cournapeau, Evgeni Burovski, Pearu Peterson, Warren Weckesser, Jonathan
  Bright, St{\'e}fan~J. {van der Walt}, Matthew Brett, Joshua Wilson, K.~Jarrod
  Millman, Nikolay Mayorov, Andrew R.~J. Nelson, Eric Jones, Robert Kern, Eric
  Larson, C~J Carey, {\.I}lhan Polat, Yu~Feng, Eric~W. Moore, Jake
  {VanderPlas}, Denis Laxalde, Josef Perktold, Robert Cimrman, Ian Henriksen,
  E.~A. Quintero, Charles~R. Harris, Anne~M. Archibald, Ant{\^o}nio~H. Ribeiro,
  Fabian Pedregosa, Paul {van Mulbregt}, and {SciPy 1.0 Contributors}.
\newblock {{SciPy} 1.0: Fundamental Algorithms for Scientific Computing in
  Python}.
\newblock \emph{Nature Methods}, 17:\penalty0 261--272, 2020.
\newblock \doi{10.1038/s41592-019-0686-2}.

\bibitem[Wolsey(2020)]{wolsey2020integer}
Laurence~A Wolsey.
\newblock \emph{Integer programming}.
\newblock John Wiley \& Sons, 2020.

\end{thebibliography}
\newpage
\appendix
\newpage

\section{supplementary theoretical results}

\subsection{Proof of Proposition \ref{thm:main}}
\label{app:thm_main_proof}

To validate Proposition \ref{thm:main}, we follow the methodology outlined in~\cite{bowly2019stress}.
Before the proof of Proposition \ref{thm:main}, we first give the definition of boundedness, feasibility and optimal solutions of MILP, then we discuss the existence of optimal solutions of LP and MILP. 

\begin{definition}[Feasibility of MILP]
\label{def:feasibility_milp}
    An $\text{MILP}(\mA,\vb,\vc)$ is \textit{feasible} if there exists an $\vx$ such that all the constraints are satisfied: $\vx \in \mathbb{Z}^n_{\geq 0}, \mA \vx \leq \vb$. Such an $\vx$ is named a feasible solution.
\end{definition}

\begin{definition}[Boundedness of MILP]
\label{def:boundedness_milp}
    An $\text{MILP}(\mA,\vb,\vc)$ is \textit{bounded} if there's an upper bound on $\vc^\top \vx$ across all feasible solutions.
\end{definition}

\begin{definition}[Optimal Solution for MILP]
\label{def:optimal_solution_milp}
    A vector $\vx^\star$ is recognized as an optimal solution if it's a feasible solution and it is no worse than all other feasible solutions: $\vc^\top \vx^\star \geq \vc^\top \vx$, given $\vx$ is feasible.
\end{definition}

All LPs must fall into one of the following cases~\cite{bertsimas1997introduction}:
\begin{itemize}
    \item Infeasible. 
    \item Feasible but unbounded. 
    \item Feasible and bounded. Only in this case, the LP yields an optimal solution.
\end{itemize}
However, general MILP will be much more complicated. Consider a simple example: $\min \sqrt{3} x_1 - x_2, \ \text{s.t.} \ \sqrt{3} x_1 - x_2 \geq 0, x_1 \geq 1, \vx \in \mathbb{Z}^2_{\geq 0}$. No feasible solution has objective equal to zero, but there are feasible solutions with objective arbitrarily close to zero. In other words, \textit{an MILP might be bounded but with no optimal solutions.} Such a pathological phenomenon is caused by the irrational number $\sqrt{3}$ in the coefficient. Therefore, we only consider MILP with rational data: \[\mA \in \sQ^{m\times n}, \vb \in \sQ^{m}, \vc \in \sQ^{m}.\] 
Such an assumption is regularly adopted in the research of MILP. 

Without requiring $\vx$ to be integral, \eqref{equ:primal_format} will be relaxed to an LP, named its \textit{LP relaxation}:
\[\textbf{LP}(\mA,\vb,\vc): \quad \max_{\vx} \vc^\top \vx, \quad  \text{s.t. } \mA \ \vx \leq \vb, \ \vx \geq 0.\]
The feasibility, boundedness, and existence of optimal solutions, along with the relationship with its LP relaxation, are summarized in the following lemma.
\begin{lemma}
\label{lemma:milp-lp}
Given $\mA \in \sQ^{m\times n}, \vb \in \sQ^{m}, \vc \in \sQ^{m}$, it holds that
    \begin{itemize}
    \item (I) If $\text{LP}(\mA,\vb,\vc)$ is infeasible, $\text{MILP}(\mA,\vb,\vc)$ must be infeasible.
    \item (II) If $\text{LP}(\mA,\vb,\vc)$ is feasible but unbounded, then $\text{MILP}(\mA,\vb,\vc)$ must be either infeasible or unbounded. 
    \item (III) If $\text{LP}(\mA,\vb,\vc)$ is feasible and bounded, $\text{MILP}(\mA,\vb,\vc)$ might be infeasible or feasible. If we further assume $\text{MILP}(\mA,\vb,\vc)$ is feasible, it must yield an optimal solution. 
\end{itemize}
\end{lemma}
\begin{proof} Conclusion (I) is trivial. Conclusion (II) is exactly \cite[Theorem 1]{byrd1987recognizing}. Conclusion (III) is a corollary of \cite[Theorem 2.1]{meyer1974existence}. To obtain (III), we first write $\text{MILP}(\mA,\vb,\vc)$ into the following form:
\[ \max_{\vx,\vr} \vc^\top \vx \quad \text{s.t. }\mA\vx + \vr = \vb, ~ \vx \geq \vzero, ~\vr \geq \vzero,~ \vx \text{ is integral} \]
Then the condition (v) in \cite[Theorem 2.1]{meyer1974existence} can be directly applied. Therefore, the feasibility and boundedness of $\text{MILP}(\mA,\vb,\vc)$ imply the existence of optimal solutions, which concludes the proof.
\end{proof}

With Lemma \ref{lemma:milp-lp}, we could prove Proposition \ref{thm:main} now.

\begin{proof}[Proof of Proposition \ref{thm:main}] At the beginning, we define the space of $[\mA,\vb,\vc]$ generated based on $\gF$ as $\gH''$ for simplicity.
\[ \gH'' :=  \Big\{ [\mA,\vb,\vc]: \vb = \mA \vx + \vr, \vc = \mA^\top \vy - \vs, [\mA, \vx, \vy, \vs, \vr] \in \gF \Big\}\]
Then it's enough to show that $\gH' \subset \gH''$ and $\gH'' \subset \gH'$.

 We first show  $\gH'' \subset \gH'$: For any $[\mA,\vb,\vc] \in \gH''$, it holds that $[\mA,\vb,\vc] \in \gH'$. In another word, we have to show  $\text{MILP}(\mA,\vb,\vc)$ to be feasible and bounded for all $[\mA,\vb,\vc] \in \gH''$. The feasibility can be easily verified. The boundedness can be proved by ``weak duality." For the sake of completeness, we provide a detailed proof here. Define the Lagrangian as
\[ \gL(\vx, \vy) := \vc^\top \vx + \vy^\top \left( \vb - \mA \vx \right)\]
Inequalities $\mA\vx \leq \vb$ and $\vy \geq \vzero$ imply
\[ \gL(\vx, \vy) \geq \vc^\top \vx \]
Inequalities $\mA^\top \vy \geq \vc$ and $\vx \geq \vzero$ imply
\[ \gL(\vx, \vy) \leq \vb^\top \vy \]
Since $\vx \in \sQ^{n}_{\geq 0}$ and $\vy \in \sQ^{m}_{\geq 0}$, it holds that 
\[ -\infty < \vc^\top\vx \leq \vb^\top \vy < +\infty \]
which concludes the boundedness of $\text{MILP}(\mA,\vb,\vc)$.

We then show $\gH' \subset \gH''$: For any $\text{MILP}(\mA,\vb,\vc)$ that is feasible and bounded, there must be $[\mA, \vx, \vy, \vs, \vr] \in \gF$ such that 
\begin{align}
    \vb =& \mA \vx + \vr, \label{equ:slack-1} \\
    \vc =& \mA^\top \vy - \vs. \label{equ:slack-2}
\end{align}
The existence of $\vx,\vr$, along with \eqref{equ:slack-1}, is a direct conclusion of the feasibility of $\text{MILP}(\mA,\vb,\vc)$. Now let's prove the existence of rational vectors $\vy,\vs$, along with \eqref{equ:slack-2}. Since $\text{MILP}(\mA,\vb,\vc)$ is feasible and bounded, according to Lemma \ref{lemma:milp-lp}, $\text{LP}(\mA,\vb,\vc)$ must be feasible and bounded. Thanks to the weak duality discussed above, we conclude that $\text{DualLP}(\mA,\vb,\vc)$ must be feasible and bounded. As long as $\text{DualLP}(\mA,\vb,\vc)$ has an optimal solution $\vy^\star$ that is rational, one can obtain \eqref{equ:slack-2} by regarding $[\vy^\star, \mA^\top\vy^\star - \vc]$ as $[\vy,\vs]$. Therefore, it's enough to show $\text{DualLP}(\mA,\vb,\vc)$ has a rational optimal solution.

Define:
\[
\begin{aligned}
    \mA' =& [\mA^\top, -\mI]\\
    \vy' =& [\vy^\top, \vs^\top]^\top \\
    \vb' =& [\vb^\top, \vzero^\top]
\end{aligned}
\]
Then DualLP can be written as a standard-form LP:
\begin{equation}
    \label{equ:dual-lp-standard}
    \min_{\vy'} (\vb')^\top \vy' \quad \text{s.t. }\mA'\vy' = \vc,~\vy'\geq\vzero
\end{equation}
As long as an LP has an optimal solution, it must have a basic optimal solution~\cite{bertsimas1997introduction}. Specifically, we can split $\mA'$ in column-based fashion as $\mA' = [\mB',\mN']$ and split $\vy'$ as $\vy' = [\vy^\top_{B}, \vy^\top_{N}]^\top$, where $\vy_{N} = \vzero$. Such a $\vy'$ is termed a \textit{basic optimal solution} to the LP presented in \eqref{equ:dual-lp-standard}. Therefore,
\[ \mA'\vy' = \mB'\vy_{B} + \mN'\vy_{N} = \mB'\vy_{B} = \vc \implies \vy_{B} = (\mB')^{-1}\vc\]
Since $\mB'$ is a sub-matrix of $\mA'$, $\mB'$ is rational. Therefore, $(\mB')^{-1}$ and $\vy_{B}$ are rational, which implies $\vy'$ is rational. This concludes the existence of rational optimal solutions of DualLP, which finishes the entire proof.
\end{proof}

\subsection{Derivation of the loss function}
\label{app:derivation_loss}
Here we show the derivation from the training objective in Equation.~\ref{equ:max_likelihhod} towards the loss function in Equation.~\ref{equ:loss}.
\begin{equation}
\resizebox{0.93\textwidth}{!}{
$\begin{aligned}
    \log \mathbb{P}(G|G';\theta,\phi) & = \mathbb{E}_{\vz \sim q_{\phi}(\vz|G)} \log \mathbb{P}(G|G';\theta,\phi) \\
    & = \mathbb{E}_{\vz \sim q_{\phi}(\vz|G)} [\log \frac{p_{\theta}(G|G',\vz) p(\vz)}{q_{\phi}(\vz|G)}\frac{q_{\phi}(\vz|G)}{p(\vz|G)}] \\
    & = \mathbb{E}_{\vz \sim q_{\phi}(\vz|G)} \log \frac{p_{\theta}(G|G', \vz)p(\vz)}{q_{\phi}(\vz|G)} + \mathbb{E}_{\vz \sim q_{\phi}(\vz|G)} [\log \frac{q_{\phi}(\vz|G)}{p(\vz|G)}] \\
    & = \mathbb{E}_{\vz \sim q_{\phi}(\vz|G)}[\log p_{\theta}(G|\vz, G')] - \mathcal{D}_{KL}[q_{\phi}(\vz|G)||p(\vz)] + \mathcal{D}_{KL}[q_{\phi}(\vz|G)||p(\vz|G)] \\
    & \geq \mathbb{E}_{\vz \sim q_{\phi}(\vz|G)}[\log p_{\theta}(G|G', \vz)] - \mathcal{D}_{KL}[q_{\phi}(\vz|G)||\mathcal{N}(0,I)]
\end{aligned}$},
\end{equation}and thus we have
\begin{equation}
    \mathbb{E}_{G\sim \mathcal{G}} \mathbb{E}_{G' \sim p_{G'|G)}} \log \mathbb{P}(G|G';\theta,\phi) \geq -\mathcal{L}_{\theta, \phi}
\end{equation}

\section{supplementary implementation details}
\label{app:implementation_details}
\subsection{Hardware, Software and Platforms}
At the hardware level, we employ an Intel Xeon Gold 6248R CPU and a Nvidia quadro RTX 6000 GPU. For tasks that exclusively run on the CPU, we utilize a single core, for tasks that run on the GPU, we set the upper limit to $10$ cores. 
On the software side, we utilize PyTorch version $2.0.0+$cu$117$~\citep{paszke2019pytorch} and PyTorch Geometric version $2.0.3$~\citep{fey2019fast}. We utilize PySCIPOpt solver version $3.5.0$~\citep{MaherMiltenbergerPedrosoRehfeldtSchwarzSerrano2016} for optimization purposes with default configurations.

\subsection{Implementation of \dgm}
\label{app:implemnentation_dgm}
For both the encoder and the decoder, we adopt the bipartite GNN exactly the same as that in ~\cite{gasse2019exact} as their backbones, the original codes for the backbone is publicly available\footnote{\url{https://github.com/ds4dm/learn2branch/blob/master/models/baseline/model.py}}.

\textbf{Encoder} To obtain the latent variable samples, we feed the encoder with $G$ encoded as per the method in Table.~\ref{tab:vc_encoding}, we then incorporate two distinct multi-layer perceptron (MLP) layers following the backbone to output the mean and log variance of the latent variable $\vz$. During the training process, we use the re-parametrization trick~\citep{bengio2013estimating, maddison2016concrete, jang2016categorical} to render the process of sampling $\vz$ from the mean and variance differentiable. During inference, we directly sample $\vz \sim \mathcal{N}(0,I)$.

\begin{wraptable}{r}{0.4\textwidth}
\centering
\captionsetup{font=small} 
\vspace{-0.4cm}
\caption{The last layer design of decoder.}
\label{tab:layer_design}
\vspace{-0.3cm}
\begin{tabular}{cc}
\hline
prediction & embeddings \\ \hline
$d,e,w,\vx, \vr$ &  $\vh_v, \vz_v, v \in \mathcal{V}$\\
$\vy, \vs$ &  $\vh_c, \vz_c, c \in \mathcal{C}$\\
\hline
\end{tabular}
\vspace{-0.3cm}
\end{wraptable}
\textbf{Decoder} We feed the backbone of the decoder with the incomplete graph $G'$ to obtain the latent node representations $\vh^{G'} = \{\vh^{G'}_c, \vh^{G'}_v\}$.
The backbone is then followed by seven distinct heads conditionally independent on $\vh$ and $\vz$, each corresponding to the prediction of: 1) the degree of the removed node $d_{c_i}$, 2) the edges $e(c_i,u)$ between the constraint node $c_i$ and the nodes in the other side, 3) the edge weights $w_{c_i}$, and 4) - 7) the value of $\vx, \vy, \vr, \vs$ of the new graph $\tilde{G}$. Each head is composed of layers of MLP, and takes different combinations $\vh^{G'}, \vz^{G'}$ as inputs, which is illustrated in Table.~\ref{tab:layer_design}. We perform min-max normalization on all the variables to predict according to their maximum and minimum value occurred in the training dataset. Each part is modeled as a regression task, where we use the Huber Loss~\cite{huber1992robust} as the criterion for each part and add them together as the total loss for decoder.

For the case of binary MILP problems, their primal, dual and slack variables could be written in the form as Equation.~\ref{equ:primal_format_binary}~\ref{equ:dual_format_binary}~\ref{equ:slack_variables_binary}:

\vspace{-0.5cm}
\begin{minipage}{0.3\textwidth}
\begin{equation}
\label{equ:primal_format_binary}
\begin{aligned}
    & \textbf{Primal (Binary)} \\
   & \quad \max_{\vx}  ~~\vc^\top \vx \\
 & \quad ~ \text{s.t.}   ~~ \textbf{A} \vx \leq \vb \\
  & \quad \quad \quad ~~ \vx \leq 1 \\
  &  \quad \quad \quad ~~ \vx \geq 0 \\
  & \quad \quad \quad ~~ \vx \in \mathbb{Z}
\end{aligned}
\end{equation}
\end{minipage}
\hspace{0.5cm}
\begin{minipage}{0.3\textwidth}
\begin{equation}
\label{equ:dual_format_binary}
\vspace{0.5cm}
\begin{aligned}
    & \quad \textbf{Dual (Binary)}\\
    &\textbf{(Linear Relaxation)} \\
    & \\
  & \quad \max_{y}   ~~ [\vb^\top, 1^\top] \vy \\
 & ~~~ \text{s.t.}   ~~ [\textbf{A}^\top, I] \vy \geq \vc \\
 & ~~~ \quad \quad \quad \quad ~~ \vy \geq 0
\end{aligned}
\end{equation}
\end{minipage}
\hspace{0.5cm}
\begin{minipage}{0.3\textwidth}
\begin{equation}
\label{equ:slack_variables_binary}
\vspace{0.1cm}
\begin{aligned}
    & ~ ~ \textbf{Slack (Binary)} \\
    & \\
  & \quad \textbf{A} \vx + \vr = \vb \\
 & [\textbf{A}^\top, I] \vy - \vs = \vc \\
 & \qquad \qquad ~~~ \vr \geq 0 \\
 & \qquad \qquad ~~~\vs \geq 0 
\end{aligned}
\end{equation}
\end{minipage}

Considering the inherent structure of binary MILP, we can further decompose the dual solution $\vy$ into two parts: $\vy_1$ (corresponding to regular constraints $\mA \vx \leq \vb$) and $\vy_2$ (corresponding to constraints $\vx \leq 1$). The encoding of binary MILP problem into a bipartite VC graph is illustrated in Table.~\ref{tab:vc_encoding_binary}. And the decoder could be models as Equation.~\ref{equ:decoder_binary}.

\begin{equation}
\resizebox{0.90\textwidth}{!}{$
\label{equ:decoder_binary}
\begin{aligned}
    p_{\theta}(G|G',\vz) = & \ p_{\theta}(d_{c_i}|\vh^{G'}_{c_i},\vz_{c_i}) \cdot \prod_{u \in \mathcal{V}} p_{\theta}(e(c_i,u)|\vh^{G'}_{\mathcal{V}}, \vz_{\mathcal{V}}) \cdot \prod_{u \in \mathcal{V}:e(c_i,u)=1}  p_{\theta}(w_{c_i}|\vh^{G'}_{\mathcal{V}},\vz_{\mathcal{V}})\\
    & \cdot \prod_{u \in \mathcal{C}}p_{\theta}(\vy_{1u}|\vh^{G'}_{\mathcal{C}},\vz_{\mathcal{C}}) p_{\theta}(\vr_u|\vh^{G'}_{\mathcal{C}},\vz_{\mathcal{C}})  \cdot \prod_{u \in \mathcal{V}}p_{\theta}(\vx_u|\vh^{G'}_{\mathcal{V}},\vz_{\mathcal{V}})p_{\theta}(\vs_u|\vh^{G'}_{\mathcal{V}},\vz_{\mathcal{V}})p_{\theta}(\vy_{2u}|\vh^{G'}_{\mathcal{V}},\vz_{\mathcal{V}}),
\end{aligned}$} 
\end{equation}
where the decoder of \dgm{} specifically designed for binary MILP partitions the predicted dual solution $\vy$ into two segments $\vy_1, \vy_2$ and predict each segment separately.

\begin{table}[h]
\renewcommand{\arraystretch}{1.1} 
\setlength{\tabcolsep}{4pt}
\vspace{-0.1cm}
\centering
\captionsetup{font=small} 
\caption{V-C encoding for binary MILP.}
\vspace{-0.3cm}
\label{tab:vc_encoding_binary}
\begin{tabular}{cc}
\hline
object & feature \\ \hline
\multirow{3}{*}{\begin{tabular}[c]{@{}c@{}}constraint\\ node\\ $\mathcal{C} = \{c_1...c_m\}$\end{tabular}} & all 0's \\
 & $\vy_1 = \{\vy_{11}...\vy_{1m}\}$ \\
 & $\vr = \{\vs_1...\vs_m$\} \\ \hline
\multirow{4}{*}{\begin{tabular}[c]{@{}c@{}}variable\\ node\\ $\mathcal{V}=\{v_1...v_n\}$\end{tabular}} & all 1's \\
 & $\vx = \{\vx_1...\vx_n\}$ \\
 & $\vs = \{\vr_1...\vr_n\}$ \\  
  & $\vy_2 = \{\vy_{21}...\vy_{2n}\}$ \\  \hline
edge $\mathcal{E}$ & non-zero weights in $\mA$ \\ \hline
\end{tabular}
\vspace{-0.3cm}
\end{table}

\textbf{Hyper-parameters} Across the four datasets, we set the same learning rate for \dgm{} as $1e-3$. We use the Adam optimizer~\citep{kingma2014adam}. For the SC, we set the $\alpha$ in $\mathcal{L}_{\theta,\phi}$ as $5$, for the CA, the CVS, and the IIS, we set $\alpha$ as $150$. We use the random seed as $123$ for \dgm{} training across all the four datasets.

\subsection{Implementation of baseline}
\label{app:implementation_baseline}
\textbf{`Bowly'} Here we show the implementation of generating instances from scratch with the baseline Bowly~\citep{bowly2019stress}. The generation of matrix $\textbf{A}$ is illustrated in the algorithm.~\ref{alg:bowly_A}. With the generated adjacency matrix $\textbf{A}$, where we manipulate the hyper-parameters during the generation process to ensure that the statistical properties of $\textbf{A}$ align as closely as possible with the original dataset. Specifically, we keep the size of graph ($m, n$) the same as the original dataset and uniformly sample $p_v, p_c$ from $[0,1]$ for all the four datasets. For the other hyper-parameter settings, see Tables.~\ref{tab:bowly_hyperparameter}.

\begin{table}[h]
\centering
\caption{The hyper-parameter selection of the Bowly baseline.}
\vspace{-0.3cm}
\label{tab:bowly_hyperparameter}
\begin{tabular}{@{}cccc@{}}
\toprule
 & density & $\mu_{\textbf{A}}$ & $\sigma_{\textbf{A}}$ \\ \midrule
SC & $\mathcal{U}\{0.15,0.20,0.25,0.30,0.35\}$ & -1 & 0 \\
CA & 0.05 & 1 & $\mathcal{U}(0.1,0.3)$ \\
CVS & 0.0013 & 0.2739 & 0.961 \\
IIS & 0.0488 & -1 & 0 \\ \bottomrule
\end{tabular}
\end{table}

Then we uniformly sample the solution space $\vx, \vy, \vs, \vr$ with intervals defined by their corresponding maximum and minimum from the training dataset. Then we deduce $\vb, \vc$ to get the new MILP instances.

\begin{algorithm}
\caption{Bowly - generation of matrix $\textbf{A}$}
\label{alg:bowly_A}
\begin{algorithmic}[1]
\Require $n \in [1,\infty), m \in [1,\infty), \rho \in (0,1], p_v \in [0,1], p_c \in [0,1], \mu_A \in (-\infty, \infty), \sigma_A \in (0, \infty)$
\Ensure $\text{Constraint matrix} \textbf{A} \in \mathbb{Q}^{m \times n}$
\State Set target variable degree $d_(u_i) = 1$ for randomly selected i,0 for all others
\State Set target constraint degree $d_(u_i) = 1$ for randomly selected i,0 for all others
\State $e \leftarrow 1$
\While{$e < \rho m n$}
\State $s \leftarrow$ draw $n$ values from $U(0,1)$
\State $t \leftarrow$ draw $m$ values from $U(0,1)$
\State Increment the degree of variable node $i$ with maximum $p_v \frac{d(u_i)}{e} + s_i$
\State Increment the degree of constraint node $j$ with maximum $p_c \frac{d(v_j)}{e} + t_j$
\State $e \leftarrow e + 1$
\EndWhile
\For{$i = 1,...,n$}
\For{$j = 1,...,m$}
\State $r \leftarrow$ draw from $U(0,1)$
\If{$r < \frac{d(u_i)d(v_j)}{e}$}
\State Add edge $(i,j)$ to VC
\EndIf
\EndFor
\EndFor
\While{$\min((d(u_i),d(v_j)) = 0$}
\State Choose $i$ from $\{i|d(u_i) = 0\}$, or randomly if all $d(u_i) > 0$
\State Choose $j$ from $\{j|d(v_j) = 0\}$, or randomly if all $d(v_j) > 0$
\State Add edge $(i,j)$ to VC
\EndWhile
\For{$(i,j) \in E(VC)$}
\State $a_{ij} = \mathcal{N}(\mu_{\textbf{A}}, \sigma_{\textbf{A}})$
\EndFor \\
\Return $\textbf{A}$
\end{algorithmic}
\end{algorithm}

\textbf{`Random'} We use exactly the same network architecture and generation process as \dgm. The key difference is that instead of utilizing the trained NN, we uniformly sample the variables $d_{c_i}, e(c_i,u), w_{c_i}, \vy_1, \vs, \vx, \vr, \vy_2$ required for decoder prediction within intervals delineated by the maximum and minimum values of each variable from the training set, simulating the random parameters of an untrained neural network.

\subsection{Implementation of the structural statistical characteristics}
\label{app:implementation_statistic_metric}
The explanation of various statistical metrics used for comparing the structural similarity of MILP problem instances is detailed as shown in Table.~\ref{tab:explanation_statistic}. Specific numerical values for different metrics for the SC and CA problems can be found in Table.~\ref{tab:setcover_statistics} and Table.~\ref{tab:cauctions_statistics}, respectively.

\begin{table}[h]
\centering
\caption{Explanation of the statistic metrics of the MILP instances}
\vspace{-0.35cm}
\label{tab:explanation_statistic}
\begin{tabular}{@{}cc@{}}
\toprule
name & explanation \\ \midrule
density mean & the average number of non zero values in the constraint matrix \\
cons degree mean & the average number of constraint node degree \\
cons degree std & the standard variance of constraint node degree \\
var degree mean & the average number of variable node degree \\
var degree std & the standard variance of variable node degree \\
$\vb$ mean & the average $\vb$ value \\
$\vb$ std & the standard variance of $\vb$ value \\
$\vc$ mean & the average value of $\vc$ \\
$\vc$ std & the standard variance of $\vc$ value \\ \bottomrule
\end{tabular}
\end{table}

For each statistic metric $i$ shown in Table.~\ref{tab:explanation_statistic}, we begin by collecting lists of the values from four data sources: the original dataset, the data generated by the `Bowly' baseline, the data generated by the `random' baseline, and data generated by \dgm. Each data source contains $1000$ instances. We then employ the lists from the four data sources to approximate four categorical distributions. Utilizing the \textit{numpy.histogram} function, we set the number of bins to the default value of $10$, with the min and max values derived from the collective minimum and maximum of a given metric across the four data sources, respectively.
Next, we employ Jensen-Shannon (JS) divergence $D_{js}^i$ via the function $\textit{scipy.spatial.distance.jensenshannon}$~\citep{2020SciPy-NMeth} to quantify the divergence between the original samples and the rest three data sources, resulting in $\text{score}_i$ for each statistical metric.

\begin{equation}
    \text{score}_i = (\max(D_{js}) - D_{js}^i) / (\max(D_{js}) - \min(D_{js})),
\end{equation}where $\max(D_{js}), \min(D_{js})$ are the maximum and minimum of JS divergence across all the metrics.

Then we average the score for each statistic metric to obtain the final similarity score, as is shown in Table.~\ref{tab:js_score}:
\begin{equation}
    \text{score} = \frac{1}{9} \sum_{i=1}^9 \text{score}_i.
\end{equation}

\subsection{Implementation of Data Sharing for Solver Configuration Tuning}
\label{app:implementation_downstream1}
Below are the hyper-parameters that we randomly sample to test the positive-correlation of different dataset pairs. We adhere to the configuration established in mainstream solver tuning literature to select the parameters requiring adjustment~\cite{hutter2011sequential, lindauer2018warmstarting, lindauer2022smac3}, . For a detailed explanation of each parameter, please refer to the SCIP documentation\footnote{https://www.scipopt.org/doc/html/PARAMETERS.php}.

\begin{table}[h]
\centering
\captionsetup{font=small}
\caption{The selected SCIP hyper-parameters and the range to randomly select from.}
\vspace{-0.35cm}
\label{tab:scip_hyper_parameters}
\begin{tabular}{@{}cccc@{}}
\toprule
params & whole range/choice & default & our range/choice \\ \midrule
branching/scorefunc & s, p, q & s & s, p, q \\
branching/scorefac & {[}0, 1{]} & 0.167 & {[}0, 1{]} \\
branching/preferbinary & True, False & False & True, False \\
branching/clamp & {[}0,0.5{]} & 0.2 & {[}0,0.5{]} \\
branching/midpull & {[}0,1{]} & 0.75 & {[}0,1{]} \\
branching/midpullreldomtrig & {[}0,1{]} & 0.5 & {[}0,1{]} \\
branching/lpgainnormalize & d, l, s & s & d, l, s \\
lp/pricing & l, a, f, p, s, q, d & l & l, a, f, p, s, q, d \\
lp/colagelimit & {[}-1,2147483647{]} & 10 & {[}0,100{]} \\
lp/rowagelimit & {[}-1,2147483647{]} & 10 & {[}0,100{]} \\
nodeselection/childsel & d, u, p, I, l, r, h & h & d, u, p, I, l, r, h \\
separating/minortho & {[}0,1{]} & 0.9 & {[}0,1{]} \\
separating/minorthoroot & {[}0,1{]} & 0.9 & {[}0,1{]} \\
separating/maxcuts & {[}0,2147483647{]} & 100 & {[}0,1000{]} \\
separating/maxcutsroot & {[}0,2147483647{]} & 2000 & {[}0,10000{]} \\
separating/cutagelimit & {[}-1,2147483647{]} & 80 & {[}0,200{]} \\
separating/poolfreq & {[}-1,65534{]} & 10 & {[}0,100{]} \\ \bottomrule
\end{tabular}
\end{table}

\subsection{Implementation of Optimal Value Prediction via ML}
\label{app:implementation_downstream2}
\textbf{Neural Network Architecture} In this downstream task, 
We also use the bipartite GNN backbone which is exactly the same as that in ~\cite{gasse2019exact}. We use an MLP layer and global mean pooling to produce the optimal objective value prediction. The learning rate is set as $1e-3$.

\section{Supplementary Experiment Results}
\subsection{statistical characteristics of the generated instances}
\label{app:results_statistic_metric}
We show the specific value of each statistic metric of the original dataset, and the datasets generated by the baselines as well as \dgm{} on the SC and the CA problem in Table.~\ref{tab:setcover_statistics} and Table.~\ref{tab:cauctions_statistics} respectively.
\begin{table}[t]
\caption{Statistic value comparison across the original dataset and the generated datasets with different constraints replacement rates on the set covering (SC) problem. `resolving time' calculates under default configuration of pySCIPopt. `density' represents the ratio of non zero entries in the constraint matrix. `cons degree' denotes the degree of constraint nodes, `var degree' stands for the degree of variable nodes. $\vb$ denotes the right hand side vector of the MILP, and $\vc$ is the objective coefficient vector.}
\label{tab:setcover_statistics}
\vspace{-0.35cm}
\resizebox{\linewidth}{!}{\begin{tabular}{cccccccccccc}
\hline
 & replace ratio & \begin{tabular}[c]{@{}c@{}}resolving\\ time (s)\end{tabular} & \begin{tabular}[c]{@{}c@{}}density\\ mean\end{tabular} & \begin{tabular}[c]{@{}c@{}}cons degree\\ mean\end{tabular} & \begin{tabular}[c]{@{}c@{}}cons degree\\ std\end{tabular} & \begin{tabular}[c]{@{}c@{}}var degree\\ mean\end{tabular} & \begin{tabular}[c]{@{}c@{}}var degree\\ std\end{tabular} & b mean & b std & c mean & c std \\ \hline
original & - & 0.821 & 0.251 & 100.700 & 8.447 & 50.350 & 6.854 & -1.0 & 0.0 & 50.490 & 28.814 \\
Bowly & - &  & 0.205 & 82.312 & 35.131 & 41.305 & 21.628 & 1.484 & 3.504 & 403.208 & 198.571 \\ \hline
random & 0.01 & 127.723 & 0.251 & 100.774 & 9.853 & 50.387 & 6.841 & 1.294 & 3.045 & 422.65 & 65.078 \\
random & 0.05 & 143.883 & 0.253 & 101.039 & 14.070 & 50.519 & 6.787 & 1.218 & 3.123 & 431.422 & 66.082 \\
random & 0.10 & 187.851 & 0.253 & 101.357 & 17.706 & 50.678 & 6.727 & 1.164 & 3.210 & 441.696 & 67.250 \\
random & 0.20 & 304.216 & 0.255 & 101.900 & 22.808 & 50.950 & 6.607 & 1.062 & 3.351 & 460.696 & 69.379 \\
random & 0.50 & 1312.595 & 0.258 & 103.348 & 31.305 & 51.674 & 6.375 & 0.664 & 3.629 & 509.337 & 74.864 \\ \hline
ours & 0.01 & 83.681 & 0.251 & 100.700 & 8.876 & 50.350 & 7.431 & -0.515 & 1.351 & 44.863 & 0.939 \\
ours & 0.05 & 70.476 & 0.251 & 100.712 & 10.202 & 50.356 & 9.977 & -0.456 & 1.386 & 44.958 & 0.984 \\
ours & 0.10 & 54.650 & 0.251 & 100.738 & 11.365 & 50.369 & 13.354 & -0.413 & 1.441 & 45.057 & 1.032 \\
ours & 0.20 & 54.830 & 0.251 & 100.754 & 12.872 & 50.377 & 19.992 & -0.368 & 1.576 & 45.112 & 1.071 \\
ours & 0.50 & 22.462 & 0.252 & 100.830 & 14.433 & 50.415 & 37.017 & -0.005 & 1.271 & 44.967 & 1.872 \\ \hline
\end{tabular}}
\end{table}

\begin{table}[t]
 \caption{Statistic value comparison across the original dataset and the generated datasets with different constraints replacement rates on the combinatorial auction (CA) problem. `resolving time' calculates under default configuration of pySCIPopt. `density' represents the ratio of non zero entries in the constraint matrix. `cons degree' denotes the degree of constraint nodes, `var degree' stands for the degree of variable nodes. $\vb$ denotes the right hand side vector of the MILP, and $\vc$ is the objective coefficient vector.}
\label{tab:cauctions_statistics}
\vspace{-0.35cm}
\resizebox{\linewidth}{!}{\begin{tabular}{cccccccccccc}
\hline
 & replace ratio & \begin{tabular}[c]{@{}c@{}}resolving\\ time (s)\end{tabular} & \begin{tabular}[c]{@{}c@{}}density\\ mean\end{tabular} & \begin{tabular}[c]{@{}c@{}}cons degree\\ mean\end{tabular} & \begin{tabular}[c]{@{}c@{}}cons degree\\ std\end{tabular} & \begin{tabular}[c]{@{}c@{}}var degree\\ mean\end{tabular} & \begin{tabular}[c]{@{}c@{}}var degree\\ std\end{tabular} & b mean & b std & c mean & c std \\ \hline
original & - & 1.360 & 0.050 & 14.538 & 13.834 & 5.578 & 3.253 & 1.0 & 0.0 & 330.999 & 234.444 \\
Bowly & - & 0.281 & 0.048 & 14.415 & 13.633 & 5.544 & 7.262 & 1.668 & 1.617 & 510.211 & 1101.065 \\ \hline
random & 0.01 & 0.416 & 0.051 & 14.664 & 13.970 & 5.634 & 3.240 & 1.748 & 1.602 & 524.961 & 563.436 \\
random & 0.05 & 0.502 & 0.054 & 15.225 & 14.531 & 5.878 & 3.201 & 1.792 & 1.647 & 560.369 & 561.074 \\
random & 0.10 & 0.555 & 0.056 & 15.877 & 15.088 & 6.152 & 3.161 & 1.855 & 1.706 & 598.047 & 555.956 \\
random & 0.20 & 0.821 & 0.061 & 17.098 & 15.953 & 6.658 & 3.106 & 1.966 & 1.797 & 669.168 & 552.853 \\
random & 0.30 & 1.056 & 0.065 & 18.186 & 16.527 & 7.105 & 3.070 & 2.053 & 1.850 & 735.284 & 548.606 \\
random & 0.50 & 2.353 & 0.072 & 19.959 & 17.222 & 7.837 & 3.006 & 2.267 & 1.972 & 841.971 & 545.471 \\ \hline
ours & 0.01 & 0.361 & 0.050 & 14.490 & 13.776 & 5.565 & 3.253 & 1.645 & 1.348 & 361.711 & 264.798 \\
ours & 0.05 & 0.360 & 0.050 & 14.361 & 13.609 & 5.535 & 3.286 & 1.609 & 1.325 & 351.417 & 261.927 \\
ours & 0.10 & 0.301 & 0.050 & 14.205 & 13.401 & 5.500 & 3.366 & 1.589 & 1.329 & 342.702 & 261.313 \\
ours & 0.20 & 0.217 & 0.049 & 13.819 & 12.854 & 5.412 & 3.586 & 1.525 & 1.315 & 324.282 & 260.848 \\
ours & 0.30 & 0.140 & 0.047 & 13.454 & 12.330 & 5.344 & 3.847 & 1.454 & 1.280 & 304.911 & 260.949 \\
ours & 0.50 & 0.055 & 0.045 & 12.869 & 11.379 & 5.254 & 4.282 & 1.350 & 1.233 & 271.474 & 255.515 \\ \hline
\end{tabular}}
\end{table}

\subsection{Data Sharing for Solver configuration Tuning}
\label{app:results_downstream2_cvs_iis}
\textbf{CVS and IIS} There are five total instances in CVS, comprising three for training \dgm{} and the downstream predictor and two for testing. The IIS has two instances, one for training and one for testing (with allocation based on alphabetical order). Please refer to Table.~\ref{tab:downstream2_cvs_iis} for the model's performance. `ground truth' corresponds to the true values of the optimal objectives for each problem. Models trained exclusively on the `original' training set exhibit superior fitting and more accurate predictions on the training set itself.
However, models trained on the datasets where we introduce $20$ additional newly generated instances by \dgm{} with varying constraint replacement ratio $\gamma$ not only demonstrate minimal gap in prediction on the training set towards the models trained solely on the original data compared with the baselines, but also showcase improved predictive performance on previously unseen test sets. This underscores the notion that the \dgm-generated data can indeed increase structural and solution label diversity to a certain extent, thereby enhancing the generalization capability and overall performance of the models. Again, similar to the previous two experiments, `Bowly' degrades the predictive performance of the model, `random' results in marginal improvement in out-of-distribution prediction accuracy.

\begin{table}[t]
\caption{The predicted value and relative mean square error (MSE) of the optimal objective value on the CVS and the IIS problem. In the CVS, `cvs08r139-94',`cvs16r70-62',`cvs16r89-60' are used as training data, `cvs16r106-72',`cvs16r128-89' are used as testing data. In the IIS, `iis-glass-cov' is used as the training data, `iis-hc-cov' is used as the testing data. `original' shows the performance of the model trained merely on the three (CVS) or single (IIS) original training instances.}
\vspace{-0.3cm}
\label{tab:downstream2_cvs_iis}
\renewcommand{\arraystretch}{1.2}
\resizebox{\textwidth}{!}{\begin{tabular}{cc|cccccc|cccc|cc|cc}
\hline
\multirow{2}{*}{} & \multirow{2}{*}{} & \multicolumn{6}{c|}{in-distribution} & \multicolumn{4}{c|}{out-of-distributio} & \multicolumn{2}{c|}{in-distribution} & \multicolumn{2}{c}{out-of-distribution} \\ \cline{3-16} 
 &  & \multicolumn{2}{c}{cvs08r139-94} & \multicolumn{2}{c}{cvs16r70-62} & \multicolumn{2}{c|}{cvs16r89-60} & \multicolumn{2}{c}{\textbf{cvs16r106-72}} & \multicolumn{2}{c|}{\textbf{cvs16r128-89}} & \multicolumn{2}{c|}{iis-glass-cov} & \multicolumn{2}{c}{\textbf{iis-hc-cov}} \\ \hline
dataset & ratio & value & msre & value & msre & value & msre & value & msre & value & msre & value & msre & value & msre \\ \hline
ground truth & - & 116 & 0 & 42 & 0 & 65 & 0 & 81 & 0 & 97 & 0 & -17 & 0 & -21 & 0 \\ \hline
original & - & \textbf{115.994} & \textbf{2e-9} & \textbf{41.998} & \textbf{1e-9} & \textbf{64.997} & \textbf{1e-9} & 77.494 & 0.001 & 89.258 & 0.006 & \textbf{-20.999} & \textbf{3e-10} & -94.451 & 20.756 \\ \hline
Bowly & - & 65.712 & 0.187 & 82.353 & 0.923 & 66.858 & 8e-6 & 61.504 & 0.057 & 66.045 & 0.101 & -88.756 & 17.816 & -88.756 & 17.816 \\ \hline
random & 0.01 & 138.459 & 0.037 & 45.312 & 0.006 & 67.875 & 0.001 & 58.754 & 0.075 & 68.192 & 0.088 & -22.263 & 3e-4 & -83.146 & 15.139 \\
random & 0.05 & 163.412 & 0.167 & 34.571 & 0.031 & 45.605 & 0.089 & 41.110 & 0.242 & 24.952 & 0.551 & -20.695 & 2e-4 & -82.297 & 14.753 \\
random & 0.10 & 116.824 & 5e-5 & 60.440 & 0.192 & 79.152 & 0.047 & 68.641 & 0.023 & 79.321 & 0.033 & -20.991 & 1e-7 & -807.680 & 2163.238 \\
random & 0.20 & 144.962 & 0.062 & 79.849 & 0.812 & 99.552 & 0.282 & 71.821 & 0.0128 & 99.898 & 8e-4 & -21.678 & 0.001 & -227.610 & 153.482 \\
random & 0.50 & 159.807 & 0.142 & 49.364 & 0.030 & 65.213 & 1e-5 & 103.960 & 0.080 & 122.321 & 0.068 & -21.633 & 9e-3 & -100.224 & 23.966 \\ \hline
\dgm & 0.01 & 116.981 & 7e-5 & 42.197 & 2e-5 & 64.876 & 3e-6 & 78.646 & 8e-4 & \textbf{96.831} & \textbf{3e-6} & -20.933 & 1e-5 & -90.556 & 18.721 \\
\dgm & 0.05 & 161.558 & 0.154 & 26.181 & 0.141 & 23.439 & 0.408 & 66.119 & 0.033 & 76.119 & 0.046 & -21.108 & 2e-5 & -61.217 & 6.765 \\
\dgm & 0.10 & 118.609 & 5e-4 & 45.461 & 0.006 & 67.216 & 0.001 & \textbf{80.706} & \textbf{1e-5} & 95.745 & 1e-4 & -20.976 & 1e-6 & -65.385 & 8.101 \\
\dgm & 0.20 & 114.622 & 1e-4 & 42.933 & 4e-4 & 62.627 & 0.001 & 83.379 & 8e-4 & 120.641 & 0.0594 & -20.159 & 0.001 & \textbf{-55.926} & \textbf{5.243} \\
\dgm & 0.50 & 120.361 & 0.001 & 44.472 & 0.003 & 69.287 & 0.004 & 84.870 & 0.002 & 104.333 & 0.005 & -21.009 & 2e-7 & -90.427 & 18.655 \\ \hline
\end{tabular}
}
\end{table}

We present the visual results for CA, SC, and IIS datasets, see Fig.~\ref{fig:downstream1_visualization_ca}, ~\ref{fig:downstream1_visualization_sc}, ~\ref{fig:downstream1_visualization_iis}.
\begin{figure}[t]
  \centering
  \begin{minipage}{\textwidth}
    \centering
    \begin{subfigure}[b]{0.28\textwidth}
      \centering
      \includegraphics[width=\textwidth]{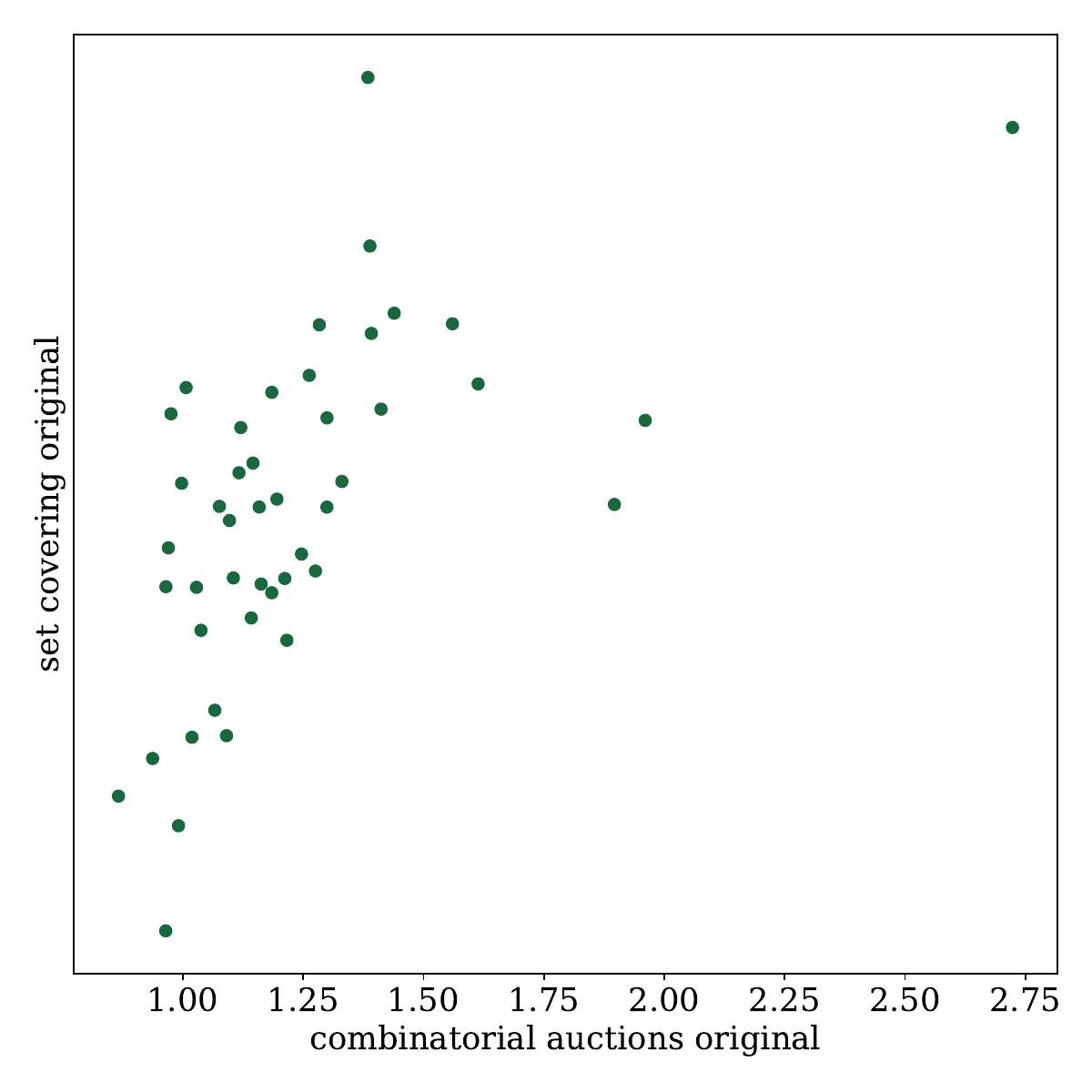}
      \vspace{-0.6cm}
      \caption{CA - SC}
      \label{fig:downstream1_ca_sc}
    \end{subfigure}
    \hspace{-0.3cm}
    \begin{subfigure}[b]{0.28\textwidth}
      \centering
      \includegraphics[width=\textwidth]{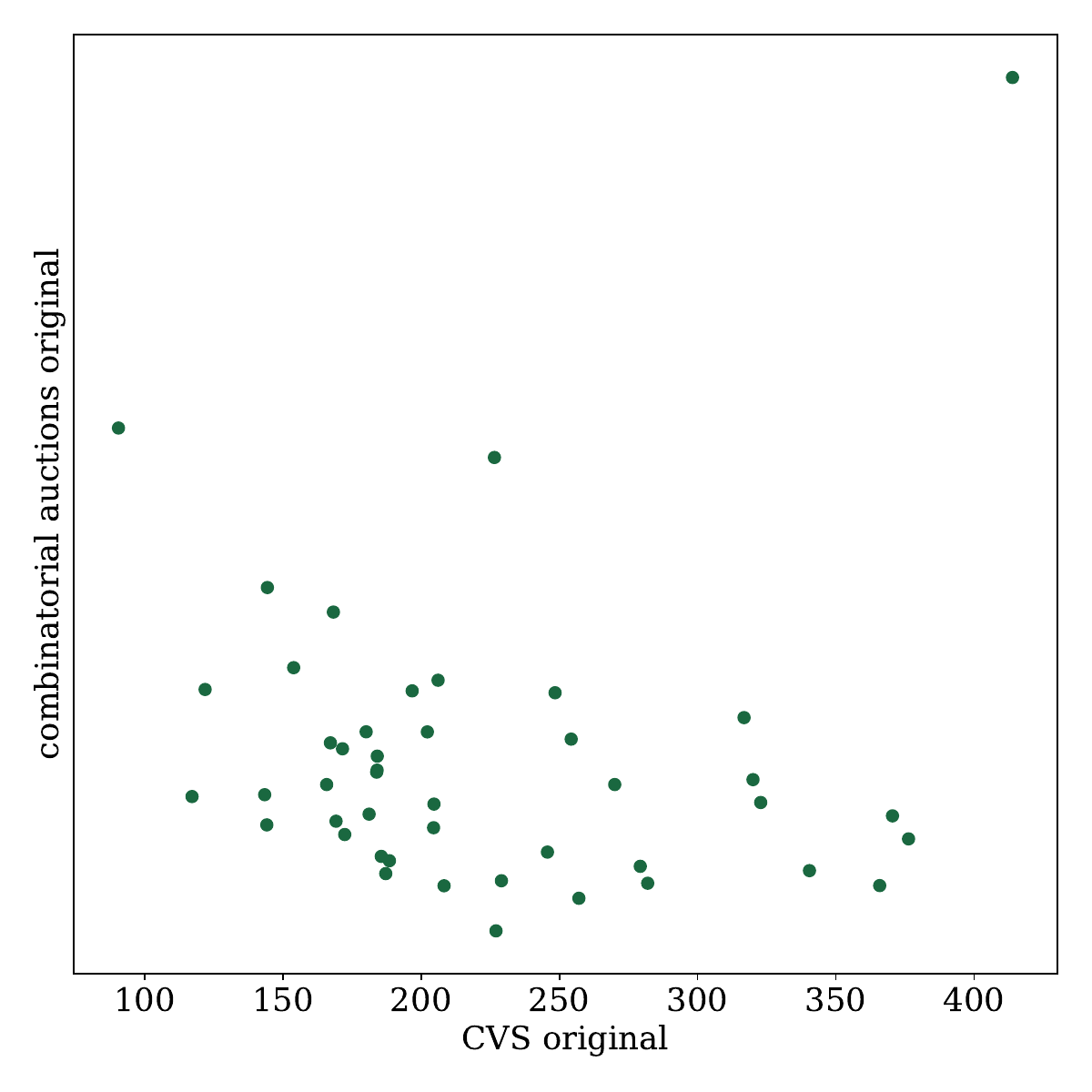}
      \vspace{-0.6cm}
      \caption{CVS - CA}
      \label{fig:downstream1_cvs_ca}
    \end{subfigure}
    \hspace{-0.3cm}
    \begin{subfigure}[b]{0.28\textwidth}
      \centering
      \includegraphics[width=\textwidth]{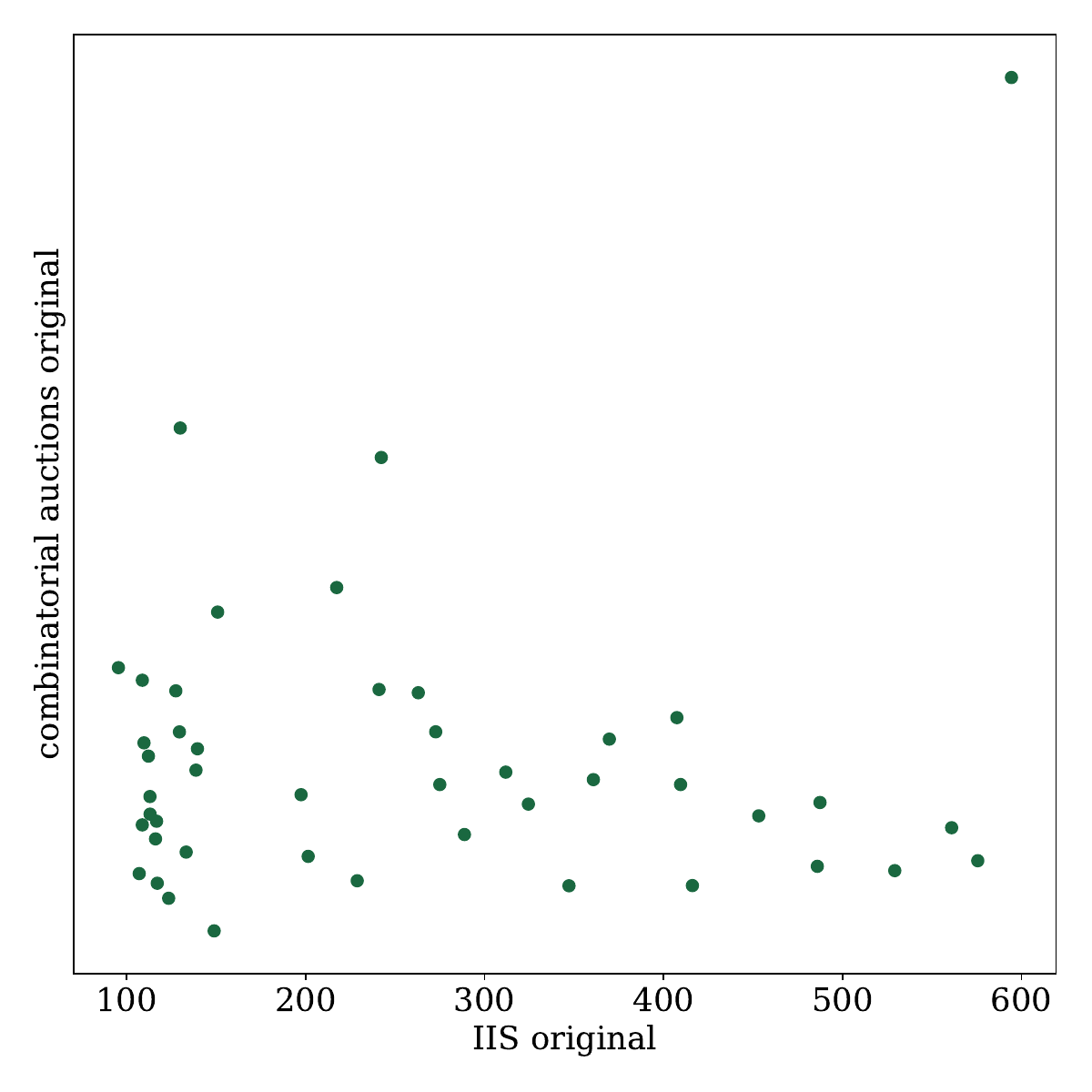}
      \vspace{-0.6cm}
      \caption{IIS - CA}
      \label{fig:downstream1_iis_ca}
    \end{subfigure}
    \vfill
    \vspace{-0.1cm} 
    \begin{subfigure}[b]{0.28\textwidth}
      \centering
      \includegraphics[width=\textwidth]{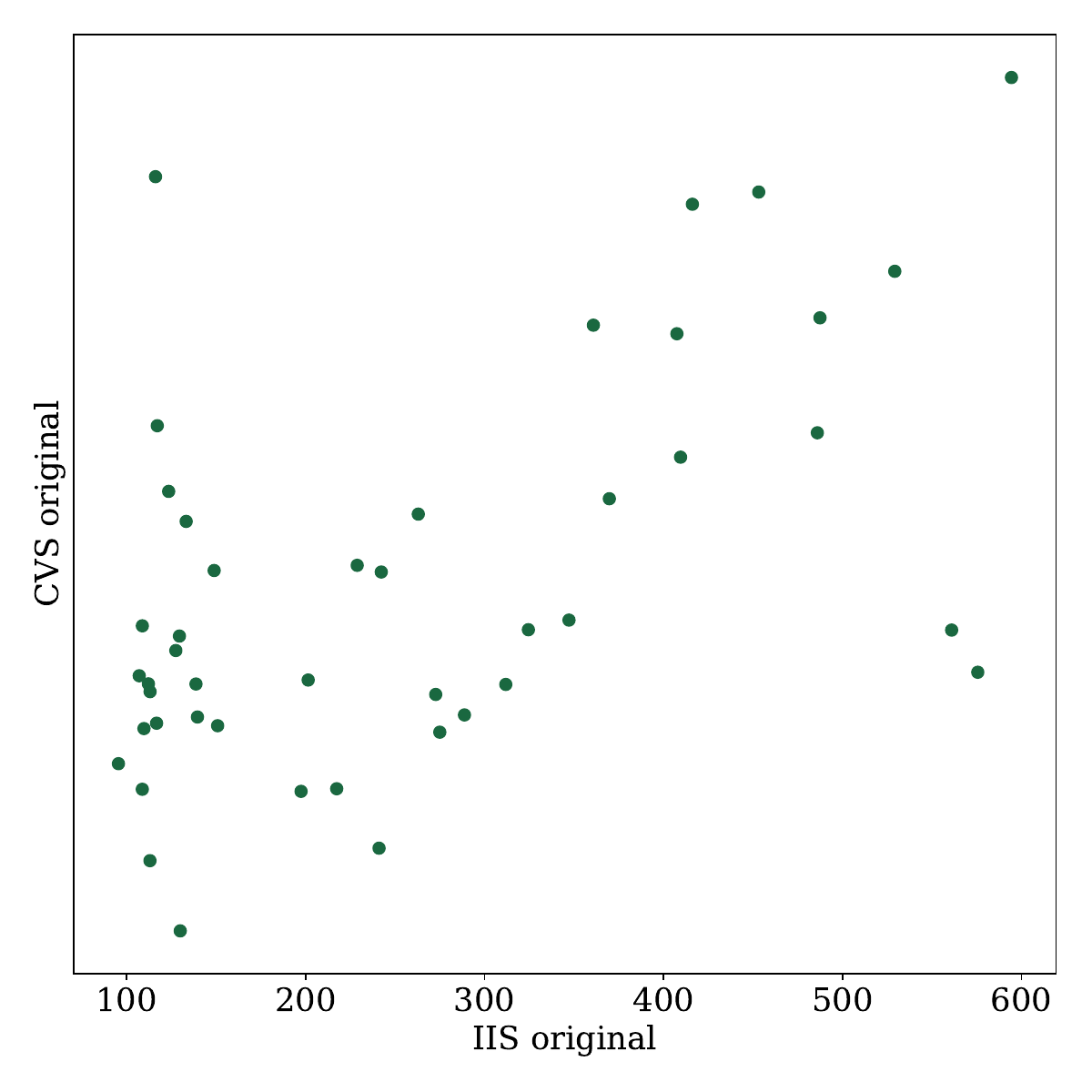}
      \vspace{-0.6cm}
      \caption{IIS - CVS}
      \label{fig:downstream1_iis_cvs}
    \end{subfigure}
     \hspace{-0.3cm}
    \begin{subfigure}[b]{0.28\textwidth}
      \centering
      \includegraphics[width=\textwidth]{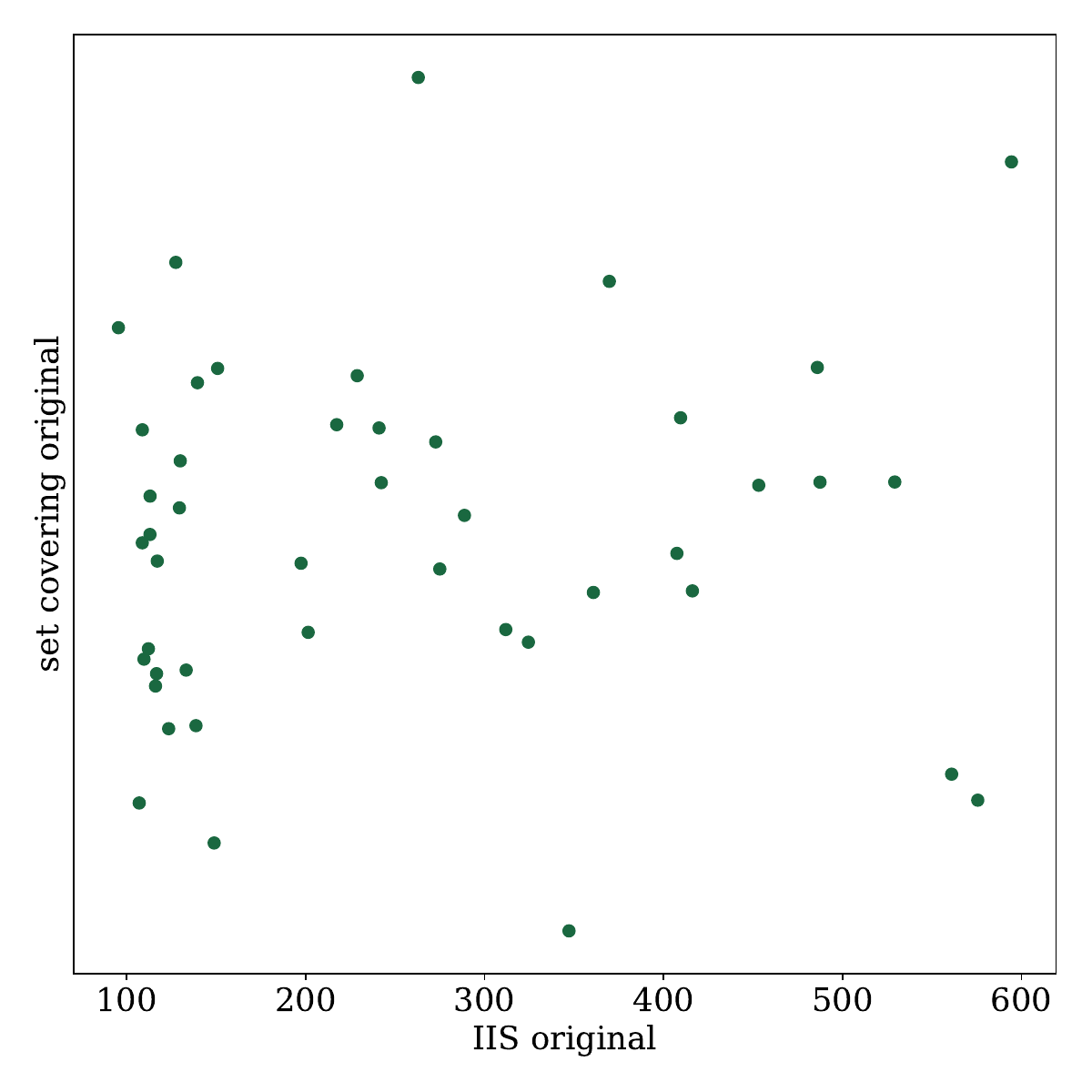}
      \vspace{-0.6cm}
      \caption{IIS - SC}
      \label{fig:downstream1_iis_sc}
    \end{subfigure}
     \hspace{-0.3cm}
    \begin{subfigure}[b]{0.28\textwidth}
      \centering
      \includegraphics[width=\textwidth]{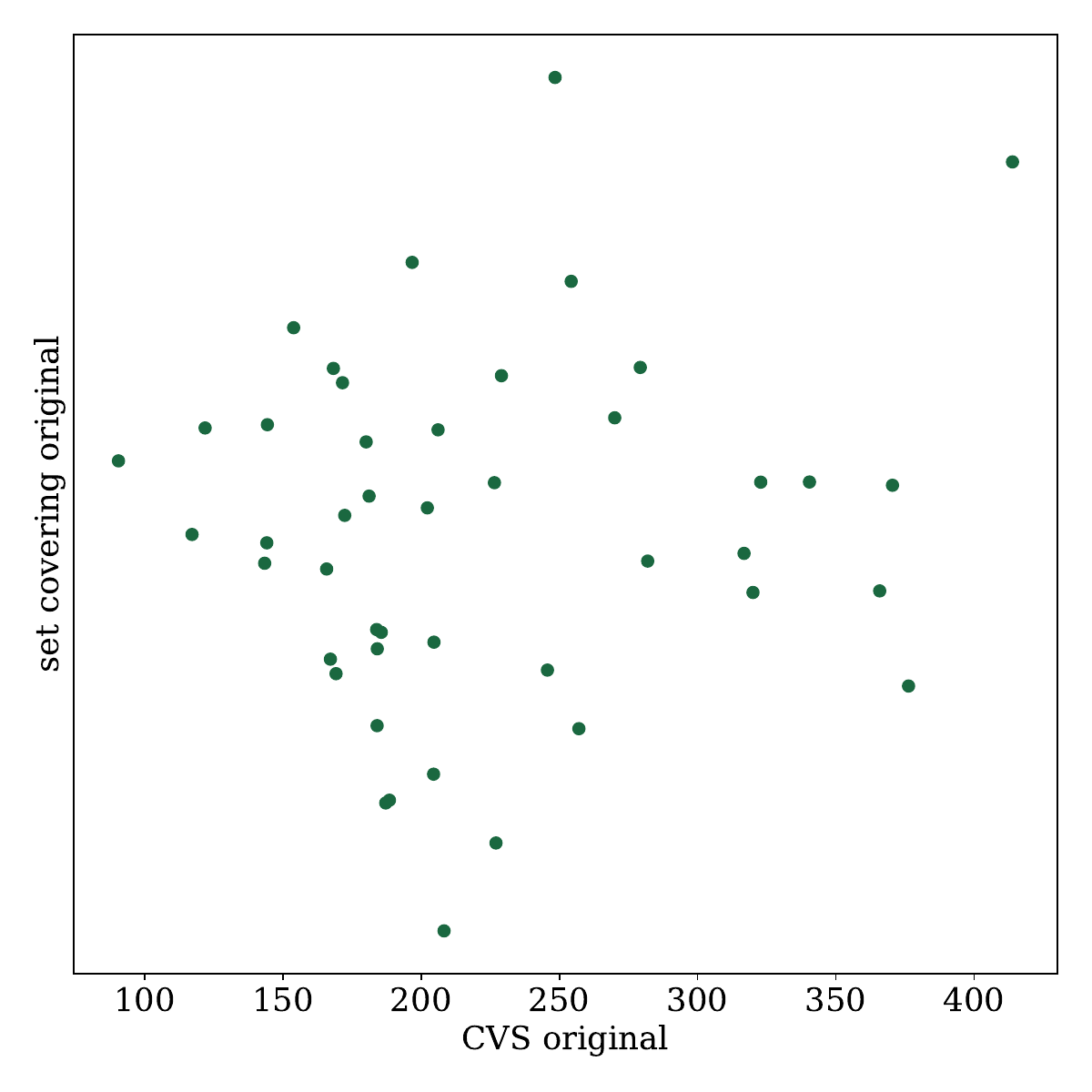}
      \vspace{-0.6cm}
      \caption{CVS - SC}
      \label{fig:downstream1_cvs_sc}
    \end{subfigure}
  \end{minipage}
  \caption{The solution time of SCIP with different parameter sets across different original datasets.}
    \label{fig:downstream1_visualization_cross_datasets}
\end{figure}

\begin{figure}[t]
  \centering
  \begin{minipage}{\textwidth}
    \centering
    \begin{subfigure}[b]{0.20\textwidth}
      \centering
      \includegraphics[width=\textwidth]{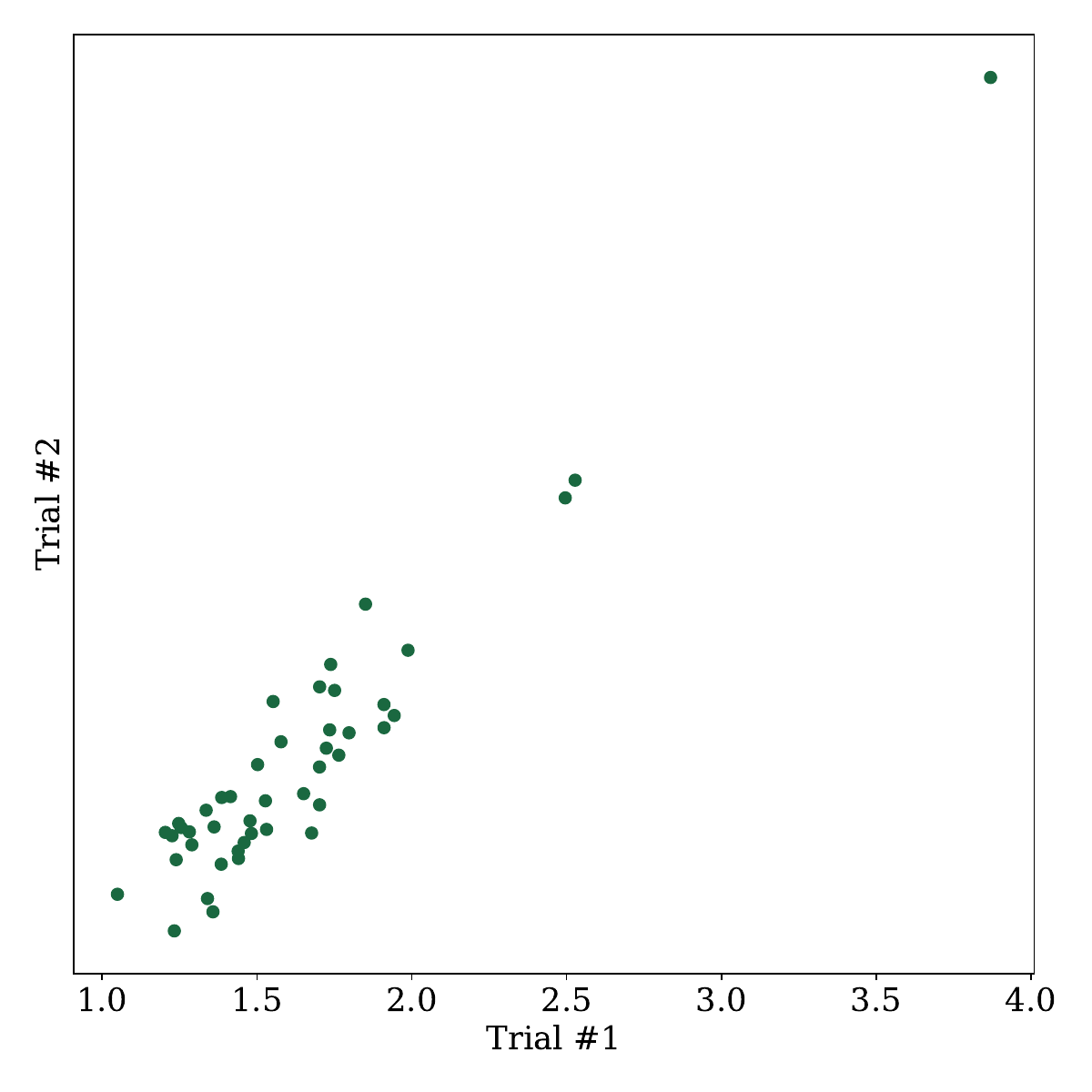}
      \vspace{-0.6cm}
      \caption{two trials}
      \label{fig:downstream1_ca_ca}
    \end{subfigure}
    \centering
    \begin{subfigure}[b]{0.20\textwidth}
      \centering
      \includegraphics[width=\textwidth]{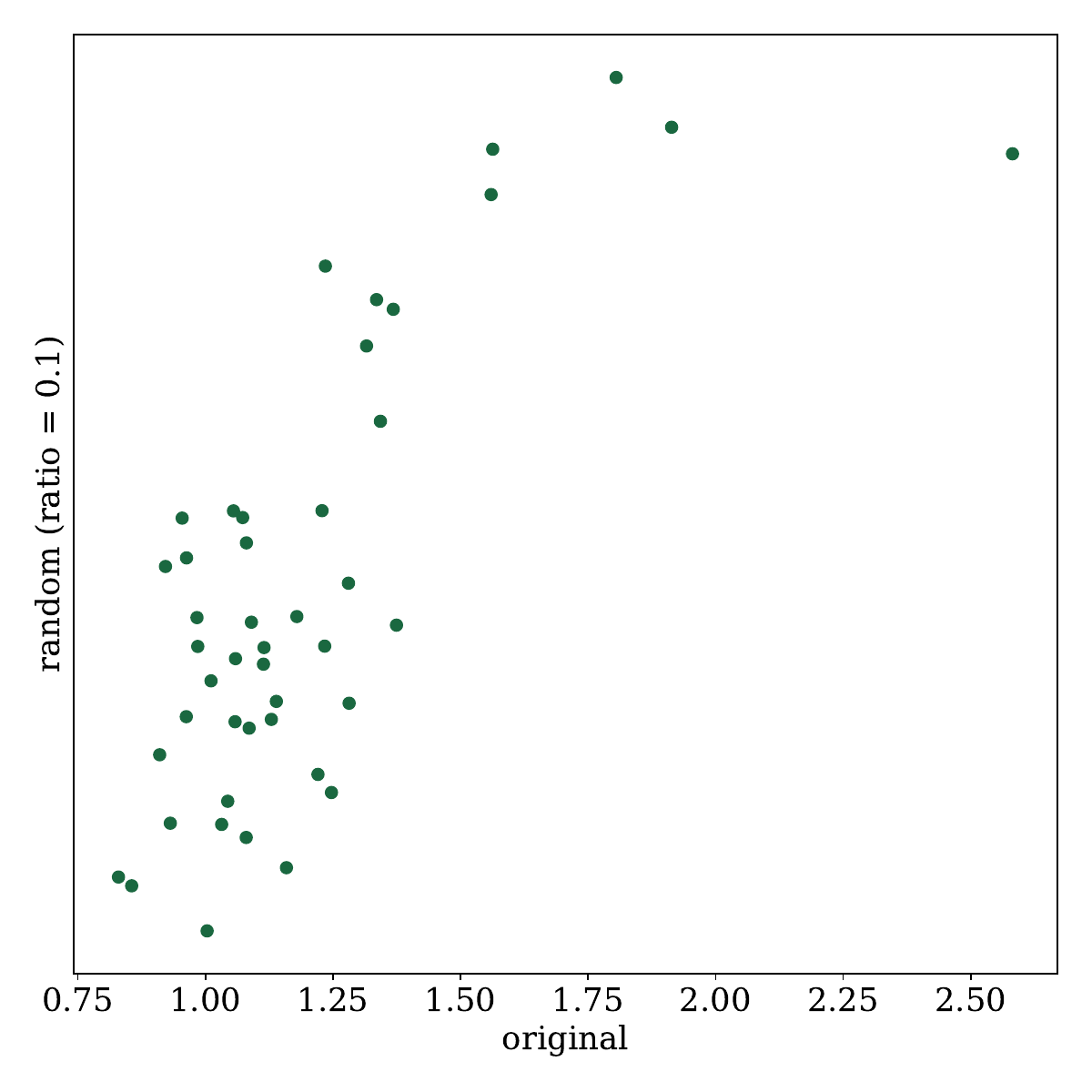}
      \vspace{-0.6cm}
      \caption{random ($\gamma$ = 0.1)}
      \label{fig:downstream1_ca_random010}
    \end{subfigure}
    \begin{subfigure}[b]{0.20\textwidth}
      \centering
      \includegraphics[width=\textwidth]{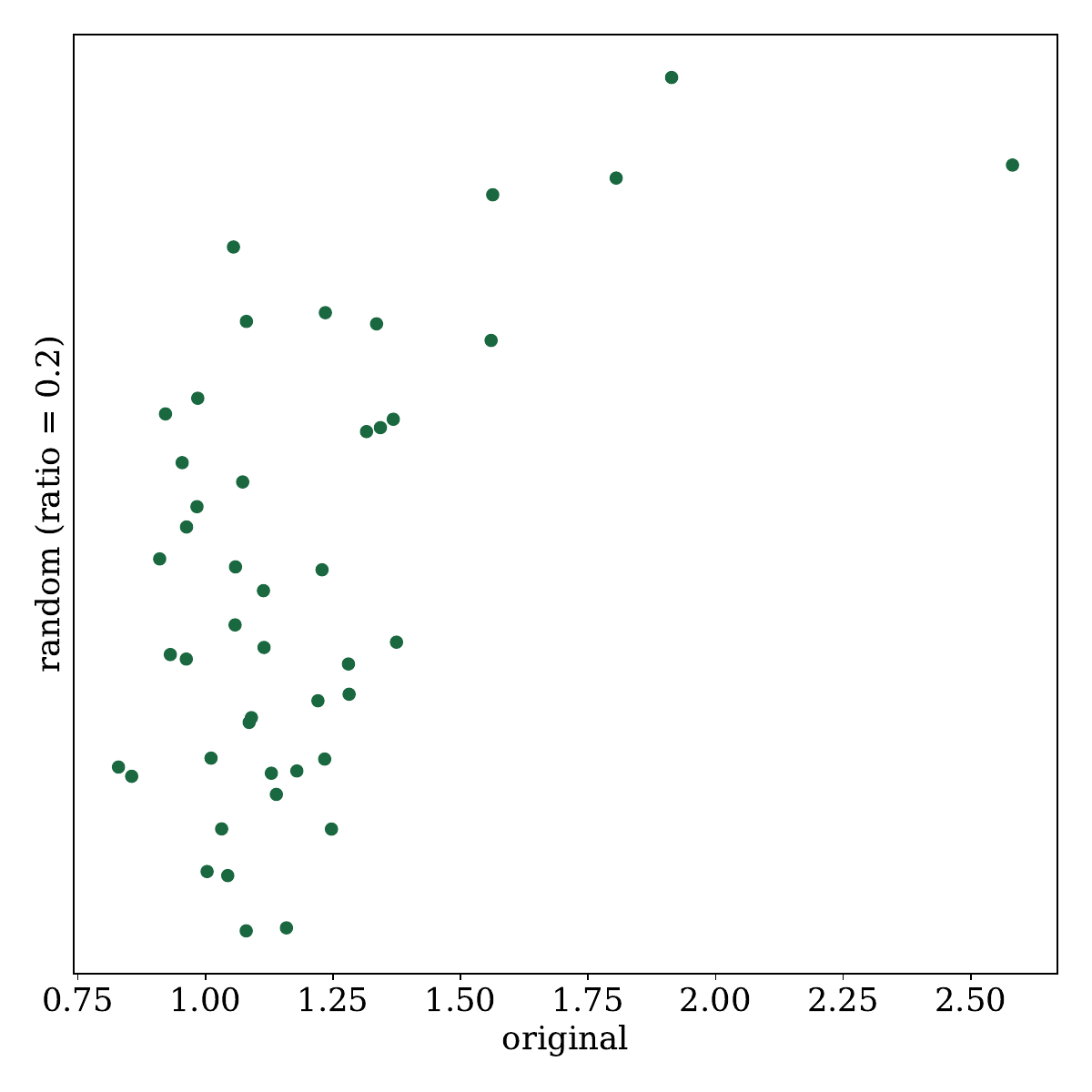}
      \vspace{-0.6cm}
      \caption{random ($\gamma$ = 0.2)}
      \label{fig:downstream1_ca_random020}
    \end{subfigure}
    \begin{subfigure}[b]{0.20\textwidth}
      \centering
      \includegraphics[width=\textwidth]{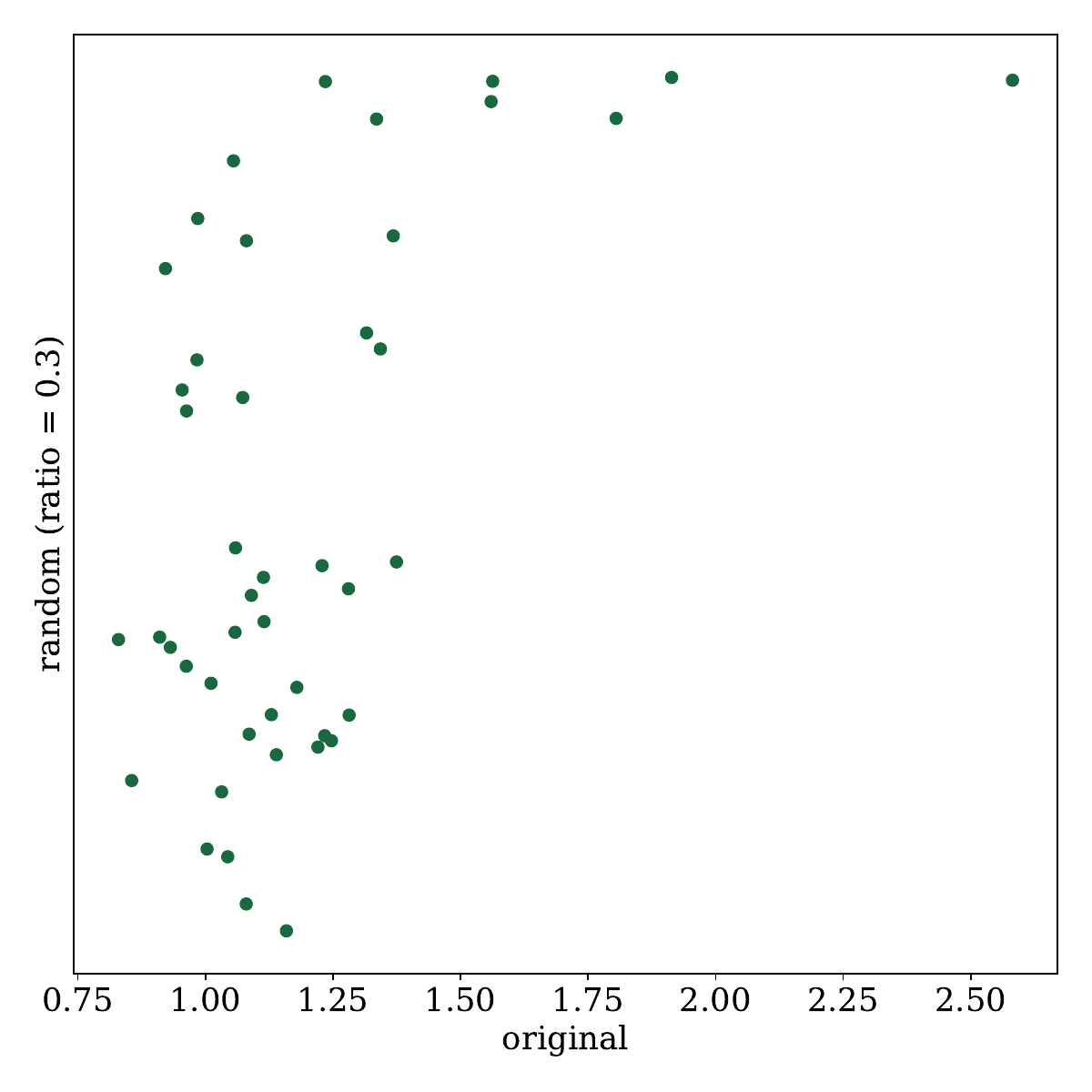}
      \vspace{-0.6cm}
      \caption{random ($\gamma$ = 0.3)}
      \label{fig:downstream1_ca_random030}
    \end{subfigure}
    \vfill
    \vspace{0.1cm} 
    \begin{subfigure}[b]{0.20\textwidth}
      \centering
      \includegraphics[width=\textwidth]{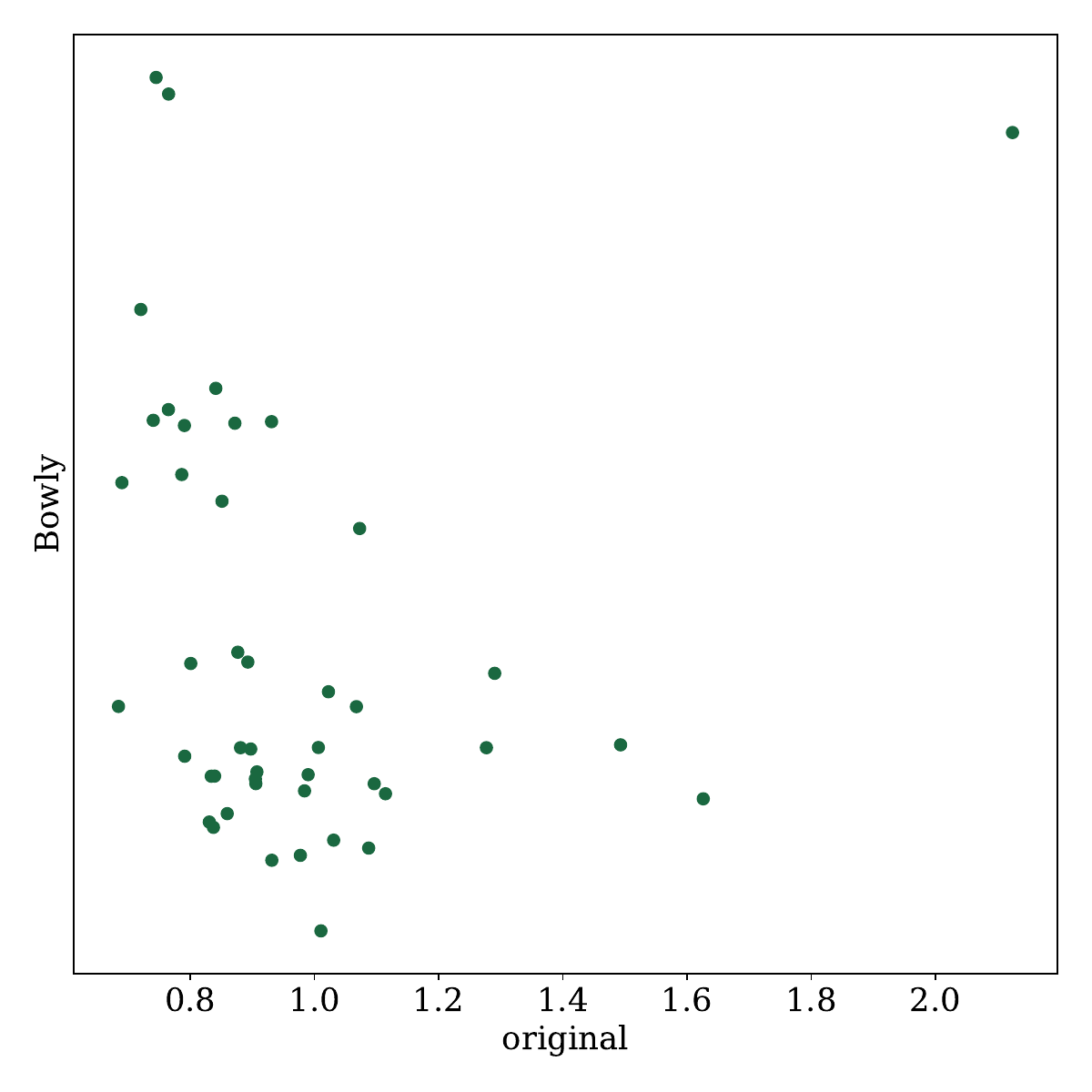}
      \vspace{-0.6cm}
      \caption{Bowly}
      \label{fig:downstream1_ca_bowly}
    \end{subfigure}
    \begin{subfigure}[b]{0.20\textwidth}
      \centering
      \includegraphics[width=\textwidth]{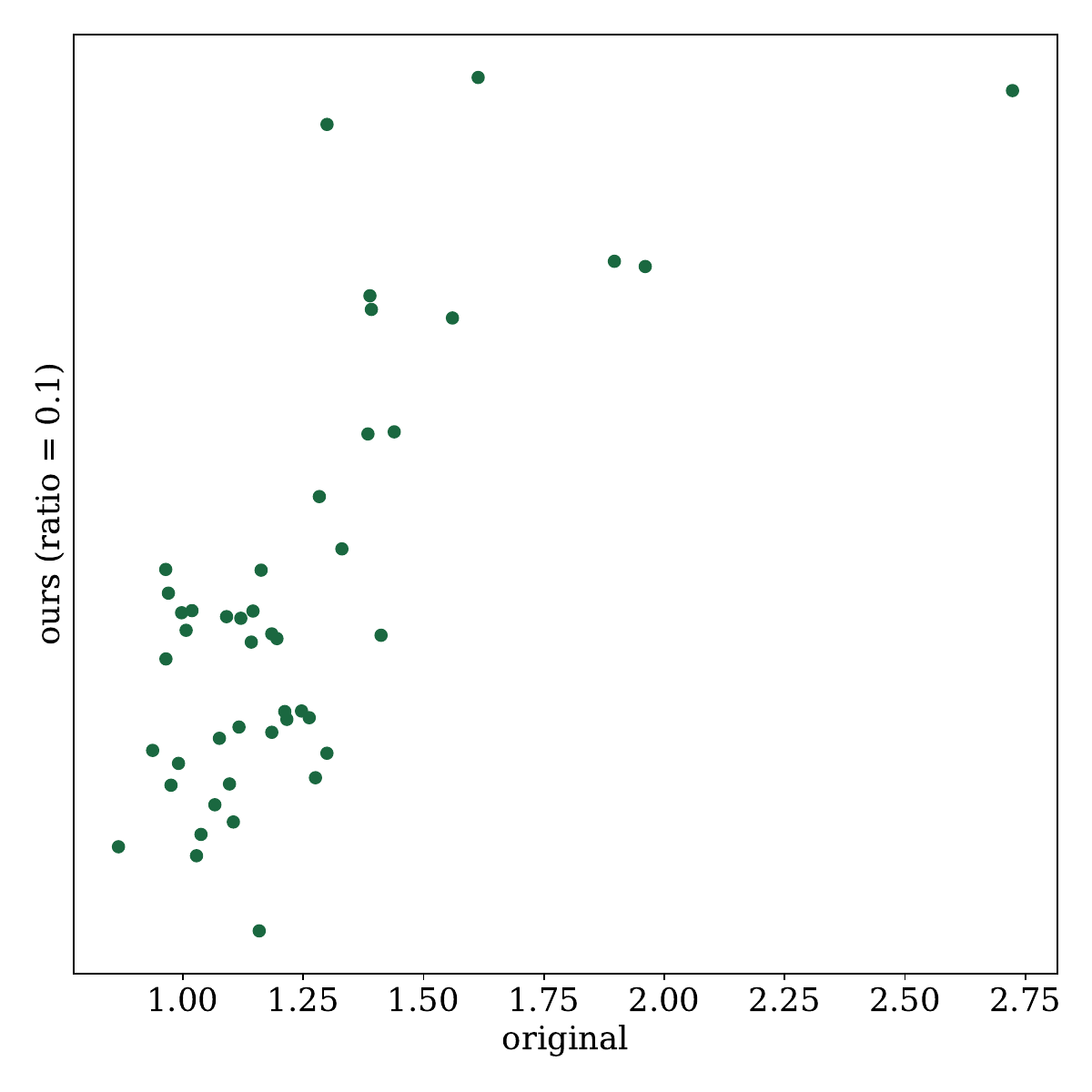}
      \vspace{-0.6cm}
      \caption{ours ($\gamma$ = 0.1)}
      \label{fig:downstream1_ca_ours010}
    \end{subfigure}
    \begin{subfigure}[b]{0.20\textwidth}
      \centering
      \includegraphics[width=\textwidth]{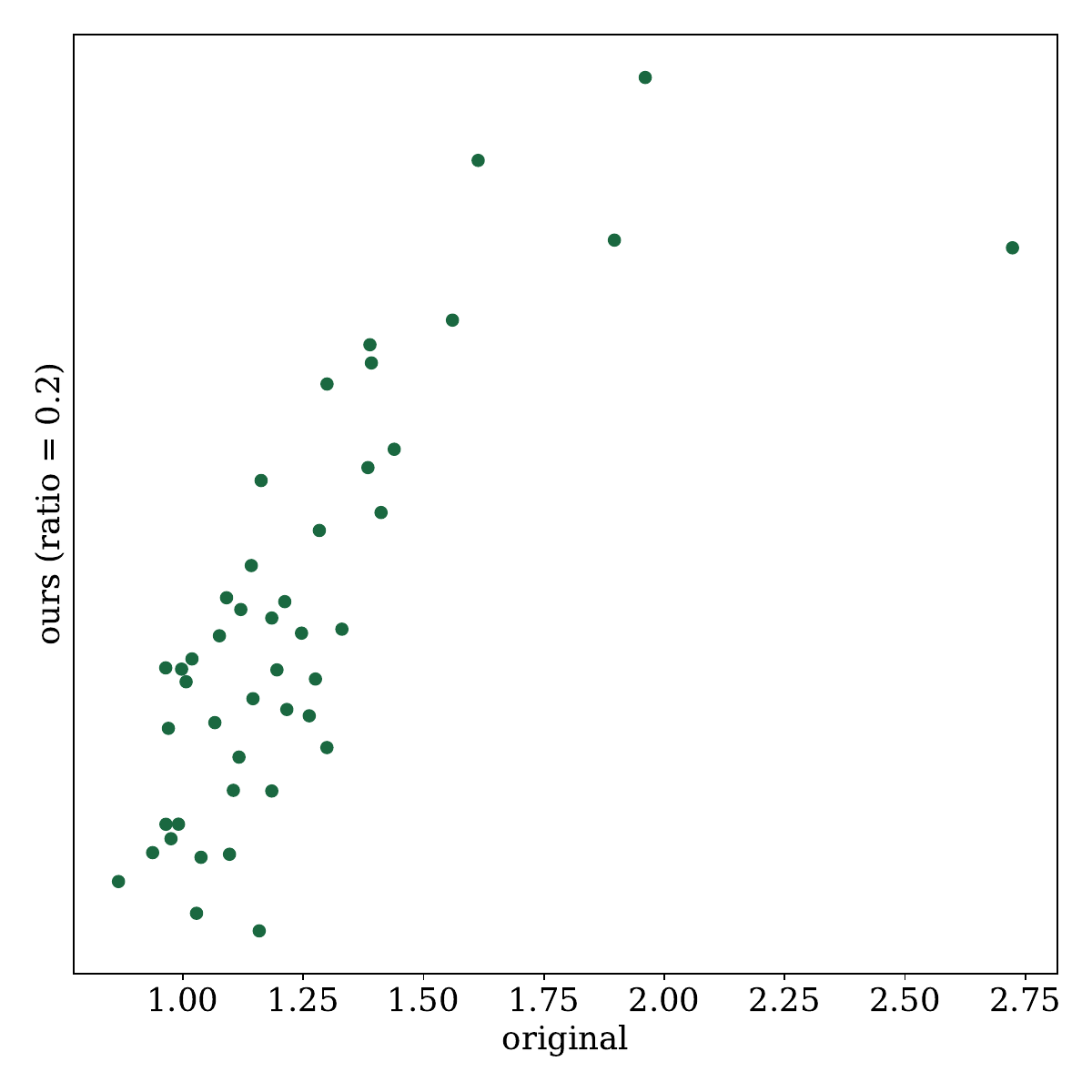}
      \vspace{-0.6cm}
      \caption{ours ($\gamma$ = 0.2)}
      \label{fig:downstream1_ca_ours020}
    \end{subfigure}
    \begin{subfigure}[b]{0.20\textwidth}
      \centering
      \includegraphics[width=\textwidth]{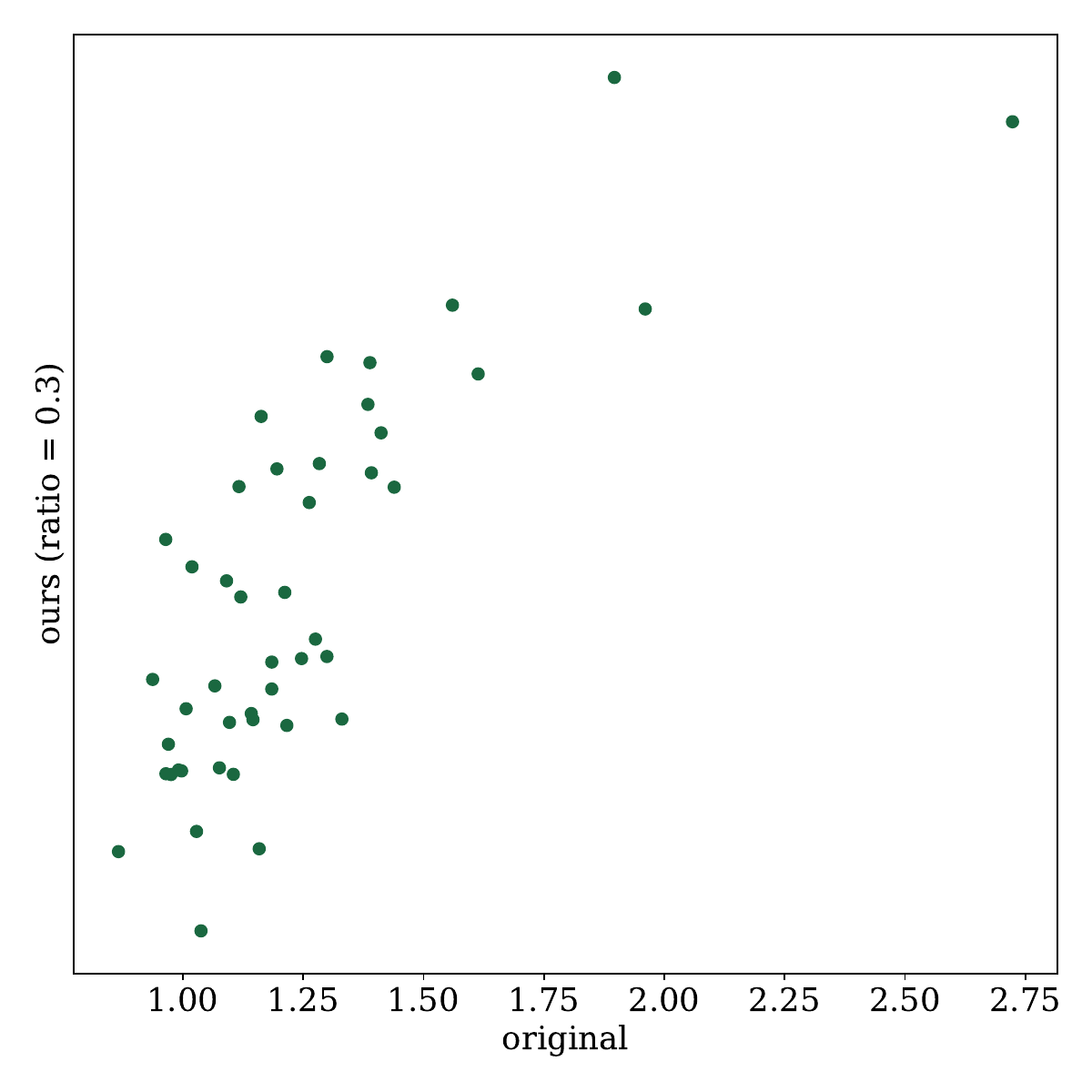}
      \vspace{-0.6cm}
      \caption{ours ($\gamma$ = 0.3)}
      \label{fig:downstream1_ca_ours030}
    \end{subfigure}
  \end{minipage}
  \caption{The solution time of SCIP on the CA with $45$ different hyper-parameter sets.}
    \label{fig:downstream1_visualization_ca}
\end{figure}

\begin{figure}[t]
  \centering
  \begin{minipage}{\textwidth}
    \centering
    \begin{subfigure}[b]{0.20\textwidth}
      \centering
      \includegraphics[width=\textwidth]{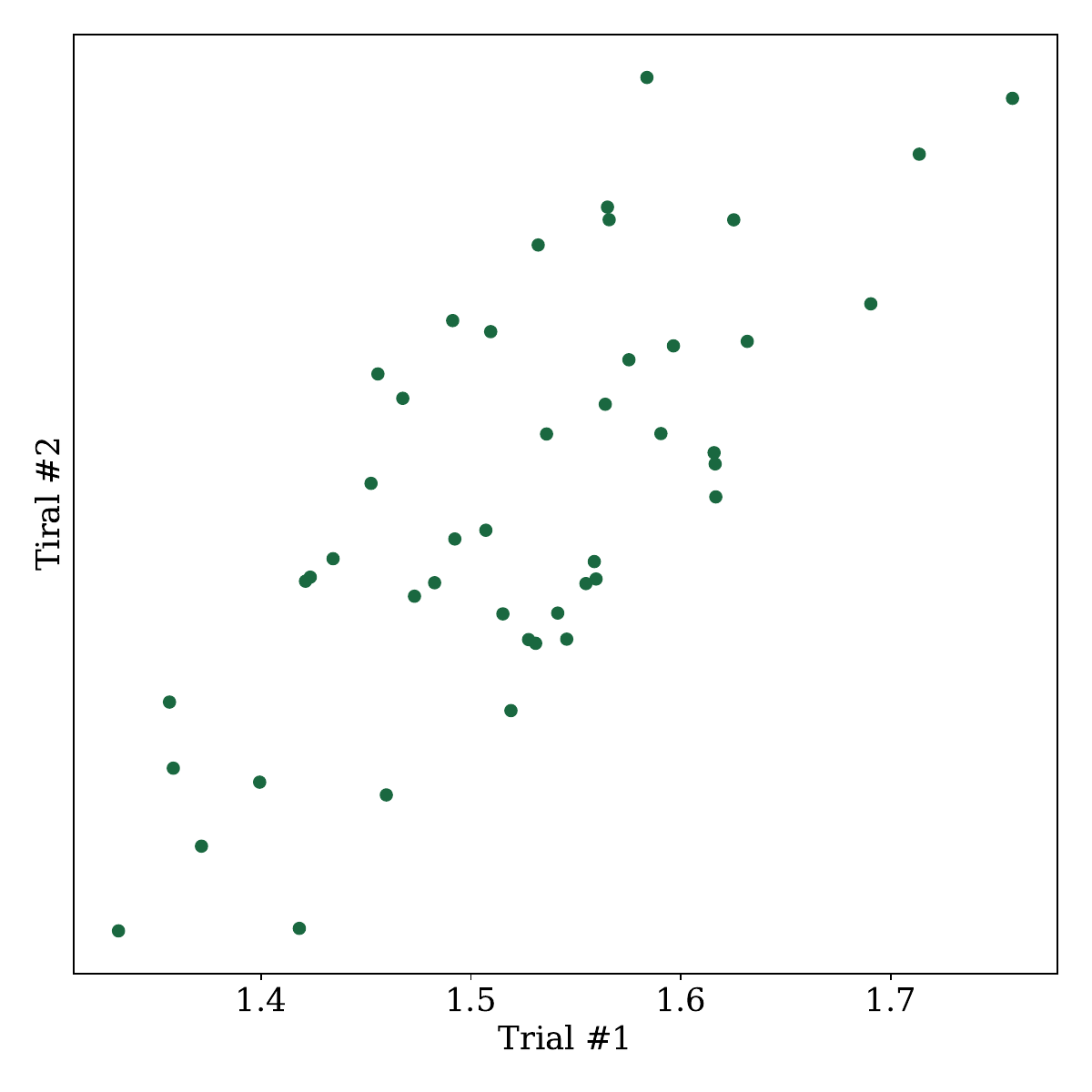}
      \vspace{-0.6cm}
      \caption{two trials}
      \label{fig:downstream1_sc_sc}
    \end{subfigure}
    \centering
    \begin{subfigure}[b]{0.20\textwidth}
      \centering
      \includegraphics[width=\textwidth]{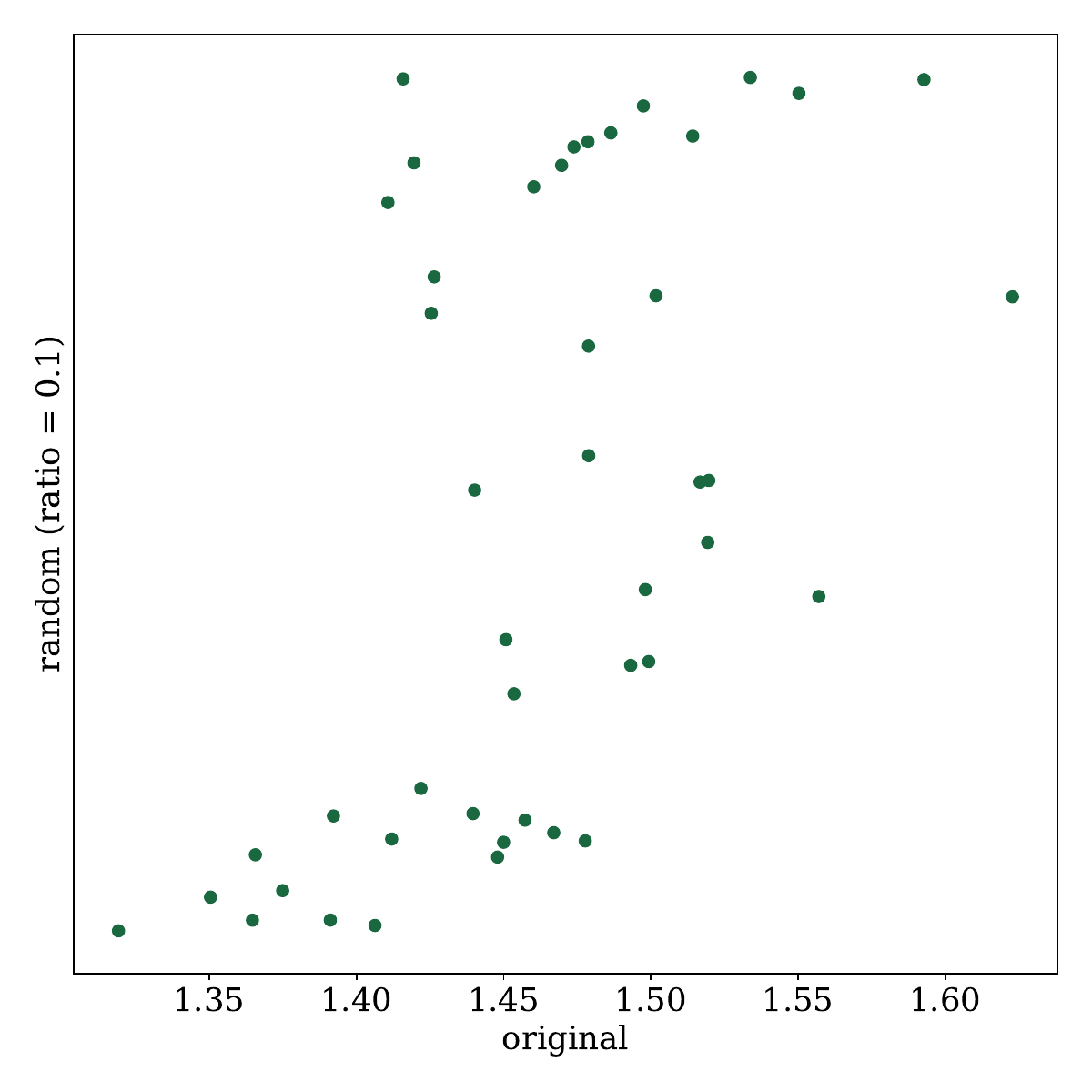}
      \vspace{-0.6cm}
      \caption{random ($\gamma$ = 0.1)}
      \label{fig:downstream1_sc_random010}
    \end{subfigure}
    \begin{subfigure}[b]{0.20\textwidth}
      \centering
      \includegraphics[width=\textwidth]{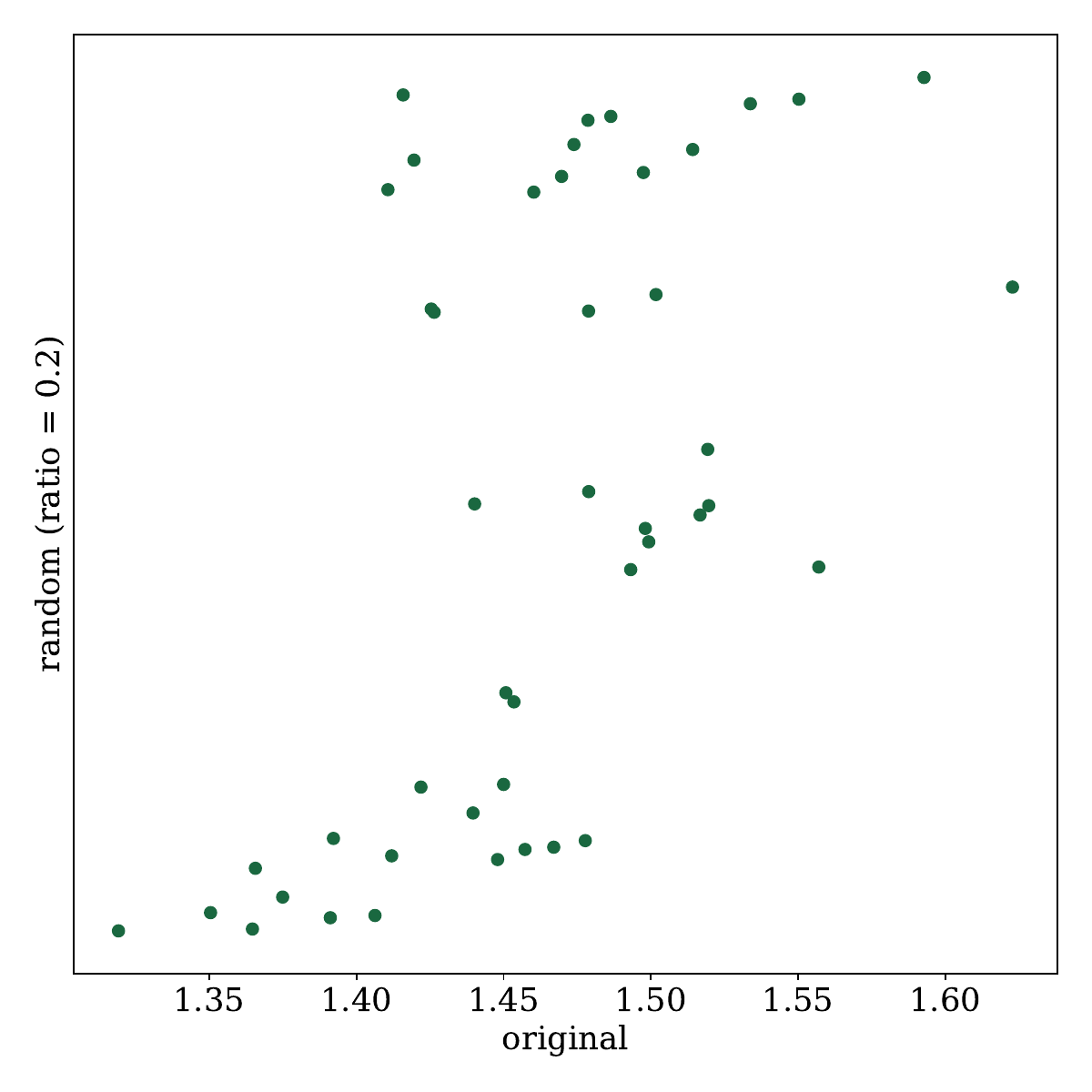}
      \vspace{-0.6cm}
      \caption{random ($\gamma$ = 0.2)}
      \label{fig:downstream1_sc_random020}
    \end{subfigure}
    \begin{subfigure}[b]{0.20\textwidth}
      \centering
      \includegraphics[width=\textwidth]{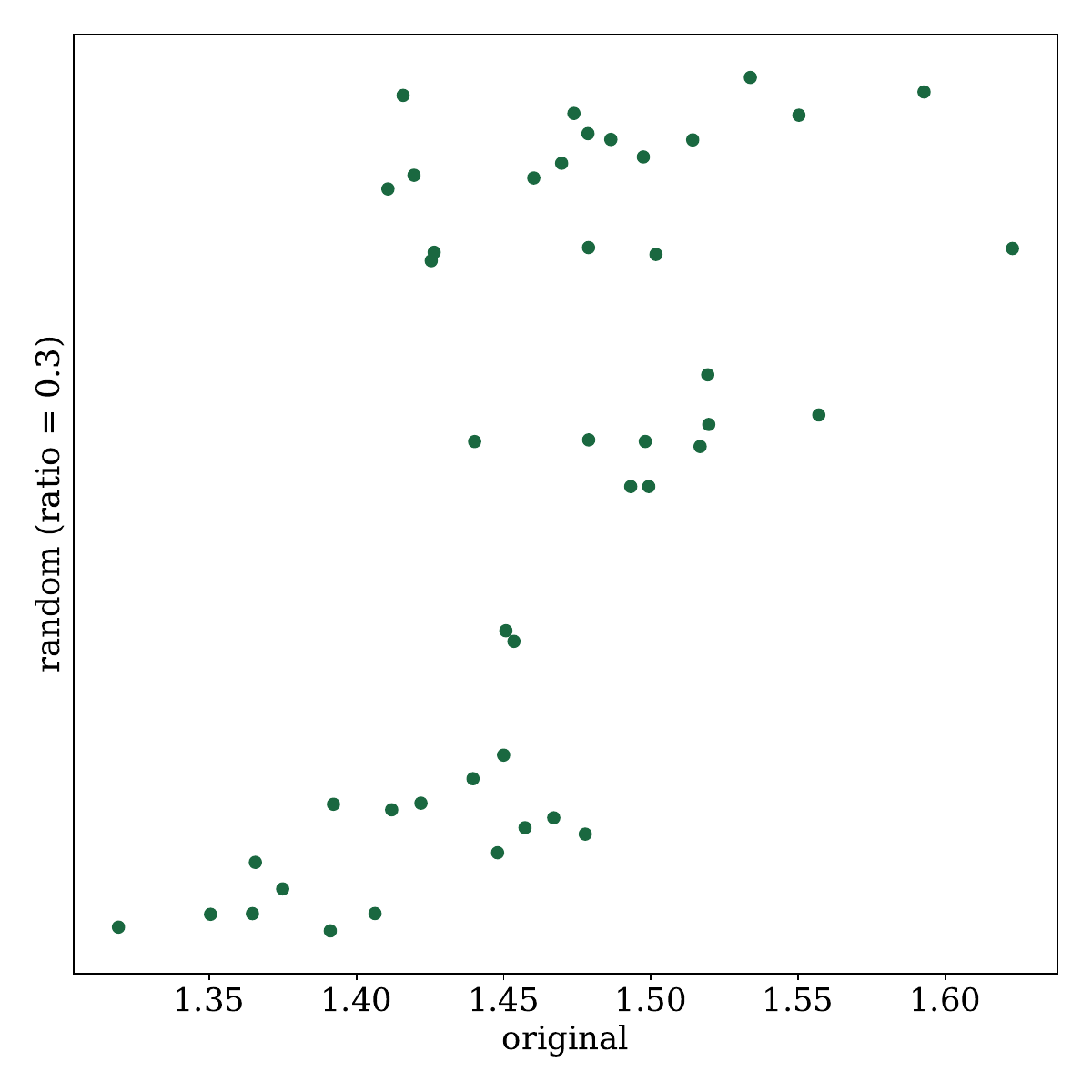}
      \vspace{-0.6cm}
      \caption{random ($\gamma$ = 0.3)}
      \label{fig:downstream1_sc_random030}
    \end{subfigure}
    \vfill
    \vspace{0.0cm} 
    \begin{subfigure}[b]{0.20\textwidth}
      \centering
      \includegraphics[width=\textwidth]{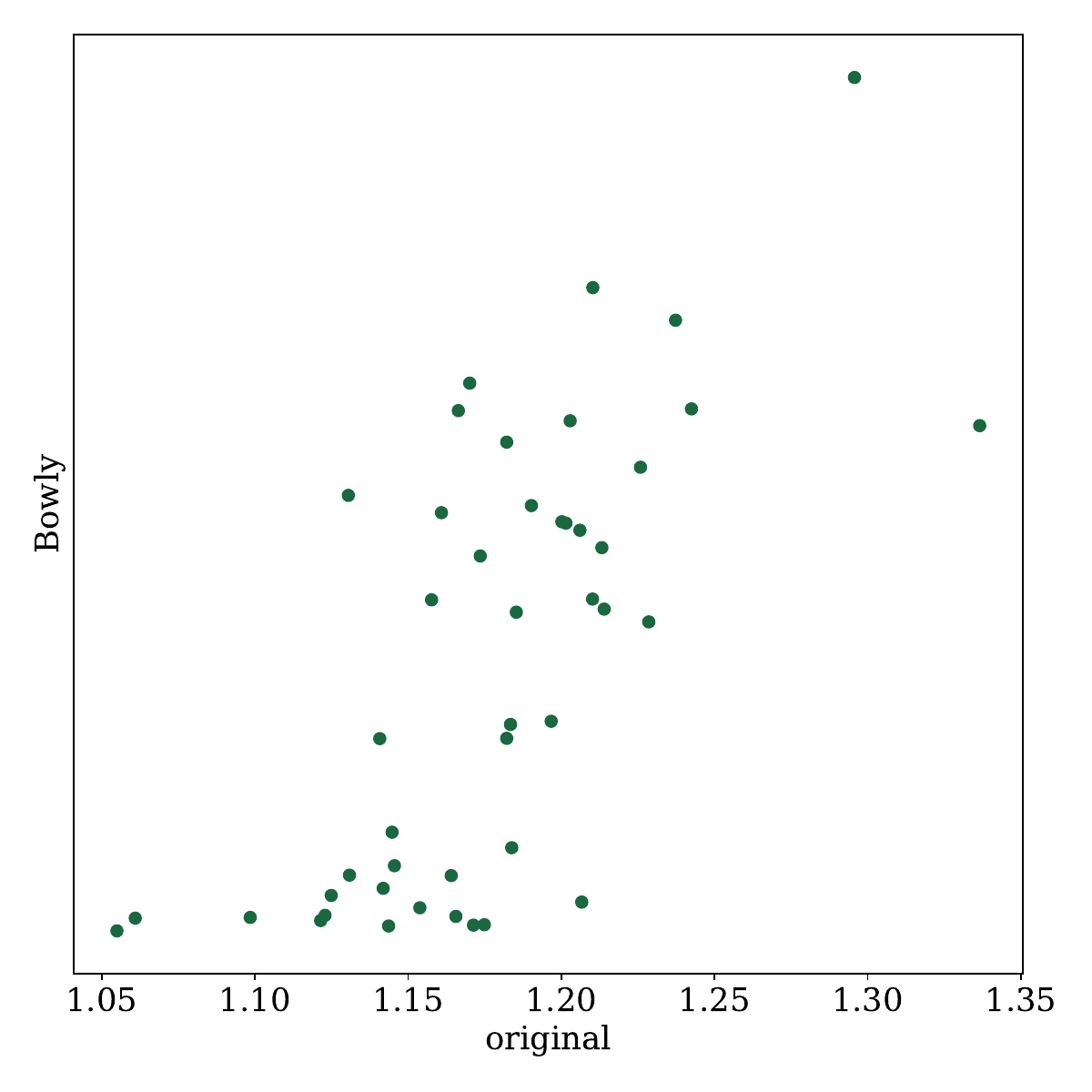}
      \vspace{-0.6cm}
      \caption{Bowly}
      \label{fig:downstream1_sc_bowly}
    \end{subfigure}
    \begin{subfigure}[b]{0.20\textwidth}
      \centering
      \includegraphics[width=\textwidth]{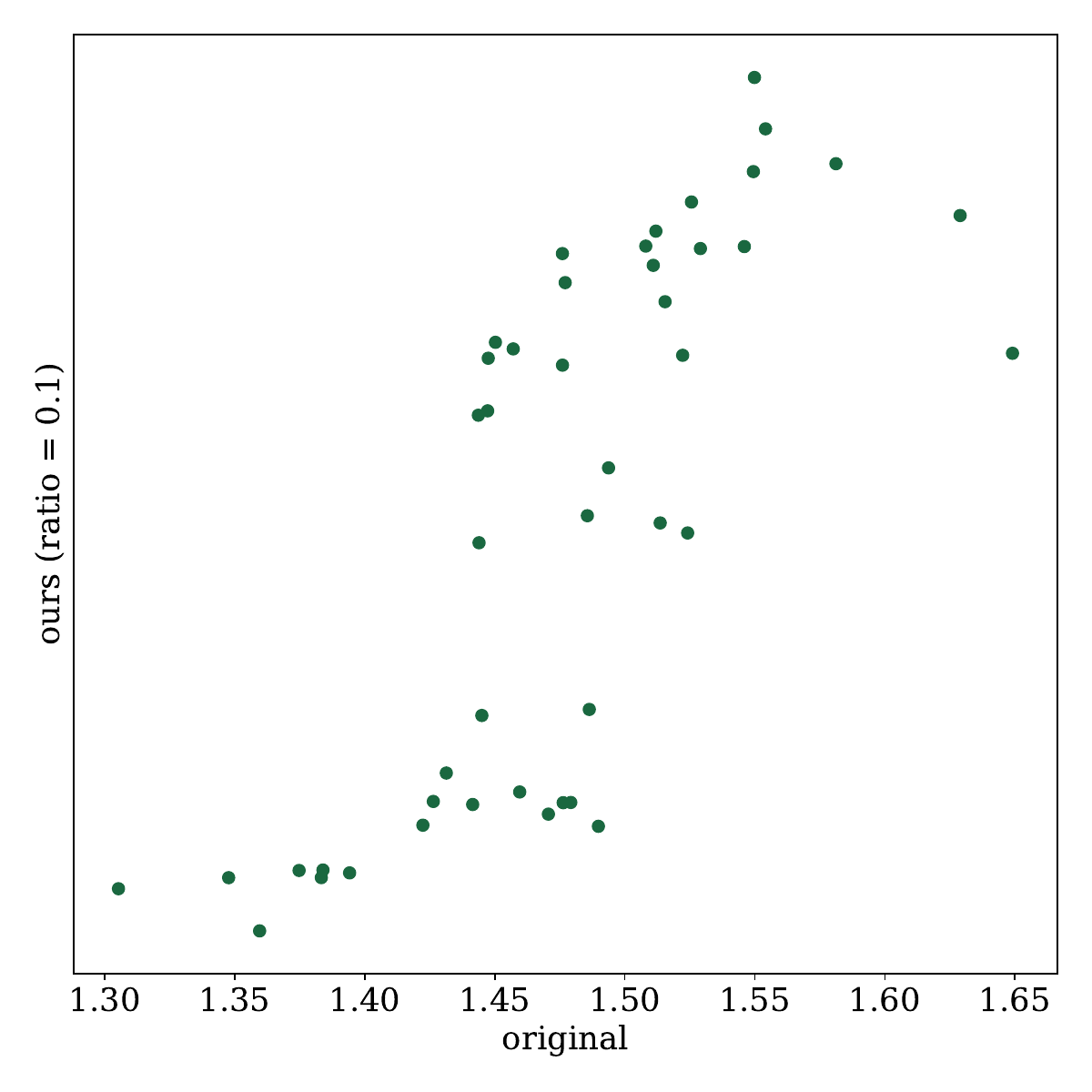}
      \vspace{-0.6cm}
      \caption{ours ($\gamma$ = 0.1)}
      \label{fig:downstream1_sc_ours010}
    \end{subfigure}
    \begin{subfigure}[b]{0.20\textwidth}
      \centering
      \includegraphics[width=\textwidth]{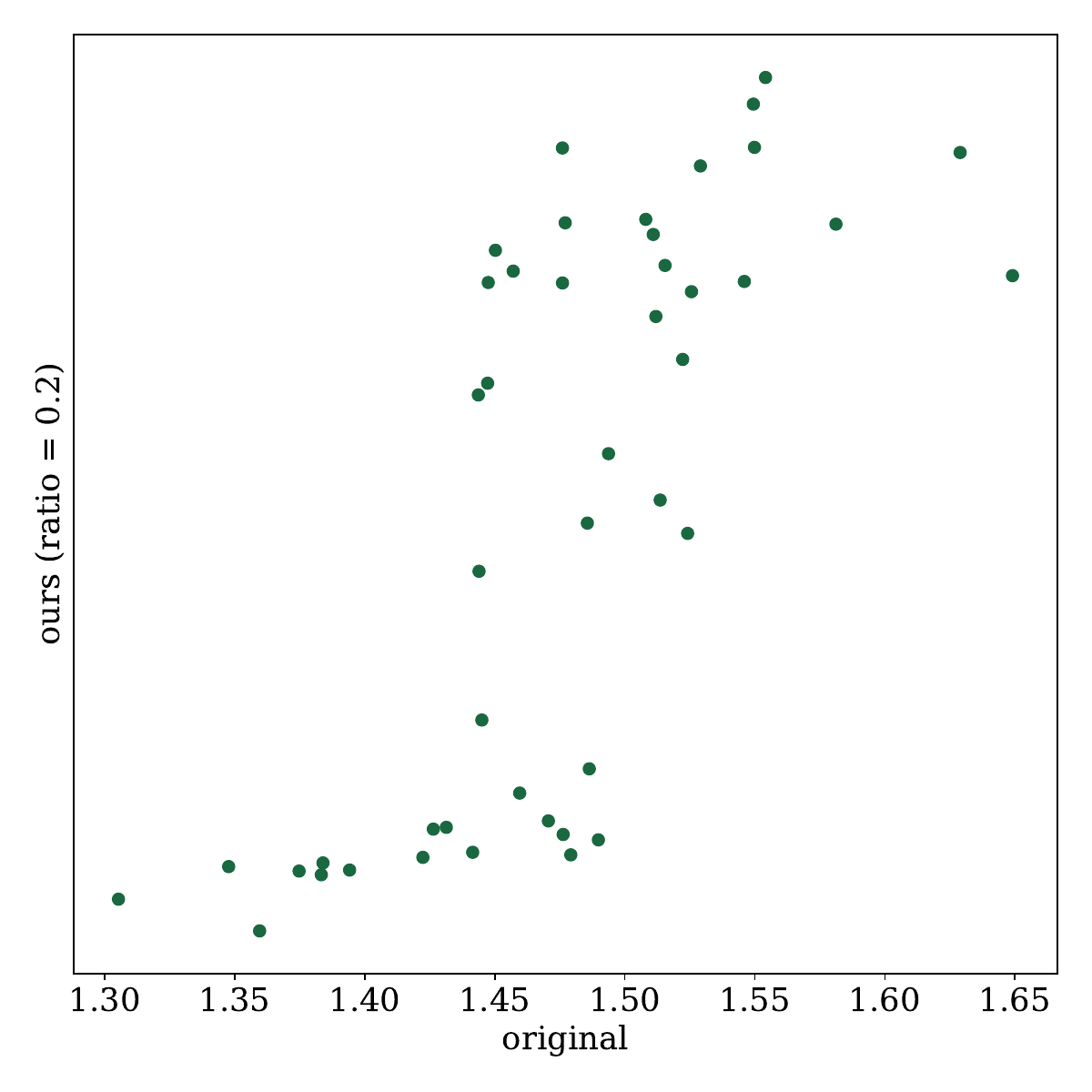}
      \vspace{-0.6cm}
      \caption{ours ($\gamma$ = 0.2)}
      \label{fig:downstream1_sc_ours020}
    \end{subfigure}
    \begin{subfigure}[b]{0.20\textwidth}
      \centering
      \includegraphics[width=\textwidth]{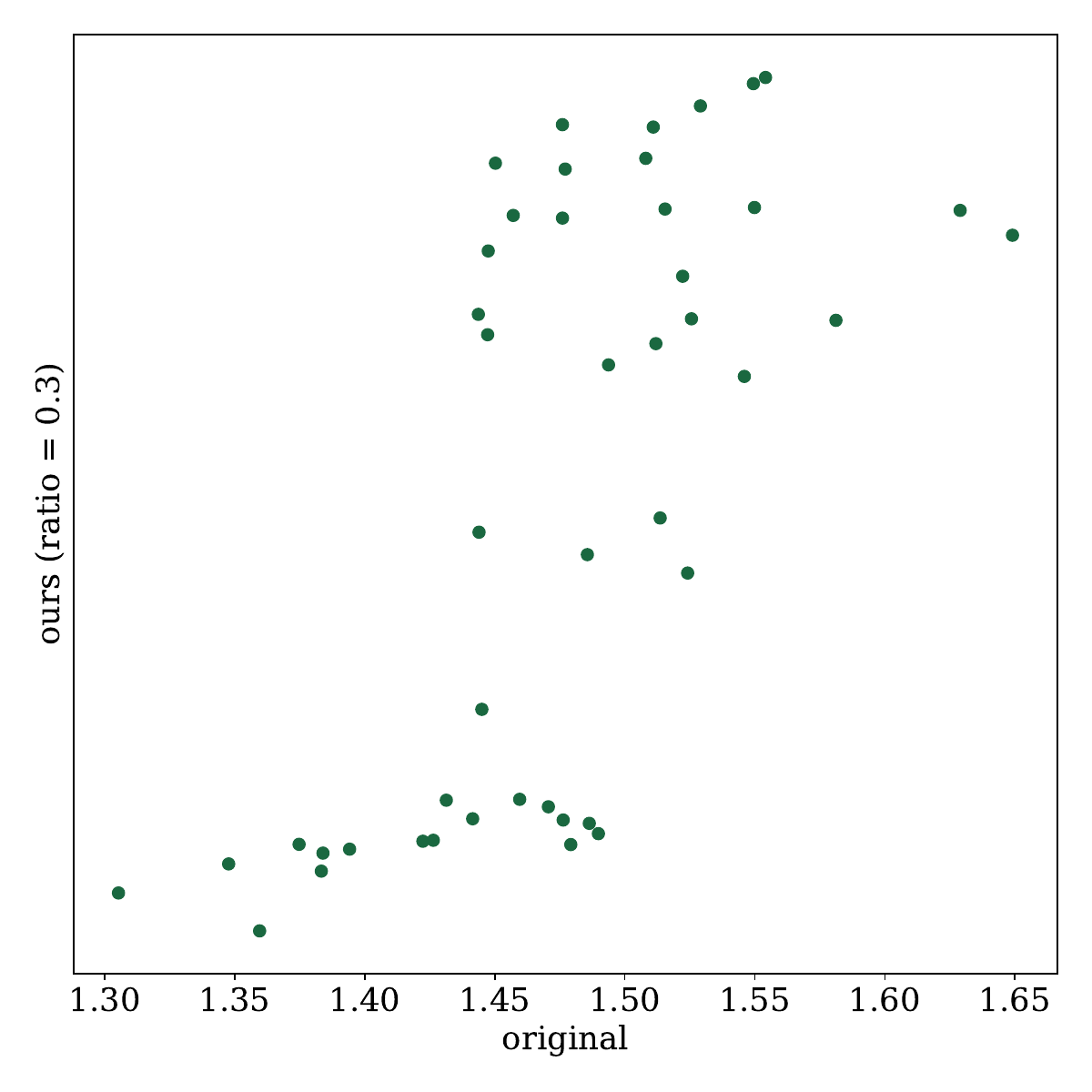}
      \vspace{-0.6cm}
      \caption{ours ($\gamma$ = 0.3)}
      \label{fig:downstream1_sc_ours030}
    \end{subfigure}
  \end{minipage}
  \caption{The solution time of SCIP on the SC with $45$ different hyper-parameter sets.}
    \label{fig:downstream1_visualization_sc}
\end{figure}

\begin{figure}[t]
  \centering
  \begin{minipage}{\textwidth}
    \centering
    \begin{subfigure}[b]{0.20\textwidth}
      \centering
      \includegraphics[width=\textwidth]{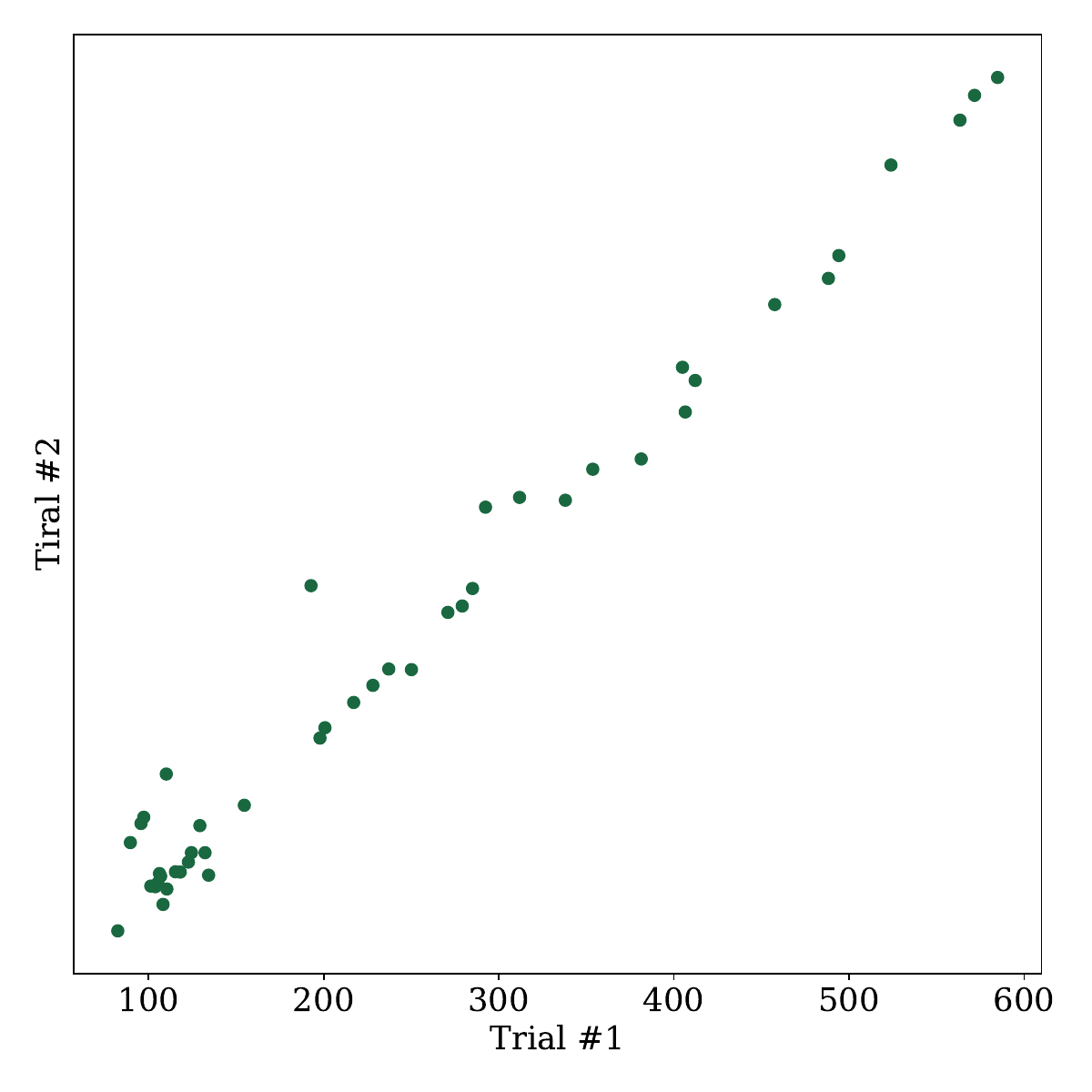}
      \vspace{-0.6cm}
      \caption{two trials}
      \label{fig:downstream1_iis_iis}
    \end{subfigure}
    \centering
    \begin{subfigure}[b]{0.20\textwidth}
      \centering
      \includegraphics[width=\textwidth]{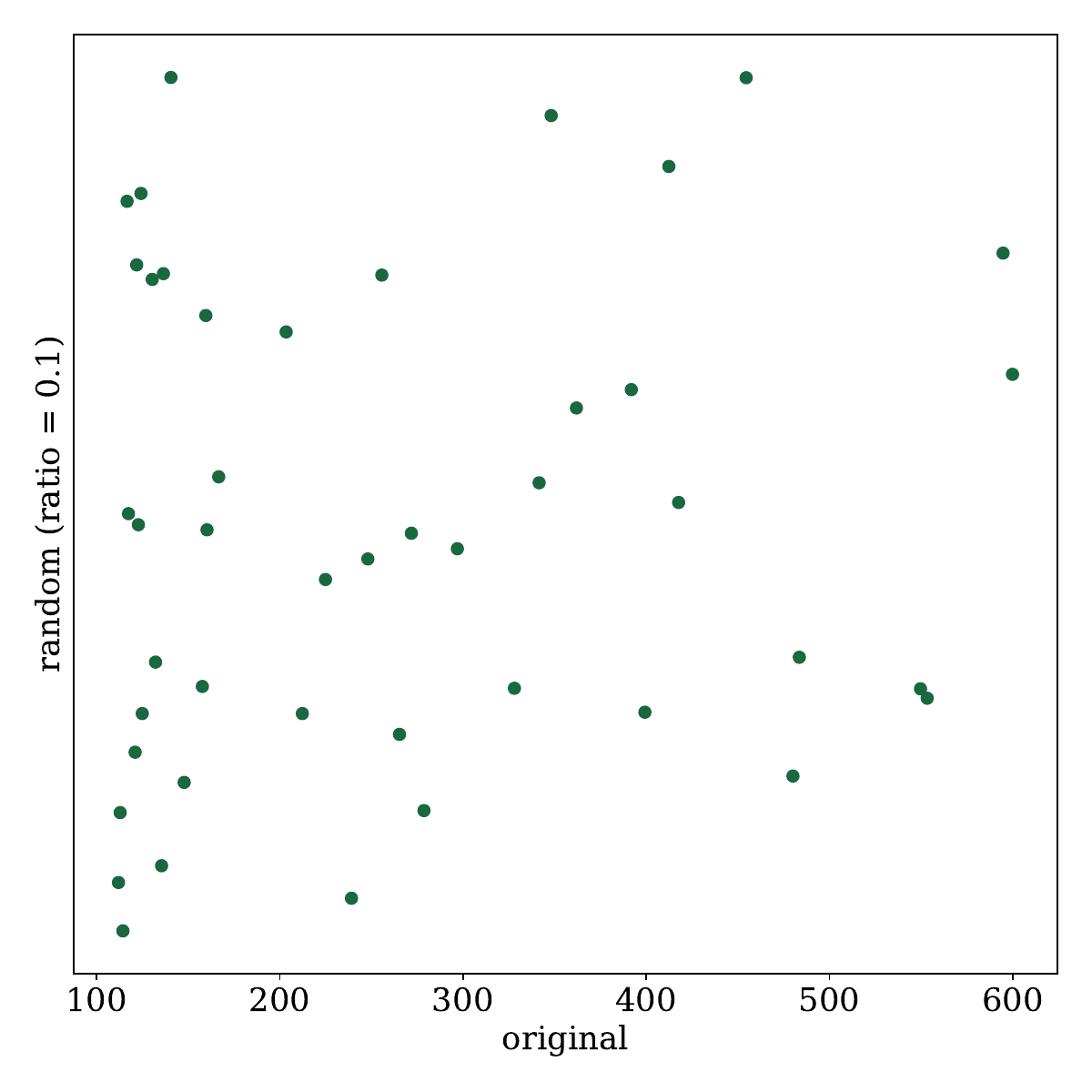}
      \vspace{-0.6cm}
      \caption{random ($\gamma$ = 0.1)}
      \label{fig:downstream1_iis_random010}
    \end{subfigure}
    \begin{subfigure}[b]{0.20\textwidth}
      \centering
      \includegraphics[width=\textwidth]{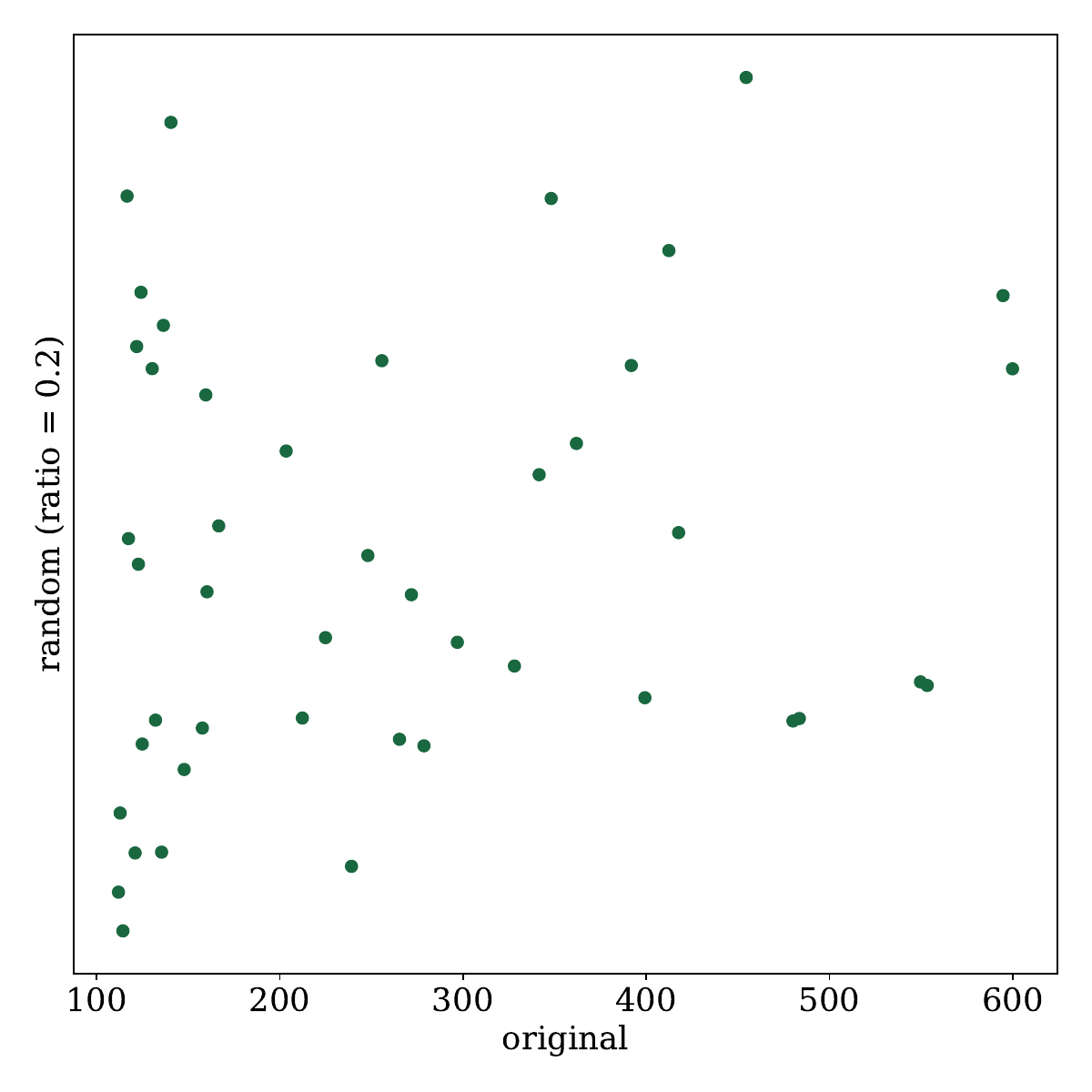}
      \vspace{-0.6cm}
      \caption{random ($\gamma$ = 0.2)}
      \label{fig:downstream1_iis_random020}
    \end{subfigure}
    \begin{subfigure}[b]{0.20\textwidth}
      \centering
      \includegraphics[width=\textwidth]{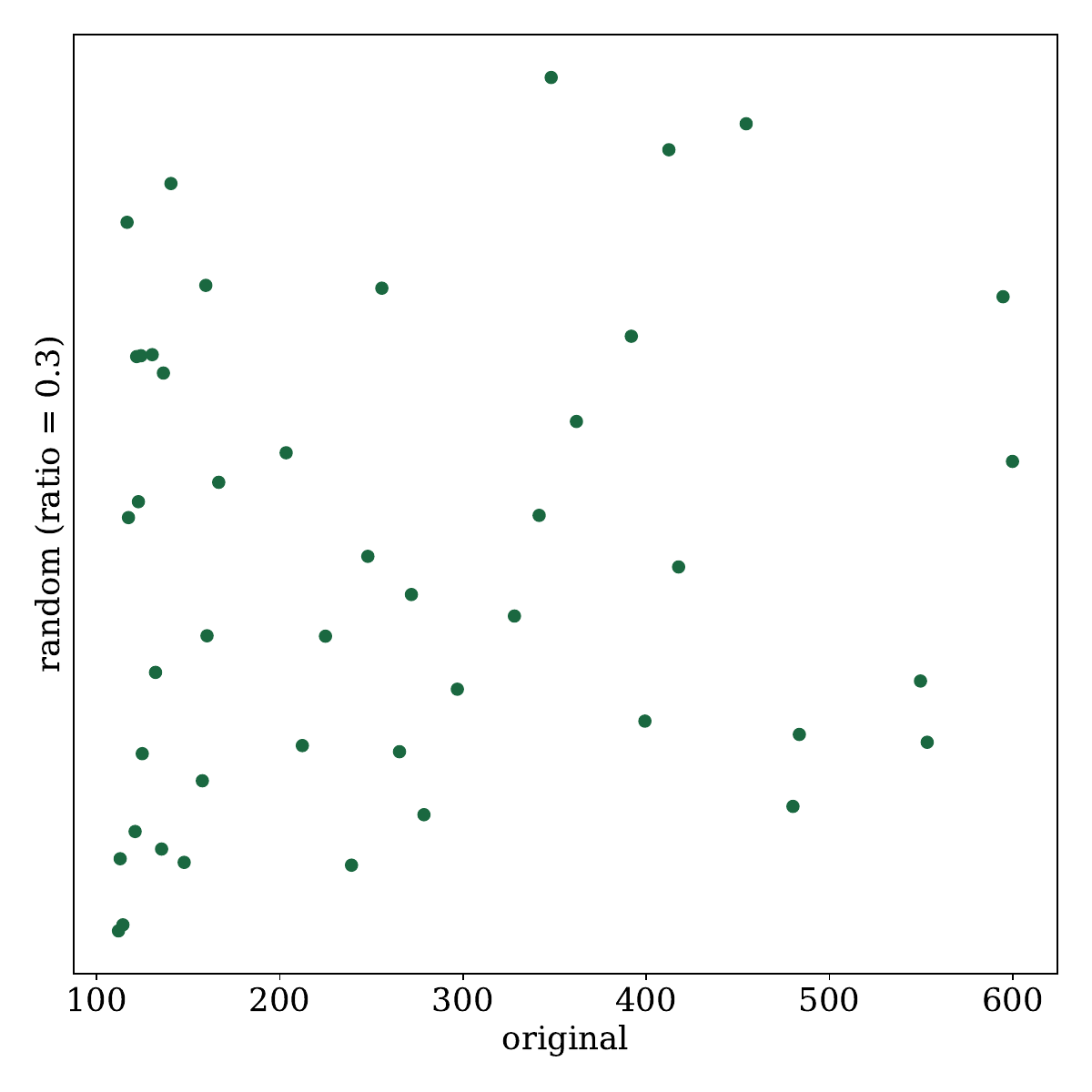}
      \vspace{-0.6cm}
      \caption{random ($\gamma$ = 0.3)}
      \label{fig:downstream1_iis_random030}
    \end{subfigure}
    \vfill
    \vspace{0.0cm} 
    \begin{subfigure}[b]{0.20\textwidth}
      \centering
      \includegraphics[width=\textwidth]{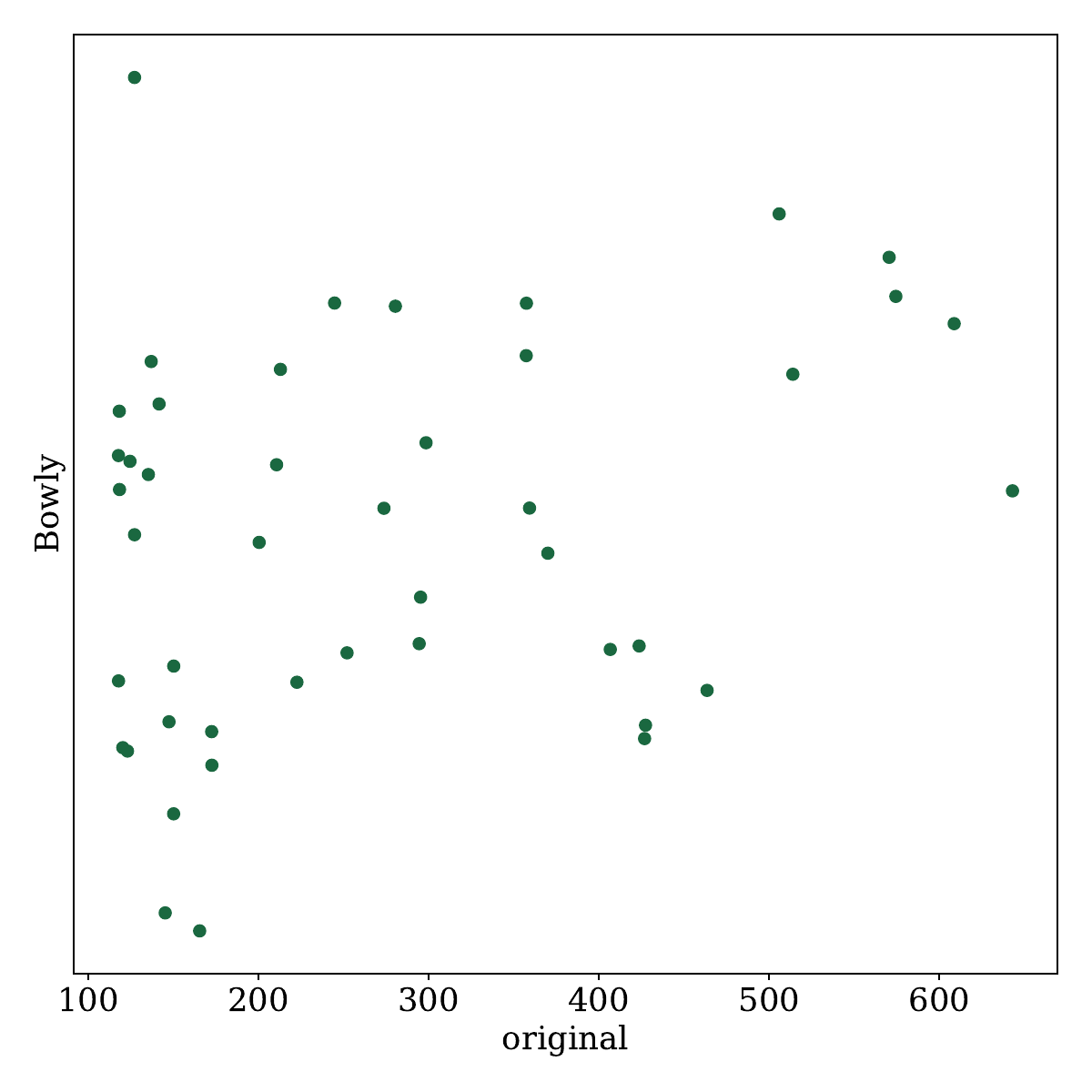}
      \vspace{-0.6cm}
      \caption{Bowly}
      \label{fig:downstream1_iis_bowly}
    \end{subfigure}
    \begin{subfigure}[b]{0.20\textwidth}
      \centering
      \includegraphics[width=\textwidth]{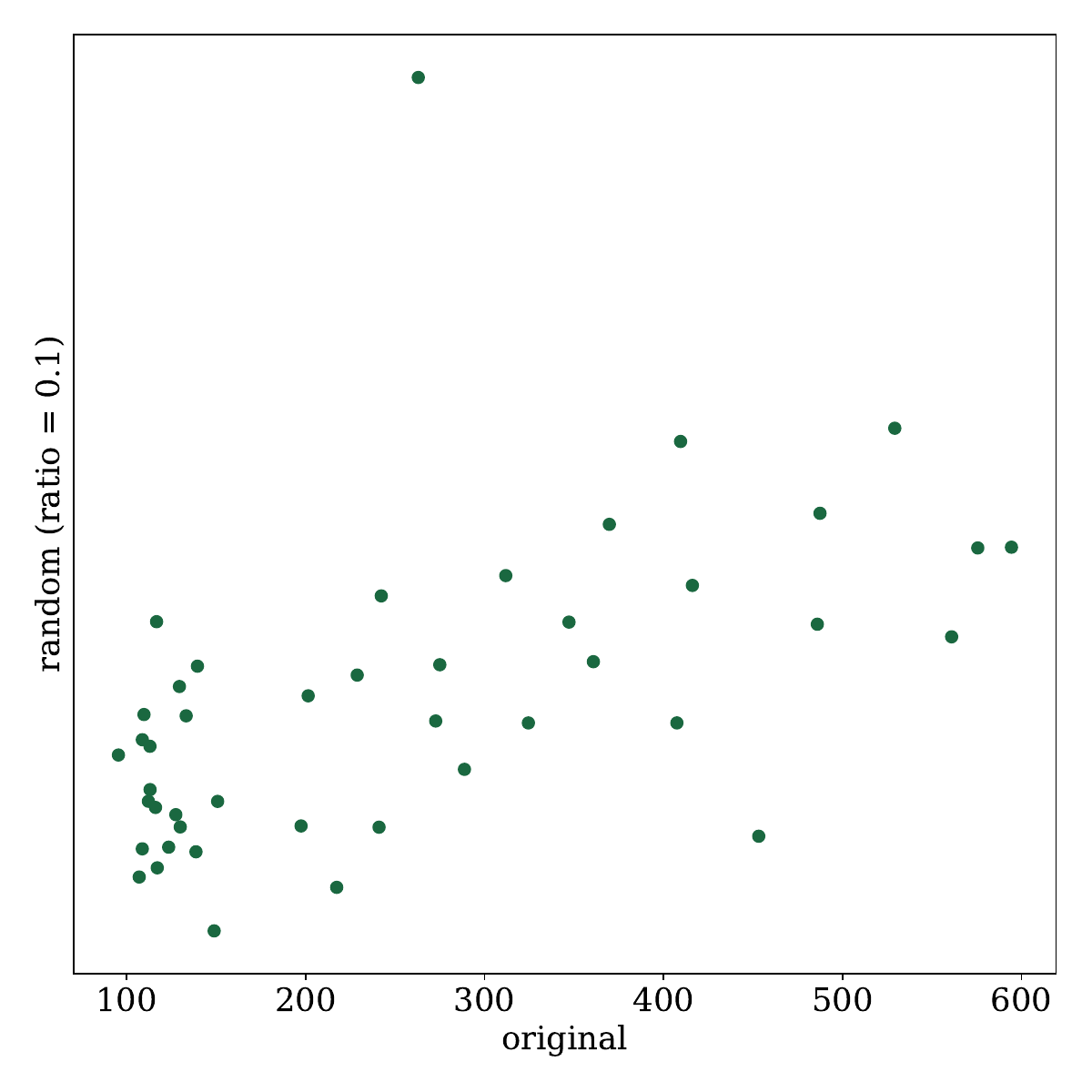}
      \vspace{-0.6cm}
      \caption{ours ($\gamma$ = 0.1)}
      \label{fig:downstream1_iis_ours010}
    \end{subfigure}
    \begin{subfigure}[b]{0.20\textwidth}
      \centering
      \includegraphics[width=\textwidth]{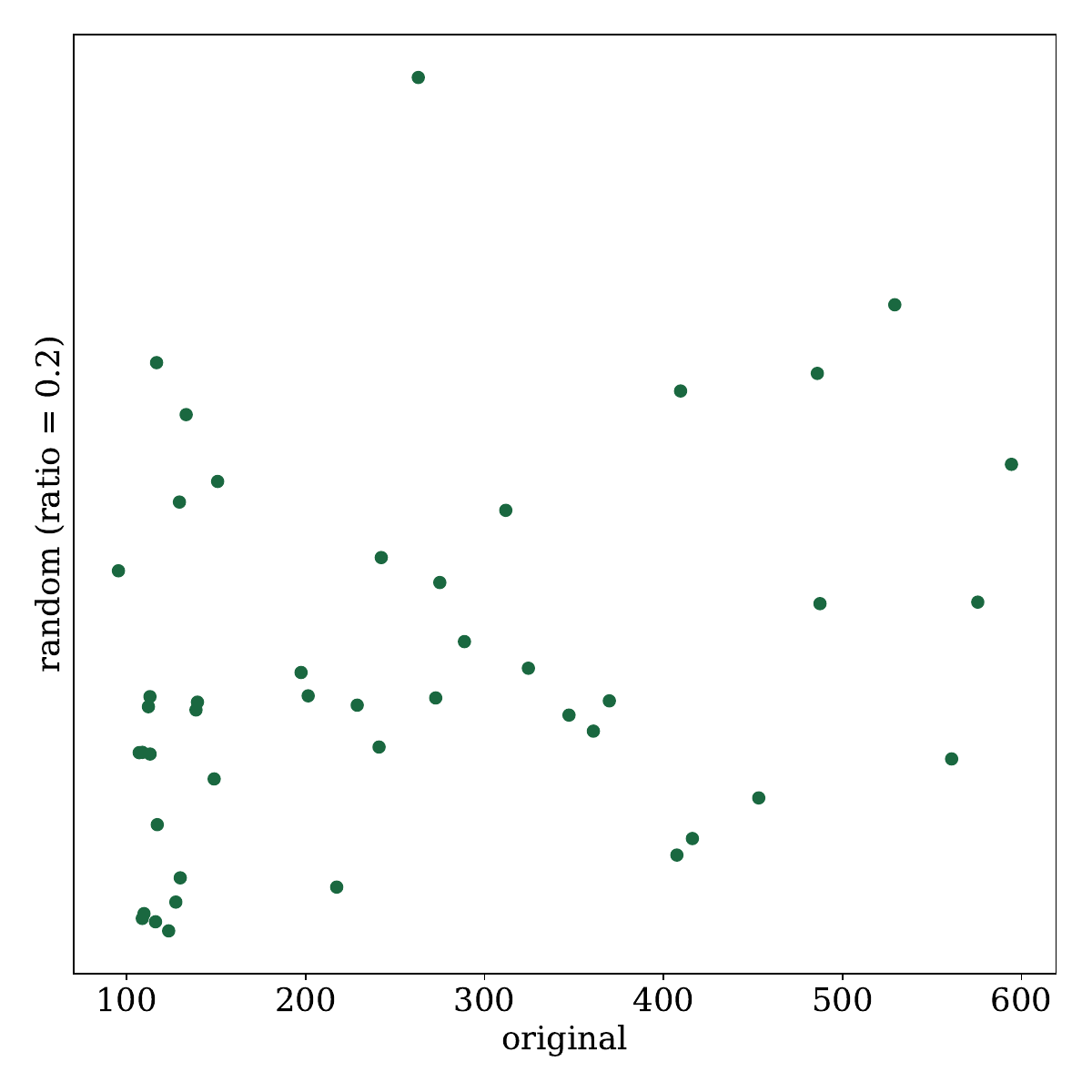}
      \vspace{-0.6cm}
      \caption{ours ($\gamma$ = 0.2)}
      \label{fig:downstream1_iis_ours020}
    \end{subfigure}
    \begin{subfigure}[b]{0.20\textwidth}
      \centering
      \includegraphics[width=\textwidth]{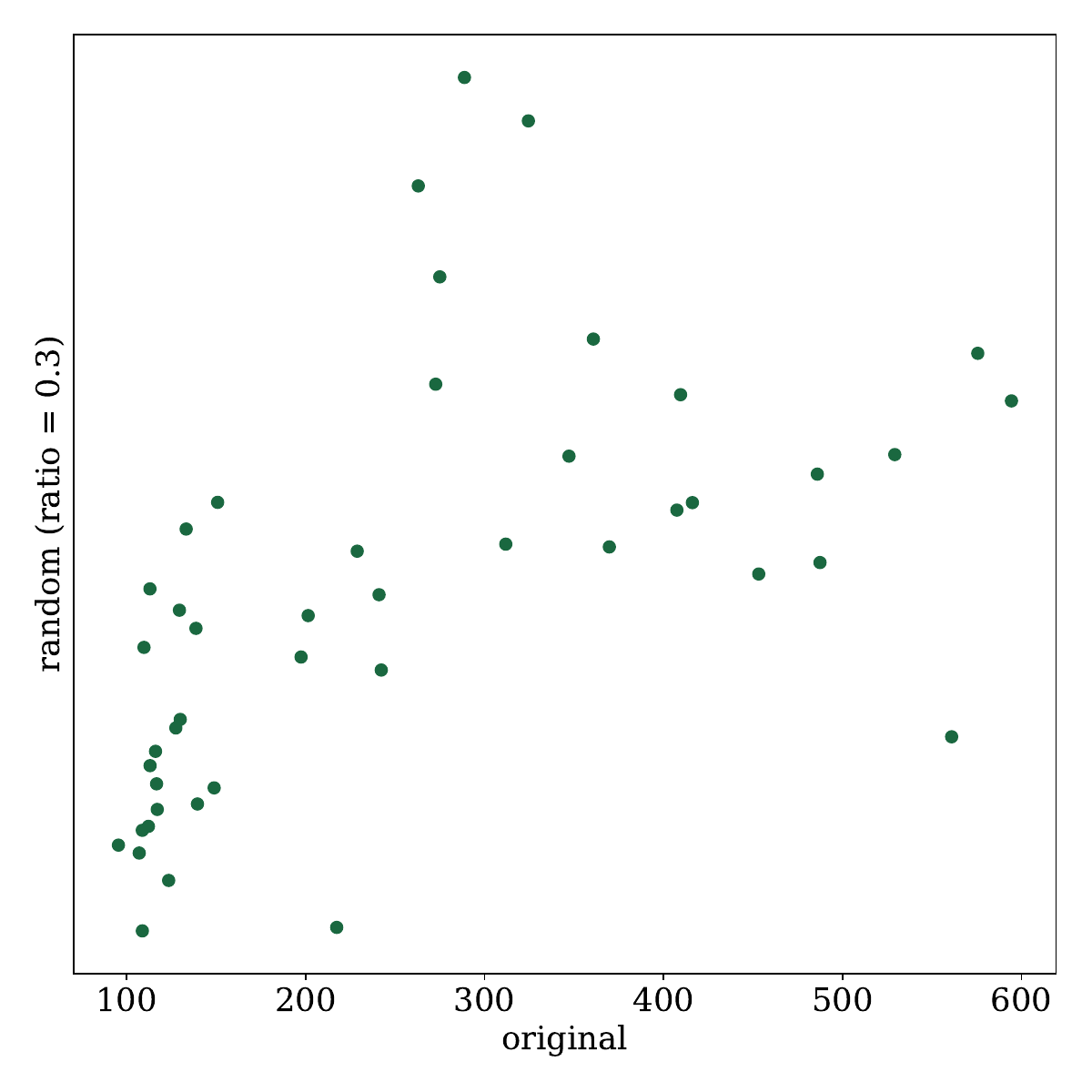}
      \vspace{-0.6cm}
      \caption{ours ($\gamma$ = 0.3)}
      \label{fig:downstream1_iis_ours030}
    \end{subfigure}
  \end{minipage}
  \caption{The solution time of SCIP on the IIS with $45$ different hyper-parameter sets.}
    \label{fig:downstream1_visualization_iis}
\end{figure}

\end{document}